\documentclass[twoside,11pt]{article}

\usepackage{jmlr2e}

\usepackage{amssymb,amsmath,comment}
\usepackage{amsfonts,mathrsfs}
\usepackage{color}
\usepackage{bbm}
\usepackage{bm}
\usepackage{booktabs,float}
\usepackage{graphicx}
\usepackage{enumerate}
\usepackage{rotating}
\newcommand{\cmmnt}[1]{} 

\newcommand{\R}{\mathbb{R}}
\newcommand{\E}{\mathbb{E}}
\newcommand{\F}{\mathcal{F}}

\usepackage{marginnote}
\makeatletter
\renewcommand*{\thanks}[1]{%
  \footnotemark
  \protected@xdef\@thanks{\@thanks
    \protect\footnotetext[\arabic{footnote}]{#1}}%
}
\makeatother

\ShortHeadings{Non-asymptotic Guarantees for Robust Statistical Learning }{Xu, Yao, Yao and Zhang}
\firstpageno{1}

\begin{document}

\title{Non-Asymptotic Guarantees for Robust Statistical Learning under Infinite Variance Assumption}

\author{\name Lihu Xu  \email lihuxu@um.edu.mo \\
       \addr 1. Department of Mathematics, Faculty of Science and Technology, University of Macau, Taipa Macau, China;
       2. Zhuhai UM Science \& Technology Research Institute, Zhuhai, China.
       \AND
       \name Fang Yao  \email fyao@math.pku.edu.cn \\
       \addr School of Mathematical Sciences and Center for Statistical Science, Peking University, Beijing, China.
       \AND
       \name Qiuran Yao  \email yb97478@connect.um.edu.mo \\
       \addr 1. Department of Mathematics, Faculty of Science and Technology, University of Macau, Taipa Macau, China;
       2. Zhuhai UM Science \& Technology Research Institute, Zhuhai, China.
       \AND
       \name  Huiming Zhang \thanks{Corresponding authors: Huiming Zhang (huimingzhang@um.edu.mo). The authors are in alphabetical order.}
       \email huimingzhang@um.edu.mo  \\
 \addr 1. Department of Mathematics, Faculty of Science and Technology, University of Macau, Taipa Macau, China;
       2. Zhuhai UM Science \& Technology Research Institute, Zhuhai, China.
       }


\maketitle

\vspace{1 cm}
\begin{abstract}
There has been a surge of interest in developing robust estimators for models with heavy-tailed and bounded variance data in statistics and machine learning, while few works impose unbounded variance. This paper proposes two type of robust estimators, the ridge log-truncated M-estimator and the elastic net log-truncated M-estimator. The first estimator is applied to convex regressions such as quantile regression and generalized linear models, while the other one is applied to high dimensional non-convex learning problems such as regressions via
 deep neural networks. Simulations and real data analysis demonstrate the {robustness} of log-truncated estimations over standard estimations.



\end{abstract}

\begin{keywords}
data with infinite variance, excess risk bounds,  robust ridge regressions,  robust elastic net regressions,  robust non-convex regressions, {robust deep neural network (DNN) regressions}.
\end{keywords}

\section{Introduction}\label{intro}

\subsection{Backgrounds}\label{BACK}
Robust statistics is a traditional topic that has been well studied since the pioneering work of \cite{huber1964robust} and \cite{tukey1960a} in 1960s. Distributionally robust learning is nowadays revitalized and invigorating in statistical learning, see \cite{nemirovskij1983problem} for median-of-means estimators, \cite{baraud2017new} for minimax types estimators and \cite{catoni2012challenging} for log-truncated estimators. For further details, we refer the reader to the note by \cite{lerasle2019lecture} and the review paper by \cite{lugosi2019mean} for comprehensive introductions.

 Many existing pieces of research on excess risk bounds for robust estimators heavily rely on one or more of the following three assumptions: (1) bounded loss functions \citep{bartlett2006empirical,yi2020non}; (2) bounded Lipschitz condition \citep{chinot2019robust,shen2021non}; (3) bounded data \citep{liu2014robustness,brownlees2015empirical,zhang2018ell_1} or unbounded data with sub-Gaussian assumption \citep{xu2020learning,lecue2013learning,loh2017statistical,ostrovskii2021finite}.
However, there are a lot of statistical models which do not satisfy any of the above three assumptions, see for instance \cite{zhang2022elastic} and \cite{chi2010local}.

When a distribution does not have an exponential moment, it is often called heavy-tailed; see \cite{resnick2007heavy}. However, there are many kinds of data, such as network and wealth distribution data, which only have finite $\beta$-th moment with $\beta \in (1,2)$. \cite{zhang2018ell_1} studied log-truncated M-estimator for least absolute deviation (LAD) regression under the assumption that the data have $2$nd moment, while \cite{chen2020generalized} extended their work to the data with $\beta$-th moment for $\beta \in (1,2)$.

{Most of minimization problems in machine learning have non-convex loss functions,  see \emph{mixture density estimation} in \citep{khamaru2019convergence}, \emph{the mixture of two linear regressions} in \citep{klusowski2019estimating}, \emph{truncated Cauchy non-negative matrix factorization} in \citep{guan2017truncated}, and regressions under deep neural networks (DNNs) in \cite{fan2021selective}.
In practice, the dimension of DNN regressions is usually much larger than the dimension of input, and the computation costs of training large neural networks may be huge \citep{frankle2018lottery}. To avoid training over-parameterized DNN, many works propose penalized DNN-based estimators for effectively learning the sparse DNN problems; see \cite{wen2021sparse,ohn2022nonconvex}.}

\subsection{Contributions}

This paper proposes a log-truncated M-estimator for a large family of statistical regressions and establishes its excess risk bounds under the condition that the data have $\beta$-th moment with $\beta \in (1,2)$, our \emph{contributions} are summarized as the following three aspects.

\textbf{A general function $\lambda(x)$ and the associated ridge regression}.
Table \ref{table:distributions} below lists the choices of $\lambda(x)$ which has been reported in literatures. In this paper, we propose a new $\lambda(x)$, under certain conditions on $\lambda$ and loss functions; we establish an error bound for the associated ridge regression, see Theorem \ref{thm1loss} below.  We allow the dimension $p$ to increase with the sample size $n$. Because the parameter set $\Theta$ is bounded, it seems that our ridge regression is more natural and reasonable.


 \textbf{Ridge regressions with special $\lambda(x)={|x|^{\beta}}/{\beta}$ and examples}. Taking $\lambda(x)={|x|^{\beta}}/{\beta}$ with $\beta \in (1,2)$ and applying Theorem \ref{thm1loss},
we establish a new log-truncated robust estimator in Theorem \ref{thm1loss2}, which not only extends the results about LAD estimators in  \cite{zhang2018ell_1} and \cite{chen2020generalized}, but also covers many other convex loss examples such as robust QR and robust GLMs. For GLMs, we obtain a general result for bounding excess risk and apply it to two typical classification and count data models: logistic regression and negative binomial regression.

 \textbf{High dimensional non-convex regressions with elastic net and DNN}. In the high dimension setting $p \gg n$, we propose a new robust elastic net estimator defined by \eqref{eq:cantoninet} below and obtain the error bound of the excess risk. As applications, we studied the non-convex regressions via DNN, and apply our results to study several typical regression problems such as LAD regression and logistic regression.
 In practice, the DNN regressions can be solved by some algorithms based on stochastic gradient descents (SGDs). Empirical studies including Boston housing and MNIST datasets are performed well by the proposed robust DNN regression models. We stress that Theorems \ref{thm1loss2} and \ref{thm:Elastic} below can be applied to many other non-convex regressions, e.g., robust two-component mixed linear regression and robust non-negative matrix factorization, in which one has to design specific algorithms rather than use SGDs.

\subsection{Related works}\label{se:RW}

\cite{catoni2012challenging} put forward a logarithm truncation for mean regression with finite $2$nd data and obtained an estimator whose confidence interval has a length comparable with that of the classical mean estimation with sub-Gaussian data. Since then, Catoni's idea has been extensively applied to study regressions and estimations with heavy-tailed data; see \cite{fan2017estimation,sun2020adaptive,wang2022new,sun2021we}.
\cite{zhang2018ell_1} used Catoni's truncation technique to study a LAD regression for the data with $2$nd moment and showed that
the associated estimator is consistent in the measure of excess risk. By modifying the logarithm truncation of Catoni, \cite{chen2020generalized} extended the results in \cite{zhang2018ell_1} to the data with $\beta$-th moment for $\beta \in (1,2)$. Due to the increasing applications whose data do not have $2$nd moment, more and more generalizations of Catoni's truncation have been proposed, see \cite{lam2021robust}, and \cite{lee2020optimal}.  \cite{xu2020learning} studied the excess risk bounds for learning with general non-convex truncated losses, in which $\lambda(x)=O(x^{\beta})$ with $\beta=1$ or $\beta \ge 2$. Under bounded input assumption, \cite{shen2021robust} studied non-asymptotic error bounds for deep neural networks for regression models with heavy-tailed error output having a finite $\beta$-th moment. For canonical GLMs, \cite{zhu2021taming} required $4$th moment condition on output to study the consistency property of their proposed robust estimators. For robust mean estimation, \cite{minsker2018sub}, \cite{lam2021robust} and \cite{lee2020optimal} also considered the extensions of $\psi$ for different motivations (see the table below).

{Both \cite{zhang2018ell_1} and  \cite{chen2020generalized} assumed that the parameters to be estimated are located in a compact set, they did not consider adding a penalty on their ERM problems.} In this paper, we propose a robust ridge regression and a robust elastic net regression, and derive their excess risk bounds. They can be applied to classical statistical models such as QR and GLMs, and to high dimensional non-convex learning problems such as DNN regressions.

 {Robust and sparse estimation for DNN learning has recently drawn a lot of attentions. \cite{taheri2021statistical} studied the $\ell_1$-regularized neural networks with the specific least square loss, while \cite{wen2021sparse} derived the risk bound for sparse DNNs regression by $L_{1,\infty}$-weight normalization under the bounded loss assumption. Under the heavy-tailed output, \cite{Lederer2020Risk,shen2021deep,fan2022noise} studied risk bounds for robust DNN linear regressions by assuming that the input data is bounded or fixed if the loss is LAD or Huber or Cauchy type. In addition to heavy-tailed output, our work first attempts to study the heavy-tailed input setting systematically. }

\subsection{Notations and organizations}
The following notations will be frequently used in the rest of this paper. Define the index set $[n]:=\{1,2,\cdots,n\}$ and let $\mathbb{N}$ be the non-negative integer set. The $\mathbb{R}_{+}$ denotes the set of positive real numbers. The r.v. is the shorthand for a random variable. Let $q \ge 1$ and $p \in \mathbb N$, for $\theta \in \mathbb{R}^{p}$, define ${\| \theta  \|_{q}}: ={( {\sum_{j = 1}^p {|\theta_j|^q} } )^{1/q}}$.  Define the unit $\ell_2$-norm ball $B_{2}^{p}(r):=\left\{x \in \mathbb{R}^{p}:\|\theta\|_{{2}} \leq r\right\}$ for $r>0$, and the $\ell_0$-norm ball $B_{0}^{p}(s):=\left\{x \in \mathbb{R}^{d}:\|\theta\|_0 \leq s \right\}$ for $s \in \mathbb N \cup \{0\}$. Let $\Theta \subset \R^p$, for an $\varepsilon>0$, $\mathcal{N}(\Theta,\varepsilon) \subset \mathbb{R}^p$ is an \emph{$\varepsilon$-net} of $\Theta$ if for all ${x} \in \Theta,$ there is a ${y} \in \mathcal{N}(\Theta,\varepsilon )$ such that $\|y-x\|_2 \le \varepsilon$. The \emph{covering number} ${N}(\Theta,\varepsilon )$ is the smallest number of closed balls centered at $\Theta$ with radii $\varepsilon$ whose union covers $\Theta$.

For a probability measure $\mu$ and a measurable function $f$, let $\|f\|_{L^{2}(\mu)}:=[\E_{X \sim \mu}f^2(X)]^{1/2}$ as long as the expectation is finite. Let $a>0$ and denote by ${L^{2}([0,a]^{p})}$ the square integrable function space with respect to domain $[0,a]^{p}$, and define the ${L^{2}}$-norm for a square integrable function $g$ as $\|g\|_{[0,a]^{p}}:=[\int_{[0,a]^{p}}g^2(x){\rm{d}}x]^{1/2}$.
Let $\lfloor x \rfloor$ be the largest integer strictly smaller than $x$. A function $f$ is in the $\gamma$-H{\"o}lder function class with smoothness index $\gamma>0$ if all partial derivatives of $f$ up to order $\lfloor\gamma\rfloor$ exist and are bounded, and the $\gamma$-H{\"o}lder function space with domain $D \subset \mathbb{R}^{p}$ and radius $R>0$ is defined as
\begin{small}
\begin{align}\label{eq:holder}
\mathcal{C}^{\gamma}(D, R)=\left\{f: D \rightarrow \mathbb{R}:\sum_{\alpha:\|\alpha\|_{1}<\gamma}\left\|\partial^{\alpha} f\right\|_{\infty}+\sum_{\alpha:\|\alpha\|_{1}=\lfloor\gamma\rfloor} \sup _{\substack{\mathbf{x}, \mathbf{y} \in D  \atop
\mathbf{x} \neq \mathbf{y}}} \frac{\left|\partial^{\alpha} f(\mathbf{x})-\partial^{\alpha} f(\mathbf{y})\right|}{|\mathbf{x}-\mathbf{y}|_{\infty}^{\gamma-\lfloor\gamma\rfloor}} \leq R\right\},
\end{align}
\end{small}
where $\partial^{\alpha}=\partial^{\alpha_{1}} \ldots \partial^{\alpha_{p}}$ is multi-index notation with $\alpha=\left(\alpha_{1}, \ldots, \alpha_{p}\right) \in \mathbb{N}^{p}$.  For two matrices $A$ and $B$ with compatible dimensions, denote $A \succ B$ if $B-A$ is positive definite. Let $\|A\|_{F}:= \sqrt{\sum_{i = 1}^m \sum_{j = 1}^n a_{ij}^2}$ be the Frobenius norm of a square matrix $A =(a_{ij}) \in \mathbb{R}^{m \times n}$.

The rest of the paper is organized as follows. Section \ref{Log-truncated} introduces the robust estimator based on the log-truncated loss function and provides three main theorems: Theorems \ref{thm1loss},\ref{thm1loss2} and \ref{thm:Elastic}, while Sections \ref{se:example} and \ref{se:dnn} give examples for Theorems \ref{thm1loss2} and \ref{thm:Elastic} respectively.
 Section \ref{simulation} includes simulations and real data analysis, which evaluate the effectiveness of the proposed log-truncated estimation for some regressions discussed in Section 3, and Section \ref{conclusion} provides some perspectives for the future study.

\section{Estimation with log-truncated loss and main results}\label{Log-truncated}

\subsection{Problem setup} \label{ssPS}
We assume that $\{(X_{i}, Y_{i})\}_{i=1}^n$ {are} a sequence of $\mathbb{R}^{d} \times \mathbb{R}$-valued independent identically distributed (i.i.d.) r.v.s and that each $(X_{i}, Y_{i})$ is a copy of r.v. $(X, Y)$. Denote loss function by ${l(y,x,\theta)}$, where $y \in \mathbb{R}$ is the output variable, $x \in \mathbb{R}^d$ is the input variable, and $\theta \in {\Theta}$ with $\Theta \subset \mathbb{R}^p$ being the hypothesis space. Define
$${{R}_l}(\theta ):=\E[{l({Y} , X, \theta )} ],$$
the true parameter set $\Theta^*\subset \mathbb{R}^p$ is defined as the collection of all minimizers of $R_l(\cdot)$, i.e.,
\begin{equation}\label{eq:ture}
\Theta^*:=\left\{\theta^* \in  \mathop {\arg\min }\limits_{\theta  \in {\Theta}} {{R}_l}(\theta )\right\}.
\end{equation}
Any $\theta^* \in \Theta^*$ is called true parameter, it may be not unique since our loss function $l(\cdot,\cdot, \theta)$ may be non-convex.

The estimator of $\theta^*$ based on the i.i.d. data $\{(X_{i}, Y_{i})\}_{i=1}^n$ is obtained by the following \emph{empirical risk minimization} (ERM, see \cite{koltchinskii2011oracle}) problem:
\begin{equation}\label{eq:ERM}
\bar\theta_n \in  \mathop {\arg\min }\limits_{\theta  \in {\Theta}} {{\hat R}_l}(\theta ) ~~{\rm with} \ \  {{\hat R}_l}(\theta):= \frac{1}{n}\sum\limits_{i = 1}^n {l({Y_i}, X_i, \theta )}.
\end{equation}
When the data are heavy tailed,  the estimator $\bar\theta_n$ in \eqref{eq:ERM} may not be a robust estimator, see for instance \cite{ostrovskii2021finite},   \cite{mathieu2021excess} and the references therein. Given an estimator $\hat \theta_n$, its excess risk bound is given by
$${R_l}(\hat \theta_{n} ) -\mathop {\inf}\limits_{\theta  \in {\Theta}} {R_l}({\theta}).$$

In this paper, we aim to study the statistical learning problems with heavy-tailed data by modifying the ERM. In the sequel, we focus on log-truncated loss under infinite variance data assumption.

\subsection{Log-truncated loss and our estimators}\label{se:Log-t}
\cite{catoni2012challenging} put forward a non-deceasing truncation function $\psi$ such that
\begin{equation}\label{eq:bounds}
-\log \left(1-x+\frac{x^{2}}{2}\right) \leq \psi(x) \leq \log \left(1+x+\frac{x^{2}}{2}\right),
\end{equation}
and obtained a robust mean estimator for i.i.d. data with finite variance.
The truncation function $\psi(x)$ not only reduces the value of the exponential-scaled outliers but also largely retains the data fluctuation in an unbounded way, whereas the classical bounded truncated M-functions often lose the information of the data with large values.

\begin{table}[!ht]
\centering
\caption{Excess risk bound guarantees under various functions  $\lambda(x)$ for the log-truncated losses and $\beta \in (1,2)$.}
\begin{tabular}{@{}llll@{}}
\toprule
Higher order functions  $\lambda(x)$     &    References      & Moment conditions                \\ \toprule
0    & \cite{xu2020learning} & $\E\left|X\right|^2<\infty$\\\hline
$\frac{1}{2}|x|^2$    &      \cite{catoni2012challenging}       & $\E\left|X\right|^2<\infty$        \\ \hline
{$\left(\frac{{\beta-1}}{{\beta}} \vee \sqrt{\frac{2-{\beta}}{\beta}}\right)|x|^{\beta}$}   &  \cite{minsker2018sub}  & $\E\left|X\right|^\beta<\infty$         \\  \hline
$\frac{1}{\beta}|x|^{\beta}$   &  \cite{chen2020generalized}  & $\E\left|X\right|^\beta<\infty$         \\  \hline
${\left[ {2{{(\frac{{2-{\beta}}}{{\beta-1}})}^{1 - 2/(1+\varepsilon) }} + {{(\frac{{2-{\beta}}}{{\beta-1}})}^{2 - 2/{\beta} }}} \right]^{ - {\beta} /2}}|x|^{\beta}$  &  \cite{lee2020optimal} & $\E\left|X\right|^\beta<\infty$       \\  \hline
${\beta}^{-\beta/ 2}({2-{\beta}})^{1-\beta/ 2}({\beta-1})^{(\beta-1)}|x|^\beta$     & \cite{lam2021robust} & $\E\left|X\right|^\beta<\infty$         \\ \hline
$\sum_{k=2}^{m} \frac{x^{k}}{ k !},~(m\ge 2)$     & \cite{xu2020learning} & $\sum_{k=2}^{m} \frac{\E\left|X\right|^k}{ k !}<\infty$\\\hline
 the function $\lambda(x)$ in (C.1) & this paper & ${{\E}}[\lambda(H_{Y,X})]<\infty$\\\bottomrule
\end{tabular}
\label{table:distributions}
\end{table}

We replace the function $x^2/2$ in \eqref{eq:bounds} with a function $\lambda(x)$ satisfying the following conditions:
\begin{itemize}
\item [\textbullet] (C.1) The function $\lambda(x): \R_+ \rightarrow \R_+$ is a continuous non-decreasing function such that
$\lim_{x \rightarrow \infty} \frac{\lambda(x)}{x}=\infty$.  Moreover, there exist some $c_2>0$ and a function $f: \R_+ \rightarrow \R_+$ such that
\begin{itemize}
\item (C.1.1) $\lambda(tx) \le f(t)\lambda(x)$ for all $t, x \in \R_+$, where $\lim_{t \rightarrow 0^{+}} {f(t)}/{t}=0$;
\item (C.1.2)
$\lambda(x+y)\le c_2[\lambda(x)+\lambda(y)]$ for all $x, y \in \R_+$.
\end{itemize}
 \end{itemize}
 We further replace $\psi$ in \eqref{eq:bounds}  with $\psi_\lambda$ which satisfies:
\begin{align}\label{eq:CANTON}
-\log \left[1-x+\lambda(|x|)\right]\leq {\psi_\lambda}(x) \leq \log \left[1+x+\lambda(|x|)\right], \quad \forall x \in \mathbb{R}
\end{align}

We assume that our parameter space $\Theta$ satsifies
\begin{itemize}
\item [\textbullet] (C.2):  The parameter space $\Theta \subseteq \mathbb{R}^{p}$ is convex and there exists some $r_n \in (0,\infty)$, which may depend on the size $n$ of the observed data, such that $\|\theta\|_{2} \le r_n,~\forall \theta \in \Theta$.
\end{itemize}
The condition (C.2) naturally induces a ridge penalty for the ERM problem \eqref{eq:ERM} as the following:
\begin{equation}\label{eq:cantonil2}
\hat \theta_{n} \in \mathop {\arg\min }\limits_{\theta  \in {\Theta}} \{{{\hat R}_{\psi_\lambda,l,\alpha}}(\theta )+\rho\|\theta\|_2^2\}\ \ \ {\rm with} \ \ \ \ \ \
{\hat R}_{\psi_\lambda,l,\alpha}(\theta):=\frac{1}{{n\alpha }}\sum\limits_{i = 1}^n {\psi_\lambda}[\alpha l({Y_i} , {{X}}_i,\theta )],
\end{equation}
where $\alpha>0$ is a \emph{robustification parameter} to be tuned, and $\rho>0$ is a \emph{penalty parameter} for $\ell_2$-regularization.

As we shall see below, the estimator defined by \eqref{eq:cantonil2}  can only work for the case of $p<n$ with $p=o(n/\log n)$ and thus rules out the high dimensional learning problems with $p>n$.
To solve this problem, we assume an $s_n$-sparsity condition:
\begin{equation} \label{e:sparsity}
\Theta^{s}:=\{\theta \in \Theta: ~\|\theta\|_{0} \le s_n\},
\end{equation}
and introduce an elastic net \citep{zou2005regularization} as follows:
\begin{equation}\label{eq:cantoninet}
\hat \theta_{n} \in \mathop {\arg\min }\limits_{\theta  \in {\Theta}} \{\frac{1}{{n\alpha }}\sum\limits_{i = 1}^n {\psi_\lambda}[\alpha l({Y_i} , {{X}}_i,\theta )]+\rho\|\theta\|_2^2+\gamma\|\theta\|_{1}\},
\end{equation}
where $\rho>0$ and $\gamma>0$ are both tuning parameter to be chosen later.



The $\alpha$ will be chosen according to the sample size $n$ and tend to $0$ as $n\rightarrow \infty$. In order to obtain the optimal ${\alpha }$ as deriving excess risk bounds, we assume the following conditions for further analysis:
\begin{itemize}
\item [\textbullet] (C.3) \textbf{Local Lipschitz condition}: $\exists$ {\emph{locally Lipschitz constant}} $H_{y,x}$ s.t. $l(y,x,\cdot)$ satisfies
\[|l(y,x,{\theta_2}) - l(y,x,{\theta _1})| \le H_{y,x}\|{\theta _2} - {\theta _1}\|_{2};{\theta _1},{\theta _2} \in \Theta.\]

\item [\textbullet] (C.4) \textbf{Moment condition}:  ${{\E}}[\lambda(H_{Y,X})]<\infty$.

\item [\textbullet] (C.5) \textbf{The existence of risk functions}: {$R_{\lambda \circ l}(\Theta):={\sup }_{\theta  \in \Theta } R_{\lambda \circ l}(\theta )<\infty$.}
 \end{itemize}
\begin{remark}
 The assumptions (C.1)-(C.5) hold for a large class of examples including classical regressions and high dimensional non-convex regressions via DNN, see concrete examples in the next sections.

(C.4) is essentially an assumption on the moments of $X$ and $Y$. For instance, as
$\lambda(x)={|x|^{\beta}}/{\beta}$ with $\beta \in (1,2)$, (C.4) implies that $H_{Y,X}$ has $\beta$ moments.

As the data satisfy the condition $E[l(Y,X,0)]<\infty$, which is true for all the examples in this paper, by (C.2), (C.3) and (C.4),
we immediately see that (C.5) holds.
\end{remark}

{For further use, we define the following risk function:
$${R_{\lambda \circ l}}({\theta}): ={{\E}}\{\lambda[ l({Y},X,\theta)]\}, \ \ \ \ \theta \in \Theta.$$}

\subsection{Main results}\label{Regression}
In this subsection, we state our main results, Theorems \ref{thm1loss}, \ref{thm1loss2} and \ref{thm:Elastic} below,  the first theorem is a general result for the ridge regression \eqref{eq:cantonil2}  under the assumptions (C.1)-(C.5), while the other two provide the error bounds of excess risks of the regressions \eqref{eq:cantonil2}  and \eqref{eq:cantoninet} as $\lambda(x)={|x|^{\beta}}/{\beta},~\beta \in (1,2)$.

\begin{theorem}\label{thm1loss}
Let $\hat \theta_{n}$ be defined by \eqref{eq:cantonil2}. For a $\delta \in (0,1/2)$ and $\kappa>0$, under (C.1)-(C.5), we have with probability at least $1 - 2\delta$
\begin{align*}
&~~~~{R_l}(\hat \theta_{n} ) -\mathop {\inf}\limits_{\theta  \in {\Theta}} {R_l}({\theta}) \\
& \le 2\kappa\{{{\E}}H_{Y,X} + \frac{{{c_2}f(\alpha \kappa)}}{\alpha \kappa}{{\E}}[\lambda(H_{Y,X})]\} +\frac{{(c_2+1)f(\alpha )}}{\alpha }R_{\lambda \circ l}(\Theta)+ \frac{1}{{n\alpha }}\log \frac{N(\Theta, \kappa)}{\delta^2}+\rho\|\Theta^*\|_2^2,
\end{align*}
where $c_2$ is a constant in (C.1.1); $f(t)$ is the function in (C.1.2); $\|{\Theta^*}\|_2:=\inf_{\theta^*\in \Theta^*}\|\theta^*\|_2$.

In particular, choose $\kappa=1/n$ and tune $\alpha$ accordingly, then with probability at least $1 - 2\delta$
\begin{equation}\label{eq:clossexcess}
{R_l}(\hat \theta_{n} ) -\mathop {\inf}\limits_{\theta  \in {\Theta}} {R_l}({\theta})  \le \frac{{2{{\E}}H_{Y,X}}}{{n}} + \frac{{{c_2}{{\E}}[\lambda(H_{Y,X})] }}{\alpha}f(\frac{\alpha}{n}) +\frac{2}{{n\alpha }}\log \frac{N(\Theta ,n^{-1} )}{\delta^2}+\rho\|\Theta^*\|_2^2,
\end{equation}
where $\alpha = {f^{ - 1}}\left( {\frac{1}{{n({c_2} + 1)}}[R_{\lambda \circ l}(\Theta)]^{ - 1}\log \frac{{N(\Theta ,n^{-1} )}}{{{\delta ^2}}}} \right)$.
\end{theorem}
\begin{remark}
The terms $\frac{{(c_2+1)f(\alpha )}}{\alpha }R_{\lambda \circ l}(\Theta)$ and $\frac{1}{{n\alpha }}\log \frac{{N({\Theta},\kappa )}}{{{\delta ^2}}}$ in Theorem \ref{thm1loss} can be viewed as variance and bias respectively. We choose the tuning parameter $\alpha$ by setting $\frac{{(c_2+1)f(\alpha )}}{\alpha }R_{\lambda \circ l}(\Theta) = \frac{1}{{n\alpha }}\log \frac{{N({\Theta},\kappa )}}{{{\delta ^2}}}$. Note that \cite{zhang2018ell_1} and  \cite{chen2020generalized} chose their $\alpha$ without this delicate consideration.
\end{remark}

Under infinite variance assumption, by Theorem \ref{thm1loss} we can derive our \emph{second main result}, in which we need the condition $p<n$ but allow $p$ to grow with $n$.

\begin{theorem}\label{thm1loss2}
Set $\lambda(x)={|x|^{\beta}}/{\beta},~\beta \in (1,2)$,
$\alpha = \frac{1}{n^{1/{\beta}}}\left[\frac{C_{\delta,n,r}(p)}{(2^{\beta-1} + 1)R_{\lambda \circ l}(\Theta)}\right]^{1/{\beta}}$
in Theorem~\ref{thm1loss} with $C_{\delta,n,r}(p):= {{\log ({\delta ^{ - 2}}) + p\log \left( {1 + 2n r_n } \right)}}$, and assume ${{\E}}H_{Y,X}^{\beta}<\infty$. Then, with probability at least $1 - 2\delta$, one has
\begin{align*}
&~~~~{R_l}(\hat \theta_{n} ) -\mathop {\inf}\limits_{\theta  \in {\Theta}} {R_l}({\theta})\\
&\le \frac{2{{\E}}H_{Y,X}}{n}+C_{\beta,R_{\lambda \circ l}} \left[\frac{C_{\delta,n,r}(p) }{n}\right]^{\frac{\beta-1}{\beta}}+\rho\|\Theta^*\|_2^2=O\left[{{{\left( {\frac{{p\log (nr_n)}}{n}} \right)}^{\frac{\beta-1}{\beta}}}} \right]+\rho\|\Theta^*\|_2^2,
\end{align*}
where $C_{\beta,R_{\lambda \circ l}}:= \left[{2({2^{\beta-1}} + 1)}{R_{\lambda \circ l}(\Theta)}+ \frac{{{2^{\beta-1}}}}{{ {\beta }}}{{\E}}H_{Y,X}^{\beta}\right]/[{{(2^{\beta-1} + 1){R_{\lambda \circ l}(\Theta)}}}]^{\frac{\beta-1}{\beta}}$. Additionally, if we replace ${{\E}}H_{Y,X}^{\beta}<\infty$ with
\begin{equation}\label{eq:pn}
{{\E}}H_{Y,X}=o(n),{{\E}}H_{Y,X}^{\beta} = O({n^{\beta}}).
\end{equation}
Then, the consistency of excess risk is valid, i.e.
${R_l}(\hat \theta_{n} ) -\mathop {\inf}\nolimits_{\theta  \in {\Theta}} {R_l}({\theta})=o_p(1).$
\end{theorem}

{
\begin{remark}
 The high probability upper bound of ${R_l}(\hat \theta_{n} ) -\mathop {\inf}\nolimits_{\theta  \in {\Theta}} {R_l}({\theta})$ in Theorem \ref{thm1loss2} tells us that the convergence rate of excess risk is of an order $O( {{( {\frac{{p\log (nr_n)}}{n}})}^{  (\beta-1)/{{\beta}}}} )$ if the regularization error $\rho\|\theta^*\|_2^2\lesssim {{( {\frac{{p\log (nr_n)}}{n}})}^{  (\beta-1)/{{\beta}}}}$, { and it gives the consistency of excess risk if $ {{( {\frac{{p\log (nr_n)}}{n}})}^{  (\beta-1)/{{\beta}}}}=o(1)$, i.e. $r_n=n^{-1}e^{o(n/p)}$.} For $\beta \in (1,2)$, the smaller $\beta$ is, the slower the convergence rate will be.
\end{remark}
\begin{remark}
 For condition \eqref{eq:pn}, we provide two examples. In quantile regressions, $H_{y,x}\propto \| x \|_2 \le \sqrt{d}\| x \|_{\infty}$;  as long as $d=o(n^2)$ and ${{\E}}\| X \|_{\infty}^{\beta}<\infty$, we have ${{\E}}H_{Y,X}\propto{{\E}}[\sqrt{d}\| X \|_{\infty}]=o(n)$ and ${{\E}}H_{Y,X}^{\beta}\propto{{\E}}[\sqrt{d}\| X \|_{\infty}]^{\beta}=O(n^{\beta})$ in condition \eqref{eq:pn}. Similarly, negative binomial loss satisfies condition \eqref{eq:pn} if $d=o(n^2)$, ${{\E}}Y^{2\beta}<\infty$ and ${{\E}}\| X \|_{\infty}^{2\beta}<\infty$, by $H_{y,x}\propto y\left\| x \right\|_2$.

\end{remark}

\begin{remark}
Theorem 1 in \cite{zhang2018ell_1} focused on the excess risk bound of robust LAD regression under ${{\E}}\| X \|^2<\infty$. As a special case of \eqref{eq:cantonil2}, while Theorem 4.1 in \cite{chen2020generalized} considered the excess risk bound of robust LAD regression that allows infinite variance of input using $\lambda(x)=|x|^{\beta}/{\beta}$ with $\beta \in (1,2)$. Theorem \ref{thm1loss2} extends these two results to a large class of loss functions, which include many other regressions.
\end{remark}

Our \emph{third main result} is the following theorem about the estimator of elastic net defined by \eqref{eq:cantoninet} under the $s_n$-sparsity condition \eqref{e:sparsity}.
\begin{theorem}\label{thm:Elastic}
Let ${\Theta^*}$ be defined by \eqref{eq:ture} and let $\hat \theta_n$ be given by \eqref{eq:cantoninet} with $\lambda(x)=|x|^{\beta}/{\beta}$ with $\beta \in (1,2)$. If
\begin{center}
$\alpha = \frac{1}{n^{1/{\beta}}}\left(\frac{\log ({\delta ^{ - 2}}/\sqrt{2 e s_n}) +s_n\log\left[(1+2nr_n){ep}/{s_n} \right]}{(2^{\beta-1} + 1)R_{\lambda \circ l}(\Theta )}\right)^{{1}/{\beta}}$ with (C.3), $R_{\lambda \circ l}(\Theta)<\infty$ and ${{\E}}H_{Y,X}^{\beta}<\infty$,
\end{center}
then with probability at least $1 - 2\delta$ one has
\begin{small}
\begin{align*}
{R_l}(\hat \theta_{n} ) -\mathop {\inf}\limits_{\theta  \in {\Theta}} {R_l}({\theta})\le \frac{2{{\E}}H_{Y,X}}{n}+\frac{C_{\beta,R_{\lambda \circ l}}}{n^{\frac{\beta-1}{\beta}}}\left({\log (\frac{\delta ^{ - 2}}{2 e s_n}) + s_n\log\left[(1+2nr_n)\frac{ep}{s_n} \right]}\right)^{\frac{\beta-1}{\beta}}+\|{\Theta^*}\|_{\rho,\gamma},
\end{align*}
\end{small}
where $C_{\beta,R_{\lambda \circ l}}$ is a constant given in Theorem \ref{thm1loss2}, and $\|{\Theta^*}\|_{\rho,\gamma}:=\inf_{\theta^*\in \Theta^*}(\rho\|\theta^*\|_2^2+\gamma\|\theta^*\|_1)$.
\end{theorem}
\begin{remark}
Suppose that $\Theta^*$ is a bounded set, Theorem \ref{thm:Elastic} implies a rate $O( {({ {{{s_n\log (nr_n)}}/{n}} )}^{(\beta-1)/\beta}})$ excess risk bound if
\begin{center}
$\rho \vee \gamma \lesssim {({ {{{s_n\log (nr_n)}}/{n}} )}^{(\beta-1)/\beta}}$,
\end{center}
 and it works for the high-dimensional setting $p \gg n$. Moreover, put ${({ {{{s_n\log (nr_n)}}/{n}} )}^{(\beta-1)/\beta}}=o(1)$, which implies the consistency of excess risk if $r_n=n^{-1}e^{o(n/s_n)}$.
\end{remark}


\section{Examples for Theorem \ref{thm1loss2} ($p<n$)}\label{se:example}
 This section provides examples of several robust regressions, which include quantile regression and GLMs. We assume that the data in this section has the finite $\beta$-th moment with $\beta \in (1,2)$. In the all the models in the section, the dimension of the input $X$ equals that of the parameter $\theta$, i.e. $d=p$.


\subsection{Robust quantile regressions}\label{se:quantile}
Consider
\begin{equation}\label{eq:QR}
Y_i= {{X}}_i^ \top \theta^*+\epsilon_i,(i=1, \ldots, n),
\end{equation}
where ${{X}}_i=(X_{i1},\ldots,X_{ip})^\top$ is the $i$-th stochastic design point in $\mathbb{R}^{p}$, and random errors $\epsilon_i$'s are i.i.d. and satisfy  $P \left(\epsilon_i < 0 | X_ { i } \right) = \tau$ for $0<\tau<1$. The unknown regression coefficient $\theta^*=(\theta^*_{1},\ldots,\theta^*_{p})^\top$ may depend on $\tau$, but we suppress such dependence for the notational simplicity. The conditional distribution of $Y$ given $x$ is $F(y\vert x) = P(Y \leq y\vert x)$ and the $\tau\mathrm { th }$ conditional quantile of $Y$ given $x$ is $Q_{y| x} (\tau)= \inf \{t: F(t\vert x) \geq \tau\}$. The problem of interest is to estimate the unknown slope coefficient $\theta^*$ by regressing the conditional quantile function
\begin{center}
$Q_{Y_i| X_i} (\tau)={{X}}_i^ \top \theta^*,(i=1, \ldots, n)$.
\end{center}
Recall that the loss function of quantile regression is
\begin{center}
$l(y , x,\theta )={\rho _\tau }(y - x^ \top \theta )$
 with $\rho_{\tau}(u)=u[\tau-I(u<0)]$,
 \end{center}
see more details in \cite{koenker1978regression}.

{Under the i.i.d. data $\{(X_{i}, Y_{i})\}_{i=1}^n$, we study the log-truncated estimator $\hat \theta_n$ for the quantile loss:
\begin{equation}\label{eq:M-estimatorQ}
\hat \theta_n : =\mathop {\arg\min }\limits_{\theta  \in {\Theta}}{{\hat R}_{\psi_\lambda,{\rho _\tau },\alpha}}(\theta ),
\end{equation}
where ${{\hat R}_{\psi_\lambda,{\rho _\tau },\alpha}}(\theta ):=\frac{1}{{n\alpha }}\sum_{i = 1}^n \psi_\lambda [\alpha {\rho _\tau }({Y_i} - {{X}}_i^ \top \theta )]$ and $\lambda(x)=\frac{1}{\beta}|x|^{\beta},~{\beta} \in (1,2)$. The \emph{tuning parameter} $\alpha$ will be specified. The true }parameter $\theta^*$ is defined as the minimizer
\begin{equation}\label{eq:tureQ}
\theta^* : =   \mathop {\arg\min }\limits_{\theta  \in {\Theta}} {{R}_{\rho _\tau }}(\theta ),
\end{equation}
where $ {{R}_{\rho _\tau }}(\theta ):=\E[ {{\rho _\tau }({Y} - X^ \top \theta )} ]$ for ${\theta  \in {\Theta}}$. Besides (C.2), {we further assume}
 \begin{itemize}
\item [\textbullet] (Q.1): ${{\E}}\left\| {{X}} \right\|_{{2}}^{\beta}<\infty$.
\item [\textbullet] (Q.2):  ${{ R}_{\lambda \circ \rho _\tau}}(\Theta):=\sup_{\theta  \in \Theta }{{ R}_{\lambda \circ \rho _\tau}}(\theta )<\infty$ with $ {{ R}_{\lambda \circ \rho _\tau}}(\theta ):=\E[ \lambda({\rho_\tau }({Y} - {{X}}^ \top \theta ))]$.
\end{itemize}

\begin{corollary}\label{Thm2}
Let $\tau \in (0,1)$, $\delta \in (0,1/2)$. Define $\hat \theta_n$ by \eqref{eq:M-estimatorQ}, and ${\theta^*}$ is given by \eqref{eq:tureQ}. Under (C.2), (Q.1) and (Q.2), if we put
$\alpha = \frac{1}{n^{1/{\beta}}}\left[\frac{C_{\delta,n,r}(p)}{({2^{\beta-1}} + 1){{ R}_{\lambda \circ \rho _\tau}}(\Theta)}\right]^{1/{\beta}}$. Then, with probability at least $1 - 2\delta$ one has
\begin{align*}
{R_l}(\hat \theta_n ) - {R_l}({\theta^*})\le \frac{{2{l_\tau }}{\E}{\left\| {{X}} \right\|_{{2}}}}{n}+C_{\beta,R_{\lambda \circ \rho _\tau}}\left[\frac{C_{\delta,n,r}(p) }{n}\right]^{\frac{\beta-1}{\beta}}+\rho\|\theta^*\|_2^2,
\end{align*}
where $C_{\beta,R_{\lambda \circ \rho _\tau}}:= \left[{2({2^{\beta-1}} + 1)}{R_{\lambda \circ \rho _\tau}(\Theta)}+ \frac{{{2^{\beta-1}}l_\tau^\beta}}{{ {\beta }}}{{\E}}\left\| {{X}} \right\|_{{2}}^{\beta}\right]/[{{(2^{\beta-1} + 1){R_{\lambda \circ \rho _\tau}(\Theta)}}}]^{\frac{\beta-1}{\beta}}$ and $l_\tau:=\max \{1+ \tau ,2 - \tau \}$ .
\end{corollary}
 \cite[Section 4.1.2]{Koenker2005} stressed that {$2$nd moment condition of the input }is required to show the consistency for the ERM estimator $\bar \theta_n : = {\arg\min }_{\theta  \in {\Theta}} \frac{1}{{n }}\sum_{i = 1}^n {\rho _\tau }({Y_i} - {{X}}_i^ \top \theta )$, so our new method is essential in quantile regression for the data with $\beta$-th moment ($1<\beta<2$).


\subsection{Robust generalized linear models (GLMs)}\label{se:rglms}
We consider the general loss function {of GLMs \citep{mccullagh1989generalized} as below}. In this part, we assume that $\{ Y_i\}_{i=1}^{n}$ satisfy some moment conditions rather than put specific conditions on its distribution (as in the classical GLMs).

Let $u(\cdot)$ be a known link function. Consider the quasi-GLMs loss function:
\begin{equation}\label{eq-log-like}
{{\hat R}_l}(\theta ):=\frac{1}{n}\sum_{i=1}^{n}l({Y_{i}} ,X_{i}^ \top \theta ),~\theta \in \mathbb{R}^{p}
\end{equation}
where the loss function is given by $l(y ,x^\top \theta ):=k(x^ \top \theta)-y u(x^ \top \theta)$ with $k(t):=b(u (t))$ for a given function $b(\cdot)$. If $u(t)=t$, we say the quasi-GLMs has the \emph{canonical link}.

 Let $\alpha$ be a tuning parameter to be specified, the log-truncated robust estimator $\hat \theta_n$ for quasi-GLMs is
\begin{equation}\label{eq:M-estimatorglm}
\hat \theta_n : =\mathop {\arg\min }\limits_{\theta  \in {\Theta}}{{\hat R}_{\psi_\lambda,l,\alpha}}(\theta ),
\end{equation}
where ${{\hat R}_{\psi_\lambda,l,\alpha}}(\theta ):=\frac{1}{{n\alpha }}\sum_{i = 1}^n \psi_\lambda [\alpha (k(X_{i}^ \top \theta)-Y_{i} u(X_{i}^ \top \theta))]$ with $\psi_\lambda(x):={\rm{sign}}(x)\log (1 + |x| +{\beta^{-1}}|x|^{\beta})$. We assume the GLMs-related conditions:

\begin{itemize}

\item [\textbullet] (G.1): {Assume that $u(\cdot)$ is continuous differentiable and $\dot u(\cdot)\ge 0$, there exist a positive constant $A$ and a positive function ${g_A}(\cdot)$:}
\[0 \le \dot u(x^ \top \theta ) \le {g_A}(x),~\text{for}~\| \theta\|_2\le A.\]
\item [\textbullet](G.2): {Given a function $ b(\cdot)$ such that $\ddot b(\cdot)>0$, suppose that  $k(\cdot)$ is continuous differentiable and $k(\cdot)\ge 0$, there exist a positive function ${h_A}(\cdot)$:}
\[0<\dot k(x^ \top \theta) \le {h_A}(x),~\text{for}~\| \theta\|_2\le A.\]

\item [\textbullet] (G.3): Mixed moment condition: ${\E}|[|Y|{g_A}(X)+{h_A}(X)]\left\| X \right\|_2|^{\beta} <\infty$.

\item [\textbullet] (G.4): Assume $\sigma _R: ={\sup }_{\theta  \in {\Theta}} {{\E}}[k(X^ \top \theta ) - Yu(X^ \top \theta )]^{\beta}/{\beta}<\infty$.
\end{itemize}

The conditions (G.1), (G.2) and (G.3) imply (C.4), while (G.4) implies (C.5); see Remark \ref{eq:GLMS} in Appendix for more discussions for (G.1) and (G.2). Theorem \ref{thm1loss2} is applicable to obtain the following result.
\begin{corollary}\label{eq:GLMThm2}
Let $\theta^* =  \mathop {\arg\min }\nolimits_{\theta  \in {\Theta}} \E[{l(y, x,\theta )} ]$ with loss $l(y, x,\theta ):=k(x^ \top \theta)-y u(x^ \top \theta)$ defined in \eqref{eq-log-like}, and $\hat \theta_n$ is given by  \eqref{eq:M-estimatorglm}. Under (C.2), (G.1)-(G.4), if
$\alpha = \frac{1}{n^{1/{\beta}}}\left[\frac{C_{\delta,r}(p) }{(2^{\beta-1} + 1)\sigma_R}\right]^{1/{\beta}}$, then with probability at least $1 - 2\delta$
\begin{align}\label{eq:coglm}
{R_l}(\hat \theta_n ) - {R_l}({\theta^*})&\le {{\E}}\{[|Y|{g_{r_n}}(X)+{h_{r_n}}(X)]\left\| X \right\|_2\}\frac{2}{n} +C_{\beta,\sigma _R}\left[\frac{C_{\delta,r}(p) }{n}\right]^{\frac{\beta-1}{\beta}}+\rho\|\theta^*\|_2^2,
\end{align}
where $C_{\beta,\sigma _R}:= \left[{2({2^{\beta-1}} + 1)}\sigma _R+ \frac{{{2^{\beta-1}}}}{{ {\beta }}}{{\E}} |[|Y|{g_{r_n}}(X)+{h_{r_n}}(X)]\left\| X \right\|_2|^{\beta}\right]/[{{(2^{\beta-1} + 1)\sigma _R}}]^{\frac{\beta-1}{\beta}}$.
\end{corollary}

Corollary \ref{eq:GLMThm2} can be applied to the following two examples, robust logistic regression and robust negative binomial regression; see Appendix \ref{se:GLMap} for derivations.

\begin{example}[Robust logistic regression]
The output in logistic regression can take only two values: ``0, 1''.  Formally, let $Y_i$'s $\in\{0,1\}$ be the random outputs and ${\theta ^{\rm{*}}}$ be a $p\times1$ vector of unknown regression coefficients belonging to a compact subset of $\mathbb{R}^p$. Given $n$ random input $X_i$'s $\in\mathbb{R}^{n \times p}$, the logistic regression assumes $P({Y_i} = 1|{X_i};{\theta ^{\rm{*}}}): = \frac{{{e^{X_i^ \top {\theta ^{\rm{*}}}}}}}{{1 + {e^{X_i^ \top {\theta ^{\rm{*}}}}}}}.$ The empirical loss function of logistic regression is
\begin{center}
 $\hat R_l(\theta ) = \frac{-1}{n}\sum_{i = 1}^n {[{Y_i}X_i^ \top {\theta} - \log (1 + {e^{X_i^ \top {\theta}}})} ]$.
\end{center}
Note that ${H}_{y,x}=2\left\| x \right\|_2$  in Corollary \ref{eq:GLMThm2} under logistic loss. To obtain the finite excess risk \eqref{eq:coglm}, the robust logistic regression requires the moment condition
\begin{center}
${\E} \| {{X} }\|_2^{\beta}< \infty.$
\end{center}
\end{example}

For modeling count data regressions, {the classical Poisson regression as the canonical link GLMs has equal dispersion assumption (i.e. $\E(Y|X)=\mathrm{Var}(Y|X)$), which has little practical motivation.} Nevertheless, it motivates us to study the more flexible count data regressions, as shown below.

\begin{example}[Robust negative binomial regression]
As a generalization of Poisson regression,  negative binomial regression (NBR) relaxes the equadispersion assumption to the quadratic relationship between the mean and variance of the responses. NBR assumes that the overdispersed {responses $\{Y_i\}_{i=1}^n$} are modelled by two-parameter negative binomial distribution with the connection of covariates:
$P({Y_i=y}|\theta,\mu_{i})=\frac{\Gamma(\eta+y)}{\Gamma(\eta)y!}(\frac{\mu_{i}}{\eta+\mu_{i}})^{y}(\frac{\eta}{\eta+\mu_{i}})^{\eta}~\text{with}~\log\mu_{i}=X_i^ \top \theta,$
where $\eta>0$ is the known dispersion parameter, which can be estimated previously. One has ${\E}({Y_i}| {X}_i ) = {\mu _i} \le {\rm{Var}}({Y_i}| X_i ) = {\mu _i} + {{\mu _i^2}}/{\eta }$. The NBR empirical loss function is
\begin{center}
${{\hat R}_l}(\theta )=\frac{-1}{n}\sum\limits_{i=1}^{n}\{Y_{i}[{X}_{i}^ \top \theta-\log(\eta+e^{{X}_{i}^ \top \theta})]-\eta\log(\eta+e^{{X}_{i}^ \top \theta})\},$
\end{center}
see \cite{zhang2022elastic} for details. In Corollary \ref{eq:GLMThm2}, NBR loss has  ${H}_{y,x}=(y+\eta)\left\| x \right\|_2$. Note that there are no assumptions for the distribution of output, and it only requires the moment conditions
\begin{center}
${\E}\| {{X}{Y}} \|_2^{\beta} < \infty$ and ${\E} \| {{X} }\|_2^{\beta}< \infty$
\end{center}
to guarantee the excess risk bound \eqref{eq:coglm}.
\end{example}

\section{Examples for Theorem \ref{thm:Elastic} ($p>n$): non-convex regressions via DNN}\label{se:dnn}
 In many statistical learning problems, loss functions are non-convex, whereby the associated ERMs have multiple local minima; see \cite{guan2017truncated,chen2021bridging,klusowski2019estimating}. Regressions via DNN is a large family of highly non-convex learning problems due to the multiple compositions of activation functions. In this section, we shall apply Theorem \ref{thm:Elastic} to study high dimensional non-convex regressions via DNN.


We consider the {DNN} function class as follows: 
\begin{equation}\label{eq:NN}
\mathcal{N} \mathcal{N}(N, L):= \big\{ f_\theta({x}) =W_{L} \sigma_{L}\left(W_{L-1} \sigma_{L-1}\left(\ldots W_{1}\sigma_{1}\left(W_{0} {x}\right)\right)\right) | \, \theta: = ( W_0, \ldots, W_L )\big\},
\end{equation}
where $W_{j} \in \mathbb{R}^{N_{j} \times N_{j+1}}$ for $j=0,1,\dots,L-1$ with $N_{0}=d$. Here $L$ represents the depth of this class of {DNNs}, each activation function $\sigma_j: \mathbb{R}^{N_j} \to \mathbb{R}^{N_j}, j=1,2,\dots,L$, and $\theta$ is the vectorized parameter consisting of weighted matrices with the width $N=\max \left\{N_{1}, \ldots, N_{L}\right\}$; see \cite{fan2021selective} for details.

For i.i.d. observations $\{  (X_i, Y_i) \}_{i = 1}^n$ and a given loss function $l(\cdot,\cdot)$, the risk function is
$$R_l(f)= \E [ l ( Y,f({X}) )], \ \ \ \ {\rm for \ a \ function} \ \ f: \R^d \rightarrow \R.$$ In general, the true function $f^*$ belongs to a certain function family and is defined by \citep{fan2021selective}
\begin{equation}\label{ture}
    f^*\in \mathop{\rm{argmin}}\limits_{f}R_l(f).
\end{equation}
From DNN function class \eqref{eq:NN} with parameter space $\Theta  \subset \mathbb{R}^{\sum_{l=0}^{L} N_{l+1} N_{l}}$, we define $\theta_{\mathcal{N}}^*$ as
\begin{equation}\label{NN.ture}
\theta_{\mathcal{N}}^*\in \mathop{\rm{argmin}}\limits_{\theta \in \Theta}R_l(f_\theta)~\text{for}~f_\theta \in \mathcal{N} \mathcal{N}(N, L).
\end{equation}
Note that $p={\sum_{l=0}^{L} N_{l+1} N_{l}}$ and $d=N_0$ in this case. Denote $\Theta_{\mathcal{N}}^*:=\{\theta_{\mathcal{N}}^* \in  \mathop {\arg\min }\nolimits_{\theta  \in {\Theta}} {{R}_l}(f_\theta)\}$ for $f_\theta \in \mathcal{N} \mathcal{N}(N, L)$.

Now we fit the regression problem \eqref{NN.ture} into the framework of the elastic net regression \eqref{eq:cantoninet}, whose corresponding form is as follows:
\begin{equation}\label{eq:NNN}
\hat \theta_n \in \mathop {\arg\min }\limits_{\theta  \in {\Theta}}\left\{\frac{1}{{n\alpha }}\sum_{i = 1}^n \psi_\lambda  \big(\alpha l ( Y_i, f_\theta(X_i) ) \big)+\rho\|\theta\|_2^2+\gamma\|\theta\|_{1}\right\},
\end{equation}
where $\rho~\text{and}~\gamma>0$ are \emph{penalty parameters}.
We have the following assumptions:
 \begin{itemize}
\item [\textbullet] (D.1): Assume that the $l(\cdot,\cdot)$ satisfies Lipschitz condition with a \emph{Lipschitz constant} $D_{x,y}$:
\begin{center}
$|l(y,f_{\theta _2}(x)) - l(y,f_{\theta _1}(x))|\le D_{x,y}{{|f_{\theta _2}(x)-f_{\theta _1}(x)|}}~\text{for}~\theta_1,~\theta_2 \in
\Theta$,
\end{center}
where the DNN function is indexed by the \emph{$s_n$-sparse parameter space}
\begin{equation}\label{eq:l1}
\Theta:=\{\theta:=( W_1, \ldots, W_L ):\|\theta\|_{2} \le r_n,~\|\theta\|_{0} \le s_n\} \subseteq \mathbb{R}^p,
\end{equation}
{ where $r_n, s_n$ are both allowed to increase with the size $n$ of the observed data.} Further assume that there exists some $W>0$ so that $\max_{0 \le j \le L}{\sigma _{\max }}({W_j}) \le W$, where ${\sigma _{\max }}({W_j})$ is the largest singular value of $W_j$.

\item [\textbullet] (D.2):  We assume that the true function $f^*$ belongs to the H{\"o}lder function class
$f^* \in \mathcal{C}^\gamma([0,a]^{d},B)$ with smoothness index $\gamma$, where $a$ and $B$ are both positive constants (see definition of $\mathcal{C}^\gamma([0,a]^{d},B)$ in \eqref{eq:holder}).
\item [\textbullet] (D.3): For a fixed $d$, we assume that: $\E\|X\|_{\infty} \le b$ for $X \in \mathbb{R}^d$ and $b>0$. Moreover, $\E D_{X,Y}^2 <\infty$.
\end{itemize}

\begin{remark}
The Lipschitz condition (D.1), together with Proposition 6 in \cite{taheri2021statistical}, immediately implies (C.3). In the real world applications, the input data are usually transformed into a bounded interval $[0,a]^d$, this motivates the assumption (D.2). 

\end{remark}

\begin{theorem}\label{eq:DNNsThm3}
Assume that (D.1)-(D.3) hold and that ${{\E}}|\|X\|_{2}D_{X,Y}|^{\beta}<\infty, {\beta} \in (1,\infty)$. Let $f^*$ be defined by \eqref{ture}, and let $\hat \theta_n$ be given by \eqref{eq:NNN} with $\lambda(x)=|x|^{\beta}/{\beta}$ and $\beta \in (1,2)$. For a $\delta \in (0,1/2)$, if we choose
\begin{center}
$\alpha = \frac{1}{n^{1/{\beta}}}\left(\frac{\log ({\delta ^{ - 2}}/\sqrt{2 e s_n}) + s_n\log\left[(1+2nr_n){ep}/{s_n} \right]}{(2^{\beta-1} + 1)R_{\lambda \circ l}(\Theta )}\right)^{{1}/{\beta}}$,
\end{center}
then with probability at least $1 - 2\delta$ we have
\begin{small}
\begin{align}\label{eq:DNNER}
{R}_l( f_{\hat \theta_n})- {R}_l (f	^*) \le E_{1}+E_{2}+E_{3}+E_4,
\end{align}
\end{small}
for any $\|f-f^*\|_{\infty}\le F<\infty$ with ${f \in \mathcal{N}\mathcal{N}(N, L)}$, where
$$E_{1}:=\frac{4W^{L} \sqrt{L}}{n}{\E}[\|X\|_{2}|D_{X,Y}|],~E_{2}:=\inf_{\theta_{\mathcal{N}}^*\in \Theta_{\mathcal{N}}^*}(\rho\|\theta_{\mathcal{N}}^*\|_2^2+\gamma\|\theta_{\mathcal{N}}^*\|_1),$$
$$E_{3}:=2\sqrt{\E D_{X,Y}^2}{b^{\frac{\gamma}{2\gamma+1}}}{\left[\frac{(2B+1)\left(1+d^{2}+\gamma^{2}\right)  6^d MF^2}{ 2^{m}}+\frac{ 3^{\gamma}BF^2} {N^{{\gamma}/d}}\right]^{\frac{1}{2\gamma+1}}},$$
$$E_{4}:=\frac{F_{\beta,L,W}(R_{\lambda \circ l})}{n^{{(\beta-1)}/{\beta}}}\left({\log (\frac{\delta ^{ - 2}}{\sqrt{2 e s_n}}) + s_n\log\left[(1+2nr_n)\frac{ep}{s_n} \right]}\right)^{{(\beta-1)}/{\beta}},$$
with
$F_{\beta,L,W}(R_{\lambda \circ l}):=\left[2{(2^{\beta-1} + 1)R_{\lambda \circ l}(\Theta )}+\frac{(4W^{L} \sqrt{L})^\beta}{2\beta }{{{\E}} |\|X\|_{2}D_{X,Y}|^\beta}\right]/{[(2^{\beta-1} + 1)R_{\lambda \circ l}(\Theta )]^{{(\beta-1)}/{\beta}}}$, the integer $m\ge 1$ and $M\ge (\gamma+1)^d \vee(B+1) e^d$, and $L\le 8+(m+5)\left(1+\left\lceil\log _{2}(d \vee \gamma)\right\rceil\right)$.
\end{theorem}

\begin{remark}
$E_{1}, E_2, E_3, E_4$ can be interpreted as the bias, the penalization error, the error between the true function $f^*$ and the DNN, and the statistical error, respectively. For $E_3$, we require sparsity $s \leq 141(d+\gamma+1)^{3+d} M(m+6)$ in \eqref{eq:l1}.
\end{remark}

The following corollary gives an upper bound for the depth and a lower bound for the width of a DNN designed to realize the regression \eqref{eq:NNN}.
\begin{corollary}\label{coroln}
Under the setting in Theorem \ref{eq:DNNsThm3},  if  \emph{depth-sample} and \emph{width-sample} of DNNs, which may increase with $n$, satisfy the following conditions:
\begin{equation}\label{eq:Nn}
{L_n} \lesssim  \frac{\log n+{\log( {{s_n\log(npr_n)}})}}{\log W},~~N_n \gtrsim b^d\left[\frac{n}{{{s_n\log(npr_n)}}} \right]^{{\frac{d(2\gamma+1)}{\gamma}}\cdot\frac{\beta-1}{\beta}}
\end{equation}
and \emph{order of tuning parameters} $\rho \vee \gamma \lesssim  {({ {{{s_n\log(npr_n)}}/{n}} )}^{(\beta-1)/\beta}}$, then
\begin{center}
${R}_l( f_{\hat \theta_n})- {R}_l (f	^*)  \le C_\delta ( {({ {{{s\log(npr_n)}}/{n}} )}^{(\beta-1)/\beta}})$
\end{center}
 with probability at least $1 - 2\delta$, for a certain constant $C_\delta>0$.
\end{corollary}


We finish this section by giving three concrete examples of robust DNN regressions, and the models therein will be used in simulations or real data studies.
\begin{example}[Robust DNN LAD regression]\label{eg:DNNLAD1}
Suppose that i.i.d. observations $\{(Y_i, X_i)\sim (Y, X)\}_{i=1}^n \in \mathbb{R}\times \mathbb{R}^d$ satisfy
\begin{center}
$Y_i=f^*(X_{i})+e_i, \quad \E(e_i | X_i)=0,~i=1,2,\cdots,n,$
\end{center}
{where $\{e_i \sim e\}_{i=1}^n$ are i.i.d. noise. Similar to QR with ${\tau}=0.5$}, the robust DNN LAD regression problem \eqref{eq:NNN} has loss function $l(x,y,\theta)=|y - f_\theta(x)|$ with $f_{\theta}\in \mathcal{N}\mathcal{N}(N, L)$. Thus $ {{ R}_{\lambda \circ l}}(\theta ):={\beta^{-1}}\E|Y - f_\theta(X)|^{\beta}=\E|e|^{\beta}/\beta$ and we have ${{ R}_{\lambda \circ l}}(\theta )<\infty$ in Theorem \ref{eq:DNNsThm3} if
\begin{center}
$\E|e|^{\beta}<\infty$.
\end{center}
The LAD regression loss has Lipschitz constant $D_{x,y}=1$ and thus $H_{y,x}=2W^{L} \sqrt{L}\|{x}\|_{2}$ in Theorem \ref{eq:DNNsThm3} also requires
\begin{equation}\label{eq:heavy}
{\E}\left\| X \right\|_2^{\beta}<\infty.
\end{equation}
\end{example}
Recently, \cite{padilla2020quantile,shen2021deep} studied the DNN quantile regression with fixed inputs, and their estimators are only robust for output. Their setting can not deal with the robustness of the random input with heavy-tail condition \eqref{eq:heavy}.

\begin{example}[Robust DNN logistic regression]\label{eg:DNNLAD2}  
Assume that  i.i.d.  observations $\{(Y_i, X_i)\sim (Y, X)\}_{i=1}^n \in \{0,1\}\times \mathbb{R}^d$ satisfy
\begin{equation}\label{eq:1DNN}
P({Y_i} = 1|{X_i}): = \frac{{{e^{f^*(X_i)}}}}{{1 + {e^{f^*(X_i)}}}},~~P({Y_i} = 0|{X_i})=1-P({Y_i} = 1|{X_i}).
\end{equation}
 The robust DNN logistic regression problem \eqref{eq:NNN} has loss function $l(x,y,\theta)=-[y{f_{\theta}(x)}  - \log (1 + {e^{f_{\theta}(x)}})]$ with $f_{\theta}\in \mathcal{N}\mathcal{N}(N, L)$. Theorem \ref{eq:DNNsThm3} requires $R_{\lambda \circ l}(\Theta):= {\beta^{-1}} {\sup_{\theta  \in \Theta }}\E[Y{f_{\theta}(X)}  - \log (1 + {e^{f_{\theta}(X)}})]^{\beta}\lesssim{\sup\nolimits_{\theta  \in \Theta }}\E|{f_{\theta}(X)}|^{\beta} < \infty$. For logistic loss, it gives $D_{x,y}=y+1 \le 2$ and $H_{y,x}=4W^{L} \sqrt{L}\|{x}\|_{2}$. The moment conditions for robust DNN logistic regression are
\begin{center}
 ${\E}\left\| X \right\|_2^{\beta}<\infty$ and ${\sup\limits_{\theta  \in \Theta }}\E|{f_{\theta}(X)}|^{\beta} < \infty$.
\end{center}
\end{example}
\begin{example}[Robust DNN NBR]\label{eg:DNNLAD3}
Let $\eta>0$ be  the known dispersion parameter. Suppose that i.i.d. observations $\{(Y_i, X_i)\sim (Y, X)\}_{i=1}^n \in \mathbb{Z}\times \mathbb{R}^d$ satisfy
\begin{equation*}\label{eq:NBDNN} P({Y_i=y}|X_i)=\frac{\Gamma(\eta+y)}{\Gamma(\eta)y!}(\frac{\mu_{i}}{\eta+\mu_{i}})^{y}(\frac{\eta}{\eta+\mu_{i}})^{\eta},~\text{with}~\log\mu_{i}=f^* (X_i).
\end{equation*}
The robust DNN NBR problem \eqref{eq:NNN} has loss function $l(x,y,\theta)=-y[{f_{\theta}(x)}  - \log (\eta + {e^{f_{\theta}(x)}})]-\eta\log(\eta + {e^{f_{\theta}(x)}})$. If ${\sup\nolimits_{\theta  \in \Theta }}{{\E}}|Y{f_{\theta}(X)}|^{\beta}< \infty $, then Theorem \ref{eq:DNNsThm3} gives
\begin{align*}
R_{\lambda \circ l}(\Theta):& ={\beta^{-1}} {\sup_{\theta  \in \Theta }}{{\E}}{\{Y[{f_{\theta}(X)}  - \log (\eta + {e^{f_{\theta}(X)}})]-\eta\log(\eta + {e^{f_{\theta}(X)}})\}^{\beta}}\\
&\lesssim {\sup\limits_{\theta  \in \Theta }}{{\E}}|Y{f_{\theta}(X)}|^{\beta}+ {\sup\limits_{\theta  \in \Theta }}{{\E}}|{f_{\theta}(X)}|^{\beta}< \infty.
\end{align*}
For NBR loss, we get $D_{x,y}=y+\eta $ and $H_{y,x}=2W^{L} \sqrt{L}\|{x}\|_{2}(y+\eta)$. (D.3) needs ${\E}(Y+\eta)^2<\infty$. In summary, the required moment conditions are
\begin{center}
 ${\E}(Y+\eta)^2<\infty$, ${\E} [\|X\|_{2}(Y+\eta)]^{\beta}<\infty$ and ${\sup\limits_{\theta  \in \Theta }} {{\E}}|Y{f_{\theta}(X)}|^{\beta} < \infty$.
\end{center}
\end{example}
Note that if $L=0$ and $\theta =  W_0 \in \mathbb{R}^d$, Theorem \ref{eq:DNNsThm3} cannot work since the proof of excess risk bound requires $L\ge 1$. In this degenerate DNN regression, it is just the common robust ERM problem with elastic net penalty \eqref{eq:cantoninet}. Under $s_n$-sparse parameter space, we obtain the excess risk bound in Theorem \ref{thm:Elastic} for this special and important parametric regressions when $d>n$.

\section{Simulation and real data studies} \label{simulation}
	\subsection{Simulations on normal regression models}
In this part, by stochastic gradient descent (SGD) algorithms, we illustrate the effectiveness of regressions based on log-truncated ERM by the numerical experiments of ordinary logistic regression and negative binomial regression. The elastic net DNN regressions are optimized by the Adam algorithm (SGD-based algorithm, \cite{kingma2014adam}), which is an extension of SGD. Moreover, the Adam algorithm is more computationally efficient than SGD under a large number of parameters, and it has few memory requirements.

	\subsubsection{SGD}\label{se:sgdsimulation}
\begin{itemize}
	\item[1.] \textbf{SGD for our estimation}

Let us consider a regularized optimization with a given penalty function $\Omega(\theta)$:
\begin{equation}\label{eq:ridge}
	\hat \theta_n(\alpha,\rho): =\mathop {\arg\min }\limits_{\theta \in {\Theta}}\{{{\hat R}_{\psi_\lambda,l,\alpha}}(\theta )+\Omega(\theta)\},
\end{equation}
where ${{\hat R}_{\psi_\lambda,l,\alpha}}(\theta ):=\frac{1}{{n\alpha }}\sum_{i = 1}^n \psi_\lambda[\alpha l({Y_i} , {{X}}_i,\theta )]$ with $l(y , x,\theta )$ being some specific losses, and $\alpha>0$ is another tuning parameter to be chosen.
\begin{itemize}
	\item For $\Omega(\theta)=\rho\|\theta\|_2^2$, this is a $\ell_2$-regularization, where $\rho>0$ is the penalty parameter.
	
	\item For $\Omega(\theta)=\rho\|\theta\|_2^2+\gamma\|\theta\|_{1}$, this is an elastic net regularization, where $\rho,\gamma>0$ are two penalty parameters.
\end{itemize}

In practice, this optimization problem is solved by stochastic gradient descent (SGD) as the following:
\begin{equation}\label{eq:sgd}
	\theta_{t+1}=\theta_{t}-\frac{r_t}{\alpha}\nabla_{\theta}\{\psi_{\lambda}[\alpha l(Y_{i_t}, X_{i_t}, \theta_{t})]+\Omega(\theta_{t})\}, t=0,1,2,\cdots,
\end{equation}
where $i_t$ denotes a random sampled index, $\{r_t\}$ is the learning rate. For $\ell_2$-regularization, we employ five-fold cross validation (CV) method to find the optimal parameter pair $(\alpha, \rho)$ in a certain effective subset of $\R_+^2$.
For the elastic net regularized DNN model, we select the optimal parameters $(\alpha, \beta, \gamma)$ by evaluating the performances of their corresponding training models on validation data set whose size is $1/5$ of the size of the training set.

\item[2.] \textbf{SGD for the comparative estimations}

For the standard ridge regression without truncation, the corresponding optimization problem is
\begin{equation}\label{eq:ridge-1}
	\hat \theta_n(\rho): =\mathop {\arg\min }\limits_{\theta \in {\Theta}}\{\frac{1}{{n}}\sum_{i = 1}^n l({Y_i} , {{X}}_i,\theta)+\Omega(\theta)\},
\end{equation}
where $\rho$ is the penalty parameter for regularization, this optimization problem can be solved by the following SGD:
\begin{equation}\label{eq:sgd}
	\theta_{t+1}=\theta_{t}-r_t \nabla_{\theta}\{l(Y_{i_t}, X_{i_t}, \theta_{t})+\Omega(\theta_{t})\}, t=0,1,2,\cdots.
\end{equation}


We also consider the Cauchy log-truncated function in Table 1, which has the form $\phi_{\alpha}(x)=\alpha\log(1+\frac{x}{\alpha})$. Similarly, the estimator $\hat \theta_n^C$ is solved by
\begin{equation*}\label{eq:cauchy-1}
	\hat \theta_n^C: =\mathop {\arg\min }\limits_{\theta \in {\Theta}}\{\frac{1}{{n}}\sum_{i = 1}^n \phi_{\alpha}\big(l({Y_i} , {{X}}_i,\theta)\big)+\Omega(\theta)\}.
\end{equation*}
The corresponding SGD iterations are
\begin{equation*}\label{eq:cauchy-sgd}
	\theta_{t+1}=\theta_{t}-r_t \nabla_{\theta}\{\phi_{\alpha}(Y_{i_t}, X_{i_t}, \theta_{t})+\Omega(\theta_{t})\}, t=0,1,2,\cdots.
\end{equation*}
In both standard ridge and Cauchy log-truncated regressions, we also take five-fold CV to find the optimal parameters $\rho \in \R_+$ and $(\alpha,\rho) \in \R_+^2$.
\end{itemize}
\subsubsection{Numerical experiments}\label{subsub:numerical}
\begin{itemize}
\item[1.] \textbf{Simulation study}

For $\R^d$-valued covariates $\{X_i\}_{i=1}^n$, each $X_i$ can be written as
$$X_i=X_i'+\xi_i,$$
where $\{X_i'\}_{i=1}^n$ are i.i.d. $\R^d$-valued random vectors with normal distribution $N({\bf 0},\mathbf{Q}(\varsigma))$. Here the covariance matrix $\mathbf{Q}(\varsigma)$ is an identity matrix ($\varsigma=0$) or a Toeplitz matrix ($\varsigma=0.5$) which is formed
$$
\mathbf{Q}(\varsigma)=\left[\begin{array}{cccccc}
	1 & \varsigma & \varsigma^2 & \cdots & \cdots & \varsigma^{d-1} \\
	\varsigma & 1 & \varsigma & \ddots & & \vdots \\
	\varsigma^2 & \varsigma & \ddots & \ddots & \ddots & \vdots \\
	\vdots & \ddots & \ddots & \ddots & \varsigma & \varsigma^2 \\
	\vdots & & \ddots & \varsigma & 1 & \varsigma \\
	\varsigma^{d-1} & \cdots & \cdots & \varsigma^2 & \varsigma & 1
\end{array}\right].
$$
$\{\xi_i\}_{i=1}^n$ are i.i.d. $\R^d$-valued random noisy vectors, which satisfy one of the following conditions:
\begin{itemize}
	\item[(i)] \textbf{Pareto noise} (heavy tail): the noise {$\xi_i:=(\xi_{i1},...,\xi_{id})^\top$} whose entries $\{\xi_{ij}\}$ are independently drawn from Pareto distribution with scale parameter $1$ and shape parameter  $\tau \in \{1.6,1.8,2.01,4.01,6.01\}$.
	\item[(ii)]\textbf{Uniform noise} (outlier): the noise $\xi_i=Z\xi_i'$ where {$\xi_i':=(\xi_{i1}',...,\xi_{id}')^\top$} are independently drawn from uniform distribution $U(-2,2)$ in logistic regression and negative binomial regression, and $Z:={\rm{diag}}(\zeta_1,...,\zeta_d)$ where $\{\zeta_j\}_{j=1}^d$ are i.i.d. Bernoulli r.v.s with probability $\pi \in (0,1)$ taking 1. We will choose $ \pi \in \{0.3,0.5,0.8\}$.
\end{itemize}
{For each fixed pair $(d,n)$, the true value $\theta=(\theta_1,...,\theta_d)^\top \in \R^d$} is designed by drawing each $\theta_j$ from $U(0,1)$ independently. We compute the $\ell_2$-estimation error for each estimator $\hat{\theta}$, i.e., $\|\hat{\theta}-\theta\|_2$. We choose the high order log-truncated function as \eqref{eq:CANTON} with $\lambda(|x|)={|x|^{\beta}}/{\beta}$. When $\{\xi_i\}_{i=1}^n$ are Pareto noises, we choose
$\beta=1.5$ as $\tau \in \{1.6,1.8\}$ and $\beta=2.0$ as $\tau \in \{2.01,4.01,6.01\}$. When $\{\xi_i\}_{i=1}^n$ are uniform noises, we always choose $\beta=2.0$.

We will conduct experiments for the following three cases:
$$(d,n)=(100, 200), (200,500), (1000,1000).$$

Tables \ref{tab:LGpareto} and \ref{tab:LGuniform} present the comparison results of average $\ell_2$-estimation errors and standard errors (in bracket) for predicted logistic regression coefficients with 100 replications.  It is obvious that the log-truncated estimators perform much better than standard regression estimators under two noise settings. Our proposed log-truncated estimators based on high-order functions have smaller estimation errors than Cauchy log-truncated estimators. It reveals that the log-truncated regression with a high-order function is more flexible than the $1$-order log-truncation in coping with the contaminated or heavy-tailed data. We obtain similar results from simulations on NBR, which are displayed in Appendix \ref{se:appS}.

We also use the Pareto noisy model with $\tau=1.6$ to illustrate the rate $O\big((\frac{d\log(n)}{n})^{\frac{\beta-1}{\beta}}\big)$ of excess risk bound $R(\hat{\theta})-R(\theta^*)$ for different $\beta$ and sample size $n$ in Theorem \ref{thm1loss2}, if $r_n$ is constant. Figure \ref{fig:graphs} demonstrates that the numerical excess risk bound linearly decreases with the increase of $n$, and $\beta$ is closer to $1$, the excess risk bound is larger.
\begin{figure}
	\centering
	\includegraphics[width=0.6\textwidth]{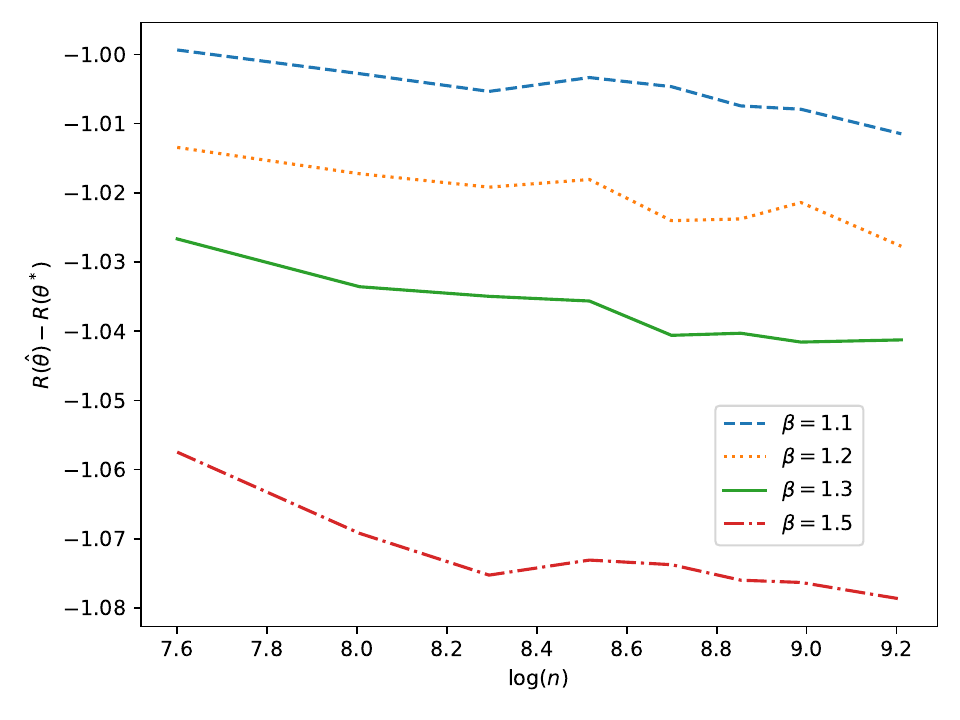}
	\caption{Plot of $R(\hat{\theta})-R(\theta^*)$ for different $\beta$ over various $n$ based on Pareto noise model with $\tau=1.6$, $d=100$ and $n$ in $[2000, 10000]$.}
	\label{fig:graphs}
\end{figure}

\item[2.] \textbf{Tuning Parameter selection}
\begin{figure}
	\centering
	\includegraphics[width=0.7\textwidth]{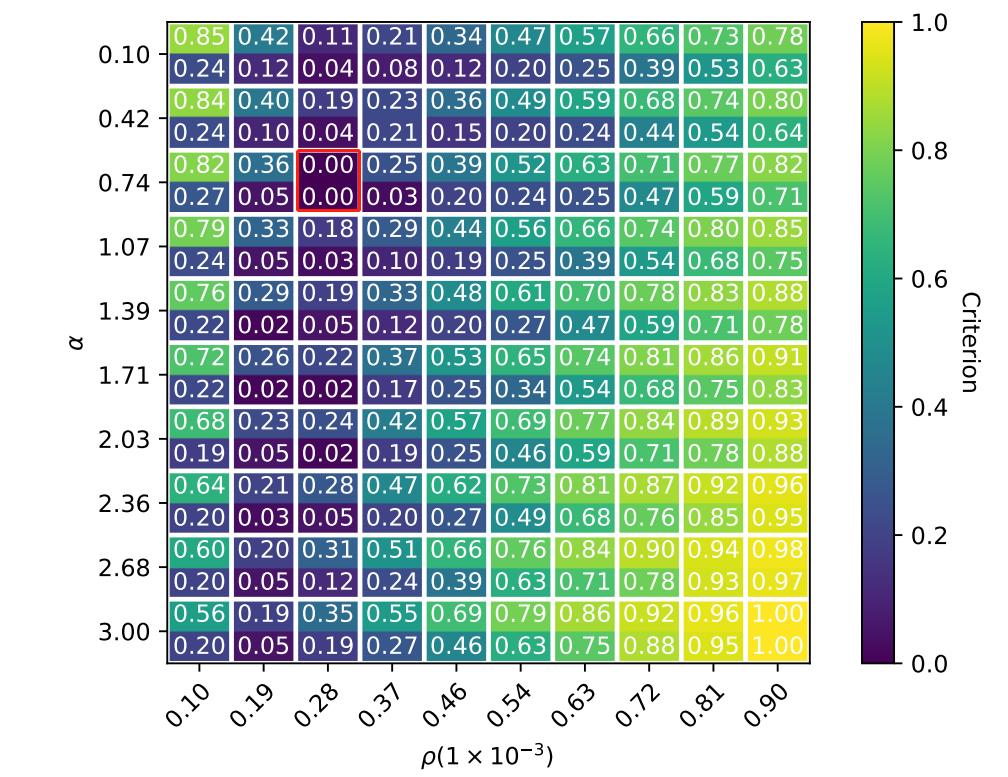}
	\caption{Comparisons of the tuning parameter selection by grid search and five-fold CV under the Pareto noise model as $d=200,n=500$. For each cell, the number in the upper half of the cell is the criterion of the grid search, i.e., $\|\hat{\theta}(\alpha,\rho)-\theta\|_2$. The number in the lower half of the cell is the criterion of five-fold CV, i.e., $\frac{1}{n}\sum_{j=1}^{5}\sum_{i\in K_j}|\hat{f}_{\alpha,\rho}^{-j}(X_i)-Y_i|$. All the values are normalized in $(0,1)$.}
	\label{fig:graph-para}
\end{figure}

Tuning parameter selection is a crucial step in the experiments, and correct parameters can enhance the prediction accuracy of a model. It is interesting to explore how to select the optimal tuning parameter in our proposed model. A simple and user-friendly tuning parameter selection method is the grid search. For example, for the regularized optimization problem (\ref{eq:ridge}), we can use the binary search to find an effective interval of the parameter pair $(\alpha, \rho)$, and use the grid search to select the optimal $(\hat{\alpha}, \hat{\rho})$ which minimizes the $l_2$ estimation error $\|\hat{\theta}(\alpha,\rho)-\theta\|_2$. However, a good fitting model on the training data can not say the model exactly works well. Thus, cross-validation is a popular technique for selecting the optimal tuning parameter. In our experiments, we use five-fold CV to select the optimal $(\hat{\alpha}, \hat{\rho})$ optimizing $\frac{1}{n}\sum_{j=1}^{5}\sum_{i\in K_j}|\hat{f}_{\alpha,\rho}^{-j}(X_i)-Y_i|$ in an effective subset of $\mathbb{R}^2_{+}$, where $K_j$ is the validation data set in five-fold CV. Figure \ref{fig:graph-para} plots the heat map of the criterion values of grid search and five-fold CV under the Pareto noise model as $d=200,n=500$. The number in the upper half of the cell is the criterion of the grid search, i.e., $\|\hat{\theta}(\alpha,\rho)-\theta\|_2$. The number in the lower half of the cell is the criterion of five-fold CV, i.e., $\frac{1}{n}\sum_{j=1}^{5}\sum_{i\in K_j}|\hat{f}_{\alpha,\rho}^{-j}(X_i)-Y_i|$. All the values are normalized in $(0,1)$. We find that the optimal parameter pairs $(\hat{\alpha}, \hat{\rho})$ selected by grid search and five-fold CV are consistent, and their corresponding criterion values are the smallest.
\end{itemize}
\subsection{Simulations on robust deep regression models}\label{se:Rdnn}

We also compare our proposed robust elastic net penalized DNN logistic regression and DNN NBR in Examples \ref{eg:DNNLAD2} and \ref{eg:DNNLAD3} of Section
$3$ with their non-truncated elastic net penalized DNN versions. The non-parametric function $f^*(x)$ in (\ref{eq:NNN}) satisfies
\begin{equation}\label{eq:f0}
	f^*(X_i)=f^*(X_i'+\xi_i), \ i=1,...,n,
\end{equation}
where $\{X_i'\}_{i=1}^n$ are i.i.d. with law $N({\bf 0},\mathbf{I}_{d\times d})$. Here
$\{\xi_i\}_{i=1}^n\in\mathbb{R}^d$ also follow Pareto or uniform distribution in the above noise setting. Differently, in the uniform noise setting, {$\xi_i':=(\xi_{i1}',...,\xi_{id}')^\top$} are independently drawn from uniform distribution $U(2,10)$. Two types of the real function $f^*$ are considered:
\begin{itemize}
	\item [(1)] \textbf{A complex function}: $f^*$ is a complex function (referred to \citep{ohn2022nonconvex}):
	$$f^*(x)=0.8\exp\big(0.03x_1+x_2^2-\sqrt{x_3+5}\big)-\cot\bigg(\frac{1}{0.01+|x_4-2x_5+x_6|}\bigg),$$
	where $d=6$. And we put $n=200$.
	
	\item [(2)] \textbf{DNN}: taking into account the higher dimensional input data, we generate a real two-layers deep neural network by Pytorch as the function $f^*$:

	
	\textit{The function $f^*(x)$: $\mathbb{R}^{d}\to\mathbb{R}$ is generated by a ReLu activated fully connected two-layers deep neural network with network width $(d, 0.6d, 0.4d, 1)$ by Pytorch. The real weights of the two-layers DNN are drawn from $U(-1,1)$ independently. We set $(d,n)=(100, 1000)$.}
\end{itemize}

The elastic net penalized robust DNN logistic regression is trained by the ERM problem (\ref{eq:NNN}) with the ReLU activated 2-layers DNN model and network width $(d, 0.6d, 0.4d, 1)$. We use the Adam optimization algorithm in PyTorch as implement with $n/4$ batch size in each case. The same network configurations and optimization algorithms are used to train the non-truncated elastic net penalized  DNN regressions.
For the elastic net based robust DNN logistic regression, we compute the accuracy of predictors $\{\hat{Y}_i\}_{i=1}^n$ for responses $\{Y_i\}_{i=1}^n\in\{0,1\}$, defined as
$$\text{Accuracy}:= \frac{1}{n}\sum_{i=1}^{n}\mathbf{1}(\hat{Y}_i=Y_i)\times 100\%.$$

Table \ref{tab:DNNLG} shows the average accuracy for robust DNN logistic regression under Pareto and uniform noise settings with 100 replications. Different values of $\beta$ are selected for different noises. The results reveal that the proposed elastic net penalized robust DNN logistic regression has higher flexibility and stronger robustness than the non-truncated elastic net-based DNN logistic regression for fitting the contaminated or heavy-tail data. The elastic net penalized robust DNN NBR results are demonstrated in Appendix \ref{se:appS}.

\begin{table}[htbp]
\caption{Comparison of average $\ell_2$-estimation error for logistic regression on Pareto noise model.}
\resizebox{1\textwidth}{!}{
	\begin{tabular}{cccccrlll}
		\toprule
		\multicolumn{9}{c}{$\ell2$-estimation error for logistic regression} \\
		\midrule
		&       & \multicolumn{3}{c}{$\varsigma=0$} &       & \multicolumn{3}{c}{$\varsigma=0.5$} \\
		\cmidrule{3-5}\cmidrule{7-9}    \multicolumn{9}{c}{$d = 100, n = 200$} \\
		\midrule
		Pareto & $\beta$   & High-order  & Cauchy & Non-truncation &       & \multicolumn{1}{c}{High-order } & \multicolumn{1}{c}{Cauchy} & \multicolumn{1}{c}{Non-truncation} \\
		\midrule
		1.60   & 1.5   & 2.9463(0.0110) & 3.1587(0.0167) & 3.5344(0.1532) &       & \multicolumn{1}{c}{2.9279(0.1401)} & \multicolumn{1}{c}{3.0992(0.1490)} & \multicolumn{1}{c}{3.5128(0.0306)} \\
		1.80   & 1.5   & 2.9922(0.0811) & 3.2056(0.0260) & 3.6353(0.0753) &       & \multicolumn{1}{c}{2.9241(0.0664)} & \multicolumn{1}{c}{3.0896(0.1063)} & \multicolumn{1}{c}{3.5103(0.0241) } \\
		2.01  & 2.0     & 2.9335(0.0217) & 3.2157(0.0416) & 3.6312(0.0715) &       & \multicolumn{1}{c}{2.7939(0.0906)} & \multicolumn{1}{c}{2.9930(0.1713)} & \multicolumn{1}{c}{3.5097(0.0425)} \\
		4.01  & 2.0     & 2.9210(0.0184) & 3.1100(0.0144) & 3.5047(0.0803) &       & \multicolumn{1}{c}{2.8205(0.0151)} & \multicolumn{1}{c}{2.9485(0.0275)} & \multicolumn{1}{c}{3.5084(0.0162)} \\
		6.01  & 2.0     & 2.8361(0.0153) &  3.0689(0.0458) & 3.4677(0.1294) &       & \multicolumn{1}{c}{2.8538(0.1114)} & \multicolumn{1}{c}{2.8994(0.1178)} & \multicolumn{1}{c}{3.5092(0.0410)} \\
		\midrule
		\multicolumn{9}{c}{$d = 200, n = 500$} \\
		\midrule
		Pareto & $\beta$   & High-order  & Cauchy & Non-truncation &       & \multicolumn{1}{c}{High-order } & \multicolumn{1}{c}{Cauchy} & \multicolumn{1}{c}{Non-truncation} \\
		\midrule
		1.60   & 1.5   & 3.8489(0.0378) & 4.0187(0.1169) & 4.4003 (0.0203) &       & \multicolumn{1}{c}{4.0401(0.1436)} & \multicolumn{1}{c}{4.0698(0.1761)} & \multicolumn{1}{c}{4.5583(0.1816)} \\
		1.80   & 1.5   & 3.8551(0.0225) & 4.0001(0.1370) & 4.4332(0.1774) &       & \multicolumn{1}{c}{3.9112(0.2145)} & \multicolumn{1}{c}{4.1223(0.1159)} & \multicolumn{1}{c}{4.5953(0.0784)} \\
		2.01  & 2.0     & 3.8271(0.0667) & 3.9885(0.0727) & 4.2360(0.0605) &       & \multicolumn{1}{c}{4.0839(0.1136)} & \multicolumn{1}{c}{4.1373(0.2145)} & \multicolumn{1}{c}{4.6304(0.1079)} \\
		4.01  & 2.0     & 3.9291(0.0194) & 4.0467(0.1243) & 4.4708(0.1353) &       & \multicolumn{1}{c}{3.9896(0.0853)} & \multicolumn{1}{c}{4.0554(0.1133)} & \multicolumn{1}{c}{4.5967(0.1821)} \\
		6.01  & 2.0     & 3.9502(0.0724) & 4.0155(0.1278) & 4.3148(0.1211) &       & \multicolumn{1}{c}{3.9835(0.0880)} & \multicolumn{1}{c}{4.0297(0.1287)} & \multicolumn{1}{c}{4.5948(0.1424)} \\
		\midrule
		\multicolumn{9}{c}{$d = 1000, n = 1000$} \\
		\midrule
		Pareto & $\beta$   & High-order  & Cauchy & Non-truncation &       & \multicolumn{1}{c}{High-order } & \multicolumn{1}{c}{Cauchy} & \multicolumn{1}{c}{Non-truncation} \\
		\midrule
		1.60   & 1.5   & 10.2445(0.0274) & 10.4860(0.0925) & 11.2177(0.0302) &       & \multicolumn{1}{c}{10.1364(0.0633)} & \multicolumn{1}{c}{10.4548(0.0382)} & \multicolumn{1}{c}{10.2014(0.1289)} \\
		1.80   & 1.5   & 10.2147(0.0372) & 10.4757(0.1345) & 11.1923(0.0221) &       & \multicolumn{1}{c}{10.0405(0.0857)} & \multicolumn{1}{c}{10.4431(0.0118)} & \multicolumn{1}{c}{4.5953(0.0784)} \\
		2.01  & 2.0     & 10.2983(0.0174) & 10.6056(0.0327) & 10.9366(0.0202) &       & \multicolumn{1}{c}{10.0366(0.0951)} & \multicolumn{1}{c}{10.4361(0.0318)} & \multicolumn{1}{c}{11.4579(0.0854)} \\
		4.01  & 2.0     & 10.2070(0.0206) & 10.4250(0.1389) & 11.7001(0.0307) &       & \multicolumn{1}{c}{9.9773(0.0584)} & \multicolumn{1}{c}{10.4361(0.0318)} & \multicolumn{1}{c}{11.5536(0.1378)} \\
		6.01  & 2.0     & 10.2290(0.0208) & 10.2782(0.0216) & 11.1155(0.0433) &       & 10.0486(0.0893) & 10.4462(0.0052) & 11.4853(0.0799) \\
		\bottomrule
	\end{tabular}}%
	\label{tab:LGpareto}%
\end{table}%

\begin{table}[htbp]
	\centering
		\caption{Comparison of average $\ell_2$-estimation error for logistic regression on Uniform noise model.}
	\resizebox{1\textwidth}{!}{
	\begin{tabular}{cccccrlcc}
		\toprule
		\multicolumn{9}{c}{$\ell_2$-estimation error for logistic regression} \\
		\midrule
		&       & \multicolumn{3}{c}{$\varsigma=0$} &       & \multicolumn{3}{c}{$\varsigma=0.5$} \\
		\cmidrule{3-5}\cmidrule{7-9}    \multicolumn{9}{c}{$d = 100, n = 200$} \\
		\midrule
		Uniform & $\beta$   & High-order  & Cauchy & Non-truncation &       & \multicolumn{1}{c}{High-order } & Cauchy & Non-truncation \\
		\midrule
		0.3   & 2.0     & 2.9615(0.0958) & 3.0826(0.1750) & 3.9509(0.0060) &       & \multicolumn{1}{c}{2.8887(0.0685)} & 2.9750(0.1000) & 3.3121(0.0938) \\
		0.5   & 2.0     & 2.9567(0.0141) & 3.0464(0.1687) & 3.9428(0.0024) &       & \multicolumn{1}{c}{2.9549(0.1425)} & 3.0426(0.1195) & 3.3372(0.1275) \\
		0.8   & 2.0     & 3.0271(0.1597) & 3.0896(0.1546) & 3.9447(0.0081) &       & \multicolumn{1}{c}{2.9613(0.0184)} & 3.0237(0.1622) & 3.3591(0.0957) \\
		\midrule
		\multicolumn{9}{c}{$d = 200, n = 500$} \\
		\midrule
		Uniform & $\beta$   & High-order  & Cauchy & Non-truncation &       & \multicolumn{1}{c}{High-order } & Cauchy & Non-truncation \\
		\midrule
		0.3   & 2.0     & 4.3418(0.0154) & 4.9339(0.0452) & 5.4107(0.0018) &       & \multicolumn{1}{c}{4.0917(0.1150)} & 4.2971(0.0507) & 4.8207(0.1894) \\
		0.5   & 2.0     & 4.4815(0.0745) & 5.0787(0.0687) & 5.4075(0.0149) &       & \multicolumn{1}{c}{4.1469(0.1168)} & 4.3671(0.1057) & 4.8559(0.0574) \\
		0.8   & 2.0     & 4.4230(0.0312) & 5.0416(0.1101) & 5.3869(0.0046) &       & \multicolumn{1}{c}{4.1265(0.0615)} & 4.4216(0.0935) & \multicolumn{1}{c}{4.8762(0.1569)} \\
		\midrule
		\multicolumn{9}{c}{$d = 1000, n = 1000$} \\
		\midrule
		Uniform & $\beta$   & High-order  & Cauchy & Non-truncation &       & \multicolumn{1}{c}{High-order } & Cauchy & Non-truncation \\
		\midrule
		0.3   & 2.0     & 10.6409(0.0149) & 10.6697(0.0314) & 10.7218(0.0176) &       & \multicolumn{1}{c}{10.4514(0.0753)} & 10.7567(0.1724) & 11.2571(0.0977) \\
		0.5   & 2.0     & 10.6744(0.1574) & 10.7760(0.1275) & 10.7279(0.0128) &       & \multicolumn{1}{c}{10.4686(0.0857)} & 10.6281(0.0694) & 11.2776(0.1124) \\
		0.8   & 2.0     & 10.7316(0.0715) & 10.8551(0.0957) & 10.7961(0.0196) &       & 10.4816(0.0965) & 10.7613(0.2020) & 11.4005(0.0542) \\
		\bottomrule
	\end{tabular}}%
	\label{tab:LGuniform}%
\end{table}%

\begin{table}[htbp]
	\centering
	\caption{Comparison of average accuracy for DNN logistic regression under two noise settings.}
	\begin{tabular}{|cc|cc|cc|}
		\hline
		\multicolumn{6}{|c|}{Accuracy (\%) for DNN logistic regression } \\
		\hline
		&       & \multicolumn{2}{c|}{$d = 6$, $n = 200$ (Complex function)} & \multicolumn{2}{c|}{$d = 100$, $n = 1000$ (DNN)} \\
		\hline
		\multicolumn{1}{|l}{$\beta$ } & \multicolumn{1}{l|}{Pareto ($\tau$) } & High-order  & Non-truncation  & High-order  & Non-truncation  \\
		\hline
		1.5   & 1.60   & 86.06(0.04) & 82.46(0.02) & 81.90(0.01) & 74.33(0.01)  \\
		1.5   & 1.80   & 85.60(0.02) & 82.25(0.02) & 80.04(0.01)  & 77.97(0.01)  \\
		2.0     & 2.01  & 86.57(0.04) & 81.07(0.02) & 83.49(0.01)  & 80.62(0.01)  \\
		2.0     & 4.01  & 85.93(0.04) & 83.98(0.02) & 93.36(0.02)  & 80.13(0.01)  \\
		2.0     & 6.01  & 87.90(0.01) & 85.06(0.05) & 95.75(0.02)  & 84.54(0.01)  \\
		\hline
		\multicolumn{1}{|l}{$\beta$ } & \multicolumn{1}{l|}{Uniform ($\pi$)} & High-order  & Non-truncation  & High-order  & Non-truncation  \\
		\hline
		2.0     & 0.3   & 84.00(0.02) & 81.33(0.02) & 89.05(0.01)  & 88.28(0.01)  \\
		2.0     & 0.5   & 82.64(0.02) & 80.26(0.03) & 88.28(0.01)  & 87.86(0.01)  \\
		2.0     & 0.8   & 82.04(0.01) & 80.47(0.02) & 88.61(0.01)  & 87.81(0.01)  \\
		\hline
	\end{tabular}%
	\label{tab:DNNLG}%
\end{table}%

\subsection{Real data analysis}
\subsubsection{Boston housing dataset}
We use the Boston housing dataset provided by the python library Scikit-Learn to learn the log-truncated standard and deep LAD models. Boston housing dataset contains $n=506$ cases, and each case includes 14 variables. We aim to predict Median Value (MEDV) of Owner-Occupied Housing Units as output, by the remaining 13 variables as input. To this end, we randomly split the dataset into two groups, one as the training set and the other as the testing set,  and train a $l_2$-regularized standard LAD regression in (\ref{eq:ridge}) and a 3-layers elastic net penalized DNN LAD regression model separately.

We denote the variable MEDV by $Y \in \R$ and the other 13 variables by $X \in \R^{13}$, so the Boston housing dataset can be represented as $\{(Y_i, X_i)\}_{i=1}^{506}$. In our experiment, $n_1=339$ samples are randomly selected for training and validation, and the remaining $n_2=167$ samples are the testing set, denoted by
$\{(Y_{tr,1}, X_{tr,1}),...,(Y_{tr,n_1},X_{tr,n_1})\}$ and $\{(Y_{te,1}, X_{te,1}),...,(Y_{te,n_2},X_{te,n_2})\}$ respectively. We use $4/5$ of the training samples to train the log-truncated standard LAD model (DNN model), then we select the optimal parameters on the remaining $1/5$ of the training set, and feed the testing set $\{X_{te,1},...,X_{te,n_2}\}$ into the trained model to get a prediction set $\{\hat Y_1,...,\hat Y_{n_2}\}$. To assess the obtained model,
we compute the absolute average errors (MAEs) of the prediction, which is defined as
\begin{equation*}
	\text{MAE}:=\frac{1}{n_{2}}\sum_{i=1}^{n_2}|\hat{Y}_i-Y_{te,i}|.
\end{equation*}

The function $\lambda$ in \eqref{eq:CANTON} is chosen as $\lambda(|x|)=|x|^{\beta}/{\beta}$. To select an appropriate $\beta$, we consider ten values of $\beta \in (1,2]$ with $\beta=1.1, 1.2,...,1$; see Table \ref{tab:Boston}. The corresponding normal LAD regressions without truncation are also considered, which are trained in the same way. Table \ref{tab:Boston} indicates that the log-truncated $l_2$-regularized standard LAD regression has smaller prediction errors than its non-truncated version. The truncated deep LAD  model outperforms the non-truncated deep LAD model in all settings. In particular, as $\beta=1.8$, the prediction errors are smallest for both truncated $l_2$-regularized standard LAD model and truncated deep LAD model. Thus, we choose $\beta=1.8$ for these two models.

\begin{table}[htbp]
	\centering
	\caption{Comparison of MAEs on Boston housing dataset.}
	\resizebox{0.8\textwidth}{!}{
		\begin{tabular}{cccccc}
			\toprule
			& \multicolumn{2}{c}{Standard LAD regression } &       & \multicolumn{2}{c}{Deep LAD regression} \\
			\cmidrule{1-3}\cmidrule{4-6}    	$\beta$      & Truncation & Non-truncation &       & Truncation & Non-truncation \\
			\cmidrule{1-3}\cmidrule{4-6}    1.1   & 6.2538 & 6.4843 &       & 6.2822 & 6.3186 \\
			1.2   & 6.2551 & 6.4843 &       & 6.2203 & 6.3186 \\
			1.3   & 6.2563 & 6.4843 &       & 6.2150 & 6.3186 \\
			1.4   & 6.2578 & 6.4843 &       & 6.1898 & 6.3186 \\
			1.5   & 6.2592 & 6.4843 &       & 6.1810 & 6.3186 \\
			1.6   & 6.2611 & 6.4843 &       & 6.2096 & 6.3186 \\
			1.7   & 6.2631 & 6.4843 &       & 6.1620 & 6.3186 \\
			1.8   & \textbf{6.2448} & 6.4843 &       & \textbf{6.0628} & 6.3186 \\
			1.9   & 6.2676 & 6.4843 &       & 6.1598 & 6.3186 \\
			2.0     & 6.2705 & 6.4843 &       & 6.1711 & 6.3186 \\
			\bottomrule
		\end{tabular}%
		\label{tab:Boston}}%
\end{table}%

The optimal $\beta$ in Table \ref{tab:Boston} can be roughly interpreted by Theorem \ref{thm1loss2}, from which we can see that there is a trade-off between the constant $\E H_{Y,X}^{\beta}$ and the excess risk convergence rate $\frac{1}{{{n^{(\beta-1)/\beta}}}}$ if $\beta$ varies from 1.1 to 2.0.

\subsubsection{MNIST database}
We use a handwritten digits database MNIST to learn a 6-layers elastic net penalized DNN LAD model, and compared it with the non-truncated model, which is learned in the same way. The activation function in \eqref{eq:NN} is ReLU. The two DNN models with elastic-net regularization are learned with Adam optimization
algorithm. MNIST database contains 70000 28$\times$28 grayscale images of the 10 digits. In experiments, we treat the digit images as the input variables $X$ and their corresponding labels as output $Y \in \{0,1,2,\cdots,9\}$. We randomly split the $70000$ images into three groups: the validation set ($10000$ images), denoted by $\{(X_{va,i}, Y_{va,i})\}\}_{i=1}^{10000}$; training set ($50000$ images), denoted by $\{(X_{tr,i}, Y_{tr,i})\}\}_{i=1}^{50000}$ and testing set ($10000$ images), denoted by  and $\{(X_{te,i}, Y_{te,i})\}_{i=1}^{10000}$ respectively. Firstly, we normalize the 28$\times$28 pixels of each image in range $(-1,1)$ and train several candidate DNN models using the parameters $(\alpha,\rho,\gamma)$ in a certain subset of $\mathbb{R}_{+}^3$.  Then, we select the optimal parameters by computing the classification accuracy of their corresponding trained DNN models on the validation set, that is:

\begin{equation}
	\text{Accuracy}(\alpha,\rho,\gamma):= \frac{1}{10000} \sum_{i=1}^{10000}\mathbf{1} (\hat{Y}_{va,i}(\alpha,\rho,\gamma)=Y_{va,i})\times100\%,
\end{equation}
where $\mathbf{1}(\cdot)$ denotes the indicator function, $\{\hat{Y}_{va,i}\}_{i=1}^{10000}$ are the predictors from 10000 images in validation set. The batch size is $64$. Next, we feed the 10000 testing images into the trained DNN model which is corresponded to the optimal parameters and compute their classification accuracy.

Three types of noises are added on the 28$\times$28 pixels of the original data: Gaussian noise $N(0.5,2)$, uniform noise $U(2,5)$ and Pareto noise with shape parameter $\beta=2.01$. For each batch, we randomly generate $20\%, 50\%$ and $80\%$ noises to contaminate the data. The truncation function is also chosen as \eqref{eq:CANTON} with $\lambda(|x|)=|x|^{\beta}/{\beta}$. We select the optimal values of $\beta$ on $(1,2]$ with step 0.1 for the three classes of noisy settings. We repeat the experiment 100 times and record the average accuracy, see Table \ref{tab:MNIST} (the values in the bracket are standard errors of accuracy). The results in Table \ref{tab:MNIST} show that the classification results of the truncated DNN LAD model are better than the standard DNN LAD model. Especially under strong noise disturbance, the truncated DNN LAD model is more robust than the standard DNN LAD model.

We can see from the tables that the optimal index $\beta$ does not change according to the proportions of noises, and this is a significant advantage of our truncation regression models.

\begin{table}[htbp]
	\centering
	\caption{Comparison of classification accuracy on MNIST dataset}
	\begin{tabular}{ccccc}
		\toprule
		\multicolumn{5}{c}{Accuracy (\%)} \\
		\midrule
		\multicolumn{5}{c}{Gaussian Noise $N(0.5,2)$} \\
		\midrule
		$\beta$ & Noise proportion & Truncation &       & Non-truncation \\
		2.0     & 20\%   & 86.85 (0.01) &       & 75.71 (0.01) \\
		2.0     & 50\%   & 84.39 (0.02) &       & 74.27 (0.03) \\
		2.0     & 80\%   & 76.06 (0.03) &       & 72.96 (0.03) \\
		\midrule
		\multicolumn{5}{c}{Uniform Noise $U(2,5)$} \\
		\midrule
		$\beta$ & Noise proportion & Truncation &       & Non-truncation \\
		1.5   & 20\%   & 96.56 (0.01) &       & 95.91 (0.09) \\
		1.5   & 50\%   & 93.87 (0.01) &       & 92.46 (0.01) \\
		1.5   & 80\%   & 93.45 (0.03) &       & 91.65 (0.01) \\
		\midrule
		\multicolumn{5}{c}{Pareto Noise $\beta = 2.01$} \\
		\midrule
		$\beta$ & Noise proportion & Truncation &       & Non-truncation \\
		2.0   & 20\%   & 88.75 (0.02) &       & 87.85 (0.01) \\
		2.0   & 50\%   & 76.03 (0.03) &       & 74.74 (0.01) \\
		2.0   & 80\%   & 65.12 (0.02) &       & 60.96 (0.04) \\
		\bottomrule
	\end{tabular}%
	\label{tab:MNIST}%
\end{table}%

\section{Discussion and Future Study}\label{conclusion}
{Our proposed robust elastic net estimators, in practice, need to be obtained by SGD-based algorithms, which are deserved to be studied in future research.}

From the simulations and real data analysis above, we can see that selecting the $\beta$ is a crucial issue for prediction. The tail index estimation has been intensively studied in extreme-value statistics; see \cite{fedotenkov2020review} for a review. We leave the research of estimating the index $\beta$ to future study.

As mentioned above, Theorem \ref{thm1loss2} or Theorem \ref{thm:Elastic} can be applied to study other robust statistical learning models, in which one may have to design specific algorithms
according to the concrete problems at hand. We conclude this paper with the following two examples,
robust two-component mixed linear regression and robust non-negative matrix factorization, whose algorithms differ from SGD. Here we only roughly address their theoretical results and leave detailed study to future work. \\

\textbf{Robust two-component mixed linear regression}. One challenging non-convex problem is the mixture of two linear
regressions. Suppose $\{(Y_{i}, X_{i})\sim (Y, X)\}_{i=1}^n$ are $\mathbb{R} \times \mathbb{R}^{d}$-valued i.i.d. random variables. With probability $\pi$, $(X, Y)$ has conditional density function $p(y,{x^ \top }{\eta _0}),~{\eta _0}\in {\mathbb{R}^d}$ for $Y=y|X=x$, and with probability $1-\pi$, $(X, Y)$ has conditional density function $p(y,{x^ \top }{\eta _1}),~{\eta _1}\in {\mathbb{R}^d}$ for $Y=y|X=x$, {where ${\eta _0}$ and ${\eta _1}$ are unknown coefficients.} Given the input $x$, the negative log-likelihood function of the output $y$ is
\begin{align}\label{eq:mix}
l(y,x;{\pi, \eta_{0}, \eta_{1}})=-\log[\pi p(y,{x^ \top }{\eta _0})+(1-\pi) p(y,{x^ \top }{\eta _1})],
\end{align}
where $\pi \in(0,1)$ is an unknown mixing probability. \cite{mei2018landscape} studied Gaussian mixture models, while \cite{khamaru2019convergence} considered mixture density estimation under sub-exponential condition.

In order to fit this example to our theory, we write $\theta : = \left( {\pi ,{\eta _0},{\eta _1}} \right)$. Under the following moment conditions:
${\E}({\left\| X \right\|_2}|Y|)^\beta< \infty$, ${\E}\left\| X \right\|_2^{2\beta}< \infty$, by Theorem \ref{thm1loss2} with $p=2d+1$, we can obtain an excess risk in the order of $O( {{{( {{{(2d+1)\log (nr_n)}}/{n}} )}^{(\beta-1) /\beta}}})$ if the regularization error $\rho\|\theta^*\|_2^2\lesssim {{( {\frac{{(2d+1)\log (nr_n)}}{n}})}^{ (\beta-1)/{{\beta}}}}$, see Appendix \ref{mixed} for details. \\




\textbf{Robust non-negative matrix factorization}.
Given $n$ vector samples $\{X_{i}\}_{i=1}^n\in \mathbb{R}^p$ arranged in a nonnegative matrix $\mathrm{S}:=\left[X_{1}, \ldots, X_{n}\right] \in \mathbb{R}_{+}^{d \times n}$ and positive integer $z \ge 1$, the log-truncated non-negative matrix factorization (NMF) considers the decomposition for $\mathrm{S}$ into a product of two non-negative matrices:
$$
\mathrm{S}=\mathrm{B}\mathrm{C} +\mathrm{R},
$$
where $\mathrm{B} \in \mathbb{R}_{+}^{d \times z}$ is the given basis, $\mathrm{C}=\left[h_{1}, \ldots, h_{n}\right] \in \mathbb{R}_{+}^{z \times n}$ is the coefficients, and $\mathrm{R} \in \mathbb{R}^{d \times n}$ is the random error matrix.

By minimizing the $\ell_2$-distance between
their product and the original data matrix, the ordinary NMF decomposes a data matrix into the
product of two lower dimensional non-negative factor matrices $\mathrm{B}$ and $\mathrm{C}$. When the original data matrix is corrupted by heavy-tailed outliers that seriously violate the second-moment assumption \citep{guan2017truncated}. We assume the element of $R_{ij}$ in $\mathrm{R}$ has only $\beta$-th moment, i.e. $\sup_{i\in [n],j\in [d]}{\E}|R_{ij}|^\beta< \infty$ with $\beta \in (1,2)$, it is of interest to study the log-truncated NMF and the robust algorithm by using the following non-convex optimization problem:
$$
\hat{\mathrm{C}}=\mathop {\arg\min }\limits_{\mathrm{C} \in \mathbb{R}_{+}^{z \times n}}\frac{1}{{np\alpha }}\sum\limits_{i\in [n],j\in [d]} {\psi_\lambda}\left[\alpha \left({\mathrm{S}-\mathrm{B}\mathrm{C}}\right)_{i j}^{2}\right]
$$
where $\lambda(x)={\beta}^{-1}|x|^{\beta},$ and the $\alpha$ is tuning parameter implicitly determined by the random error matrix. We expect that $\hat{\mathrm{C}}$ is able to learn a subspace on a dataset through the original data matrix that is contaminated by a heavy-tailed noise matrix.\


\acks{L. Xu is supported in part by NSFC Grant No.12071499, Macao S.A.R grant FDCT  0090/2019/A2 and University of Macau grant  MYRG2018-00133-FST. F. Yao is supported by NSFC Grants No.11931001 and 11871080, the LMAM, and the Key Laboratory of Mathematical Economics and Quantitative Finance (Peking University), Ministry of Education, National Key R\&D Program of China Grant (No. 2020YFE0204200). H. Zhang is supported in part by NSFC Grant No.12101630 and the University of Macau under UM Macao Talent Programme (UMMTP-2020-01).}

\vskip 0.2in
{\footnotesize{
\bibliographystyle{apalike}
\bibliography{ref}
}}
\appendix

\newpage
\section*{Appendix}
The appendix includes the proofs of the lemmas, corollaries and theorems in the main body.

\section{Proofs}

\subsection{Useful lemmas of covering number bounds}

The following covering number bound of $\ell_2$-ball is in Corollary 4.2.13 in \cite{vershynin2018high}.
\begin{lemma}[Covering number bound of the $\ell_2$-ball]\label{net cardinality}
For any $\kappa>0$, the covering number of the $p$-dimensional unit $\ell_2$-ball $B_{2}^{p}(1)$ satisfy
\begin{center}
$\left(\frac{1}{\kappa}\right)^{p} \leq {N}\left(B_{2}^{p}(1),\kappa\right) \leq\left(\frac{2}{\kappa}+1\right)^{p}.$
\end{center}
\end{lemma}

The next lemma is modified from \cite{vershynin2009role}. Here, we provide a sharper covering number bound, while (2) in \cite{vershynin2009role} contains a unknown universal constant.
\begin{lemma}[Covering number bound of the $s$-sparse $\ell_2$-ball]\label{net cardinality1}
For any $\kappa>0$, the covering number of the $p$-dimensional unit $s$-sparse $\ell_2$-ball satisfy
\begin{center}
${N}\left(B_{2}^{p}(1)\cap B_{0}^{p}(s),\kappa\right) < \frac{1}{ \sqrt{2 e s}}\left(\frac{(\kappa+2)ep}{\kappa s}\right)^s.$
\end{center}
\end{lemma}
\begin{proof}
To get the result, we consider a union bound over $k$-dimensional subspaces of $B_{2}^{p}(1)$ by using the upper bound in Lemma \ref{net cardinality} to bound the covering number of $s$-sparse $\ell_2$-ball,
\begin{center}
${N}\left(B_{2}^{p}\cap B_{0}^{p}(s), \kappa\right) \leq \left(\begin{array}{c}
p \\
s
\end{array}\right){N}\left(B_{2}^{s},\kappa\right) < \frac{(e p / s)^{s}}{ \sqrt{2 e s}}\left(\frac{2}{\kappa}+1\right)^{s}=\frac{1}{ \sqrt{2 e s}} \left(\frac{(\kappa+2)ep}{\kappa s}\right)^s,$
\end{center}
where the last inequality is by Stirling's approximation $\left(\begin{array}{l}p \\ s\end{array}\right)<\frac{(e p / s)^{s}}{ \sqrt{2 e s}}$ for $1 \leq s \leq p$.

\end{proof}

\subsection{The proof of Theorem \ref{thm1loss}}

\begin{proof}
By the definition of $\hat \theta_n$, we have for all $\theta^*\in \Theta^*$
\begin{align}\label{eq:defGLM0}
{{\hat R}_{\psi_\lambda,l,\alpha}}(\hat\theta_n )+\rho\|\hat\theta_n\|_2^2&=\frac{1}{{n\alpha }}\sum\limits_{i = 1}^n \psi_\lambda  [\alpha l({Y_i}, {{X}}_i,\hat \theta_n )]+\rho\|\hat\theta_n\|_2^2 \nonumber\\
&\le \frac{1}{{n\alpha }}\sum\limits_{i = 1}^n \psi_\lambda  [\alpha l({Y_i} , {{X}}_i, \theta^*  )]+\rho\|\theta^*\|_2^2=:{{\hat R}_{\psi_\lambda,l,\alpha}}({\theta^*})+\rho\|\theta^*\|_2^2
\end{align}
which yields ${{\hat R}_{\psi_\lambda,l,\alpha}}(\hat \theta_n ) - {{\hat R}_{\psi_\lambda,l,\alpha}}({\theta^*})\le \rho(\|\theta^*\|_2^2-\|\hat\theta_n\|_2^2)\le \rho\|\theta^*\|_2^2$ and thus
\begin{align}\label{eq:decompGLM0}
{R_l}(\hat \theta_n ) - R_l({\theta^*})&= [ {{R_l}(\hat \theta_n ) - {{\hat R}_{\psi_\lambda,l,\alpha}}(\hat \theta_n )}] +[{{\hat R}_{\psi_\lambda,l,\alpha}}(\hat \theta_n ) - {{\hat R}_{\psi_\lambda,l,\alpha}}({\theta^*})] +  [ {{{\hat R}_{\psi_\lambda,l,\alpha}}({\theta^*}) - {R_l}({\theta^*})} ]\nonumber\\
 &\le  [{{{\hat R}_{\psi_\lambda,l,\alpha}}({\theta^*}) - {R_l}({\theta^*})} ]+ [ {{R_l}(\hat \theta_n ) - {{\hat R}_{\psi_\lambda,l,\alpha}}(\hat \theta_n )}]+ \rho\|\theta^*\|_2^2.
\end{align}

Under the two high provability events $\mathcal{E}_1(l,\lambda,{\theta^*})$ in Lemma \ref{eq:LOSSshp1} and $\mathcal{E}_2(l,\lambda,\hat \theta_n)$ in Lemma \ref{eq:LOSSwhp2} below, we have by inequality \eqref{eq:decompGLM0}
\begin{align}\label{eq:lem1lem2}
&~~~~{R_l}(\hat \theta_n ) - {R_l}({\theta^*}) \le \left(\frac{f(\alpha)}{\alpha}{R_{\lambda \circ l}}({\theta^*}) + \frac{1}{{n\alpha }}\log \frac{1}{\delta }\right)\nonumber\\
 &+2\kappa {{\E}}H_{Y,X} + \frac{{{c_2}f(\alpha )}}{\alpha }\mathop {\sup }\limits_{\theta  \in {\Theta}}R_{\lambda \circ l}( \theta )+ \frac{{{c_2}f(\alpha \kappa )}}{\alpha }{{\E}}[\lambda(H_{Y,X})]    + \frac{1}{{n\alpha }}\log \frac{N({\Theta}, \kappa)}{\delta }+\rho\|\theta^*\|_2^2\nonumber\\
& =2\kappa {{\E}}H_{Y,X} + \frac{{{c_2}f(\alpha \kappa )}}{\alpha }{{\E}}[\lambda(H_{Y,X})]  +\frac{{(c_2+1)f(\alpha )}}{\alpha }R_{\lambda \circ l}(\Theta)+ \frac{1}{{n\alpha }}\log \frac{N({\Theta}, \kappa)}{\delta^2}+\rho\|\theta^*\|_2^2,
\end{align}
with probability at least $1 - 2\delta$.

For the last two terms in \eqref{eq:lem1lem2}, we put $\frac{{(c_2+1)f(\alpha )}}{\alpha }R_{\lambda \circ l}(\Theta) = \frac{1}{{n\alpha }}\log \frac{{N({\Theta},\kappa )}}{{{\delta ^2}}}$, i.e. the variance term equals to the bias term. Then we get
$\alpha = {f^{ - 1}}\left( {\frac{1}{{n({c_2} + 1)}}R_{\lambda \circ l}^{ - 1}(\Theta)\log \frac{{N(\Theta ,\kappa )}}{{{\delta ^2}}}} \right)$. So \eqref{eq:lem1lem2} implies
\begin{align*}
{R_l}(\hat \theta_n ) - {R_l}({\theta^*})\le 2\kappa {{\E}}H_{Y,X} + \frac{{{c_2}f({\alpha }\kappa )}}{\alpha }{{\E}}[\lambda(H_{Y,X})]  + \frac{2}{{n\alpha }}\log \frac{N(\Theta ,\kappa )}{\delta^2}+\rho\|\theta^*\|_2^2.
\end{align*}
Let $\kappa  = \frac{1}{n}$ and take infimum over for each $\theta^*\in \Theta^*$, we obtain
$${R_l}(\hat \theta_{n} ) -\mathop {\inf}\limits_{\theta  \in {\Theta}} {R_l}({\theta})  \le \frac{{2{{\E}}H_{Y,X}}}{{n}} + \frac{{{c_2}{{\E}}[\lambda(H_{Y,X})] }}{\alpha}f(\frac{\alpha}{n}) +\frac{2}{{n\alpha }}\log \frac{N(\Theta ,n^{-1} )}{\delta^2}+\rho\|\Theta^*\|_2^2,$$
with probability at least $1 - 2\delta$.

\end{proof}

\begin{lemma}[Concentration error bound]\label{eq:LOSSshp1}
For general loss function $l(\cdot, \cdot, \cdot)$, under (C.5), we have for all $\theta^*\in \Theta^*$
\[\mathbb{P}\{\mathcal{E}_1(l,\lambda,{\theta^*})\}\ge 1 - \delta,\]
where $\mathcal{E}_1(l,\lambda,{\theta^*}):=\{ {{{\hat R}_{\psi_\lambda,l,\alpha}}({\theta^*}) - {{ R}_l}({\theta^*})\le \frac{f(\alpha)}{\alpha}{R_{\lambda \circ l}}({\theta^*})+ \frac{1}{{n\alpha }}\log \frac{1}{\delta }}\}$.
\end{lemma}
\begin{proof}
The proof is based on bounding the exponential moment of $n\alpha {{\hat R}_{\psi_\lambda,l,\alpha}}({\theta^*})$ from Markov's inequality. Applying the upper truncated function in \eqref{eq:CANTON}, $1+x \le e^{x}$, one has
\begin{align*}
{\E}{e^{n\alpha {{\hat R}_{\psi_\alpha \circ l}}({\theta^*})}} &= {\E}{e^{ \sum\limits_{i = 1}^n \psi  [\alpha l({Y_i} , X_i, {\theta^*})]}} \le {\E}\{ \prod\limits_{i = 1}^n {[1 + \alpha l({Y_i} , X_i, {\theta^*}) +\lambda[\alpha l({Y_i},X_i, \theta^*)]]} \} \\
[\text{By independence}]~& = \prod\limits_{i = 1}^n{\{ {\E}[1 + \alpha l({Y}_i , X_i, {\theta^*}) +\lambda[\alpha l({Y_i},X_i, \theta^*)]]\} } \\
&\le e^{\alpha \sum_{i = 1}^n {\E} l(Y_i, X_i, {\theta^*}) +\sum_{i = 1}^n {\E}\{ \lambda[\alpha l({Y_i},X_i, \theta^*)]\}}\le {e^{n[\alpha {R_l}({\theta^*}) +f(\alpha){R_{\lambda \circ l}}({\theta^*})]}}.
\end{align*}
Via Markov's inequality with the exponential transform, it gives
\begin{align*}
&~~~~\mathbb{P}\{\mathcal{E}_1^c(l,\lambda,{\theta^*})\}=\mathbb{P}\{ {{\hat R}_{\psi_\lambda,l,\alpha}}({\theta^*}) > [ {R_l}({\theta^*}) +\frac{f(\alpha)}{\alpha}{R_{\lambda \circ l}}({\theta^*})] + \frac{{\log (1/\delta )}}{{n\alpha }}\} \\
& = \mathbb{P}\{ {e^{n\alpha {{\hat R}_{\psi_\lambda,l,\alpha}}({\theta^*})}} > {e^{n\alpha[ {R_l}({\theta^*}) +\frac{f(\alpha)}{\alpha}{R_{\lambda \circ l}}({\theta^*})] + \log (1/\delta )}}\}  \le \frac{{{\E}{e^{n\alpha {{\hat R}_{\psi_\lambda,l,\alpha}}({\theta^*})}}}}{{{e^{n\alpha[ {R_l}({\theta^*}) +\frac{f(\alpha)}{\alpha}{R_{\lambda \circ l}}({\theta^*})] + \log (1/\delta )}}}} \le \delta.
\end{align*}
\end{proof}
\begin{lemma}[Generalization error bound]\label{eq:LOSSwhp2}
For any $\kappa>0$, under (C.1)-(C.5), one has
\[\mathbb{P}\{\mathcal{E}_2(l,\lambda,\hat \theta_n)\} \ge 1 - \delta,\]
where $\mathcal{E}_2(l,\lambda,\hat \theta_n):=\{ {{ R}_{l}}(\hat \theta_n )-{{\hat R}_{\psi_\lambda,l,\alpha}}(\hat \theta_n ) \le 2\kappa {{\E}}H_{Y,X} + \frac{{{c_2}f(\alpha )}}{\alpha }R_{\lambda \circ l}( \Theta)+ \frac{{{c_2}f(\alpha \kappa )}}{\alpha }{{\E}}[\lambda(H_{Y,X})] + \frac{1}{{n\alpha }}\log \frac{N({\Theta}, \kappa)}{\delta } \}$.
\end{lemma}
\begin{proof}
Let $\mathcal{N}({\Theta}, \kappa)$ be an $\kappa$-net of ${\Theta}$ and denote its covering number as ${N}({\Theta}, \varepsilon)$. For each $\hat \theta_n \in {\Theta}$, the definition of $\kappa$-net implies that there exists a $\tilde{\theta} \in \mathcal{N}({\Theta}, \kappa)$ satisfying $\|\hat{\theta}_n-\tilde{\theta}\|_{{2}} \leq \kappa$, by Lipschitz condition (C.3), we obtain
\begin{align*}
l({Y_i} , {{X}}_i, \hat \theta_n) \ge l({Y_i}, {{X}}_i, \tilde{\theta})-\kappa H_{{Y_i}, {{X}}_i},~i=1,2,\cdots,n.
\end{align*}
Since $\psi_\lambda(\cdot)$ is non-decreasing and the last inequality gives
\begin{align}\label{eq:Wlower1}
{{\hat R}_{\psi_\lambda,l,\alpha}}(\hat \theta_n) &= \frac{1}{{n\alpha }}\sum\limits_{i = 1}^n \psi_\lambda[\alpha l({Y_i},X_i, \hat \theta_n )] \ge \frac{1}{{n\alpha }}\sum\limits_{i = 1}^n \psi_\lambda[\alpha l({Y_i},X_i, \tilde \theta ) - \kappa \alpha H_{{Y_i}, {{X}}_i}].
\end{align}

In below, we continue to derive the lower bound  in \eqref{eq:Wlower1} by applying covering number techniques. From the lower bound in \eqref{eq:CANTON}, we could estimate the exponential moment bound:
\begin{align*}
&~~~~ {\E}{e^{ - \sum_{i = 1}^n \psi_\lambda[\alpha l({Y_i} , X_i,\tilde \theta ) - \alpha \kappa H_{{Y_i}, {{X}}_i}]}}\\
& \le {\E} \prod\limits_{i = 1}^n \{1 - \alpha {\E}[l({Y_i},X_i, \tilde \theta )] + \alpha \kappa H_{{Y_i}, {{X}}_i}+\lambda[\alpha( l({Y_i},X_i, \tilde \theta ) - \kappa H_{{Y_i}, {{X}}_i})]\}\\
& \le \prod\limits_{i = 1}^n\left\{1 - \alpha {\E} [l({Y_i},X_i, \tilde \theta )] + \alpha \kappa {\E}H_{{Y_i}, {{X}}_i}+{\E}\{ \lambda[\alpha( l({Y_i},X_i, \tilde \theta ) - \kappa H_{{Y_i}, {{X}}_i})]\} \right\}\\
[\text{(C.1.2)}]& \le \prod\limits_{i = 1}^n\left\{1 - \alpha {\E} [l({Y_i},X_i, \tilde \theta )] + \alpha \kappa  {\E}H_{{Y_i}, {{X}}_i}+{c_2}{\E}\{ \lambda[\alpha (l({Y_i},X_i, \tilde \theta) )] + \lambda[\alpha \kappa H_{{Y_i}, {{X}}_i}]\}  \right\} \\
[\text{(C.1.1)}]&   \le {e^{n\alpha \{ { - R_l(\tilde \theta ) + \kappa {{\E}}H_{Y,X} + \frac{{{c_2}f(\alpha )}}{\alpha }R_{\lambda \circ l}(\tilde \theta ) + \frac{{{c_2}f(\alpha \kappa )}}{\alpha }{{\E}}[\lambda(H_{Y,X})] } \}}},
\end{align*}
where the last inequality stems from $1+x \leq e^{x}$. By the last exponential moment bound, Markov's inequality shows that, for a fixed $\tilde \theta \in \mathcal{N}({\Theta}, \kappa)$,
\begin{align}\label{eq:Whigpro}
 \mathbb{P}&\left\{ \frac{{ - 1}}{{n\alpha }}\sum\limits_{i = 1}^n \psi_\lambda(\alpha l({Y_i},X_i, \tilde \theta ) - \kappa \alpha H_{{Y_i}, {{X}}_i}) > - {R_l}(\tilde \theta )   + \frac{{{c_2}f(\alpha )}}{\alpha }R_{\lambda \circ l}(\tilde \theta ) \right. \nonumber\\
&\left. + \kappa {{\E}}H_{Y,X} +  \frac{{{c_2}f(\alpha \kappa )}}{\alpha }{{\E}}[\lambda(H_{Y,X})] + {\frac{{\log (1/s)}}{{n\alpha }}} \right\} \nonumber\\
& \le \frac{{e^{n\alpha \{ { - R_l(\tilde \theta ) + \kappa {{\E}}H_{Y,X} + \frac{{{c_2}f(\alpha )}}{\alpha }R_{\lambda \circ l}(\tilde \theta ) + \frac{{{c_2}f(\alpha \kappa )}}{\alpha }{{\E}}[\lambda(H_{Y,X})] } \}}}}{{{e^{n\alpha \{ - R_l(\tilde \theta ) + \kappa {{\E}}H_{Y,X} + \frac{{{c_2}f(\alpha )}}{\alpha }R_{\lambda \circ l}(\tilde \theta ) + \frac{{{c_2}f(\alpha \kappa )}}{\alpha }{{\E}}[\lambda(H_{Y,X})]\}  + \log (1/s)}}}}= s \in (0,1).
\end{align}

The set $\mathcal{N}({\Theta}, \kappa)$ has $N({\Theta}, \kappa)$ elements. From the single bound \eqref{eq:Whigpro} and putting $s=\delta/N({\Theta}, \kappa)$, we have for all $\tilde \theta \in \mathcal{N}({\Theta}, \kappa)$
\begin{align}\label{eq:higprounion0}
\mathbb{P}&\left(  \bigcup\limits_{\tilde \theta  \in \mathcal{N}({\Theta}, \kappa)}\left\{ \frac{{ - 1}}{{n\alpha }}\sum\limits_{i = 1}^n \psi_\lambda(\alpha l({Y_i},X_i, \tilde \theta ) - \kappa \alpha H_{{Y}, {{X}}})> - {R_l}(\tilde \theta )   + \frac{{{c_2}f(\alpha )}}{\alpha }R_{\lambda \circ l}(\tilde \theta ) \right. \right.\nonumber\\
&\left.\left. + \kappa {{\E}}H_{Y,X} +  \frac{{{c_2}f(\alpha \kappa )}}{\alpha }{{\E}}[\lambda(H_{Y,X})] + {\frac{{\log (1/s)}}{{n\alpha }}} \right\}\right) \nonumber\\
& \le N({\Theta}, \kappa) \cdot \mathbb{P}\left\{ \frac{{ - 1}}{{n\alpha }}\sum\limits_{i = 1}^n \psi_\lambda (\alpha l({Y_i},X_i,\tilde \theta ) - \kappa \alpha H_{{Y}, {{X}}}) \ge - {R_l}(\tilde \theta )   + \frac{{{c_2}f(\alpha )}}{\alpha }R_{\lambda \circ l}(\tilde \theta ) \right. \nonumber\\
&\left. + \kappa {{\E}}H_{Y,X} +  \frac{{{c_2}f(\alpha \kappa )}}{\alpha }{{\E}}[\lambda(H_{Y,X})] + {\frac{{\log (1/s)}}{{n\alpha }}} \right\}\le N({\Theta}, \kappa)s = :\delta.
\end{align}

Then the complementary set in \eqref{eq:higprounion0} hold with probability at least $1-\delta$. Inequalities \eqref{eq:Wlower1} and \eqref{eq:higprounion0} give the following lower bound for all $\tilde \theta \in \mathcal{N}({\Theta}, \kappa)$
\begin{align}\label{eq:cmgfupper1}
 & ~~~~{{\hat R}_{\psi_\lambda,l,\alpha}}(\hat \theta)\ge \frac{1}{{n\alpha }}\sum\limits_{i = 1}^n \psi_\lambda  [\alpha l({Y_i},X_i, \tilde \theta ) - \kappa \alpha H_{{Y}, {{X}}}]\nonumber\\
&\ge {{ R}_l}(\tilde \theta ) - \left\{\kappa {{\E}}H_{Y,X} + \frac{{{c_2}f(\alpha )}}{\alpha }R_{\lambda \circ l}(\tilde \theta )+ \frac{{{c_2}f(\alpha \kappa )}}{\alpha }{{\E}}[\lambda(H_{Y,X})] + \frac{1}{{n\alpha }}\log \frac{N({\Theta}, \kappa)}{\delta }\right\} \nonumber\\
&\ge {{ R}_l}(\tilde \theta ) - \left\{\kappa {{\E}}H_{Y,X} + \frac{{{c_2}f(\alpha )}}{\alpha }R_{\lambda \circ l}(\Theta)+ \frac{{{c_2}f(\alpha \kappa )}}{\alpha }{{\E}}[\lambda(H_{Y,X})]   + \frac{1}{{n\alpha }}\log \frac{N({\Theta}, \kappa)}{\delta }\right\}
\end{align}
 with probability at least $1-\delta$.

It remains to find the lower bound for ${{ R}_l}(\tilde \theta )$  in \eqref{eq:cmgfupper1} by the error bound of $|{R_l}(\hat \theta_n ) -{R_l}(\tilde \theta )|$. To this end, the (C.2) implies
\begin{center}
${R_l}(\hat \theta_n ) -{R_l}(\tilde \theta )\le  {\E}[H_{{Y}, {{X}}}\|\hat\theta_n-\tilde\theta\|_{2}]\le  \kappa {\E}H_{{Y}, {{X}}}~~\hat \theta,\tilde\theta \in \Theta,$
\end{center}
which gives ${R_{l}^n}(\tilde \theta )\ge {R_l}(\hat \theta_n ) -\kappa{\E}H_{{Y}, {{X}}}$. Thus, \eqref{eq:cmgfupper1} has a further lower bound:
\begin{align*}
&~~~~{{\hat R}_{\psi_\lambda,l,\alpha}}(\hat \theta_n) \\
 &\ge {{ R}_{l}}(\hat \theta_n )- \{2\kappa {{\E}}H_{Y,X} + \frac{{{c_2}f(\alpha )}}{\alpha }R_{\lambda \circ l}(\Theta)+ \frac{{{c_2}f(\alpha \kappa )}}{\alpha }{{\E}}[\lambda(H_{Y,X})]    + \frac{1}{{n\alpha }}\log \frac{N({\Theta}, \kappa)}{\delta }\}
\end{align*}
with probability at least $1-\delta$. Then, we conclude Lemma \ref{eq:LOSSwhp2}.
\end{proof}

\subsection{The proof of Theorem \ref{thm1loss2}}
\begin{proof}
Let $\kappa=1/n$ in \eqref{eq:lem1lem2}, we get with probability at least $1 - 2\delta$
\begin{align*}
&~~~~{R_l}(\hat \theta_{n} ) -\mathop {\inf}\limits_{\theta  \in {\Theta}} {R_l}({\theta}) \\
& \le \frac{2}{n}{{\E}}H_{Y,X} + \frac{{{c_2}f(\alpha /n)}}{\alpha }{{\E}}[\lambda(H_{Y,X})]  +\frac{{(c_2+1)f(\alpha )}}{\alpha }R_{\lambda \circ l}(\Theta)+ \frac{1}{{n\alpha }}\log \frac{N(\Theta, 1/n)}{\delta^2}+\rho\|\Theta^*\|_2^2.
\end{align*}
The $\lambda(x)=|x|^{\beta}/{\beta},~{\beta} \in (1,2)$ satisfies weak triangle inequality and homogeneous inequality
\begin{align}\label{eq:weak}
|x+y|^{\beta}/{\beta}\le 2^{{\beta}-1}[|x|^{\beta}/{\beta}+|y|^{\beta}/{\beta}],~~ |tx|^{\beta}/{\beta}\le |t|^{\beta}\cdot |x|^{\beta}/{\beta}
\end{align}
by $|a+b|^{\beta} \le 2^{{\beta}-1}(|a|^{\beta}+|b|^{\beta})$ for ${\beta}>1$. Thus, we have $c_2=2^{{\beta}-1}$ and $f(t)=t^{\beta}$ for $t>0$ in (C.1).

Using the variance-bias tradeoff, i.e. the variance term equals to the bias term, put $\frac{{(c_2+1)f(\alpha )}}{\alpha }R_{\lambda \circ l}( \Theta) = \frac{1}{{n\alpha }}\log \frac{{N({\Theta},\kappa )}}{{{\delta ^2}}}$. Note that $f^{-1}(t)=t^{1/\beta}$ for $t>0$. So we have
\begin{center}
$\alpha = f^{ - 1}\left( \frac{1}{n({c_2} + 1)}R_{\lambda \circ l}^{ - 1}( \Theta)\log \frac{{N(\Theta ,\kappa )}}{{\delta ^2}} \right)=\frac{1}{n^{1/\beta}}\left(\frac{\log [{{N(\Theta ,\kappa )}}/{{\delta ^2}}] }{(2^{{\beta}-1} + 1)R_{\lambda \circ l}(\Theta)}\right)^{1/\beta}$.
\end{center}
Observe that $f(t)/t=t^{\beta-1}$ for $t>0$. Then
\begin{align*}
&~~~~\frac{{(c_2+1)f(\alpha )}}{\alpha }R_{\lambda \circ l}(\Theta)+ \frac{1}{{n\alpha }}\log \frac{N(\Theta ,1/n )}{\delta^2}=\frac{{2(c_2+1)f(\alpha )}}{\alpha }R_{\lambda \circ l}(\Theta)\\
&=\frac{{2({2^{\beta  - 1}} + 1)}}{{{n^{(\beta  - {\rm{1)}}/\beta }}}}{\left( {\frac{{\log [N(\Theta ,1/n )/{\delta ^2}]}}{{({2^{\beta  - 1}} + 1)R_{\lambda \circ l}(\Theta)}}} \right)^{(\beta  - {\rm{1)}}/\beta }}R_{\lambda \circ l}(\Theta)\\
& =\frac{{2({2^{\beta  - 1}} + 1)}}{{{n^{(\beta  - {\rm{1)}}/\beta }}}}{\left( {\frac{{\log [N(\Theta ,1/n )/{\delta ^2}]}}{{{2^{\beta  - 1}} + 1}}} \right)^{(\beta  - {\rm{1)}}/\beta }}{[R_{\lambda \circ l}(\Theta)]^{\beta^{-1}}}\\
& \le \frac{{2({2^{\beta  - 1}} + 1)}}{{{n^{(\beta  - {\rm{1)}}/\beta }}}}{\left( {\frac{{\log ({\delta ^{ - 2}}) + p\log \left( {1 + 2r/\kappa } \right)}}{{{2^{\beta  - 1}} + 1}}} \right)^{(\beta  - {\rm{1)}}/\beta }}{[R_{\lambda \circ l}(\Theta)]^{\beta^{-1}}},
\end{align*}
where the last inequality is by
\begin{align}\label{eq:covering}
\log \frac{{N(\Theta ,\kappa )}}{{{\delta ^2}}}& \le \log \frac{{N(B_{2}^{p}(r_n),\kappa )}}{{{\delta ^2}}}\le  \log (\frac{1}{{{\delta ^2}}}) + p\log \left( {1 + \frac{{2r_n}}{\kappa }} \right)
\end{align}
from Lemma \ref{net cardinality} and (C.2).

Thus, $\frac{{f(\alpha/n )}}{\alpha }={\alpha }^{\beta-1}n^{-\beta}$ and $\lambda(x)=|x|^{\beta}/{\beta}$ show
\begin{align*}
\frac{{{c_2}f(\alpha /n)}}{\alpha }{{\E}}[\lambda(H_{Y,X})] & = \frac{{{2^{\beta  - 1}/\beta}}}{{ {n^{1 - {\beta ^{ - 1}} + \beta }}}}{\left( {\frac{{\log [N(\Theta ,1/n )/{\delta ^2}]}}{{({2^{\beta  - 1}} + 1)R_{\lambda \circ l}(\Theta)}}} \right)^{(\beta  - {\rm{1)}}/\beta }}{{\E}}H_{Y,X}^{\beta} \\
 &\le \frac{{{2^{\beta  - 1}/\beta }}}{{ {n^{1 - {\beta ^{ - 1}} + \beta }}}}{\left( {\frac{{\log ({\delta ^{ - 2}}) + p\log \left( {1 + 2rn } \right)}}{{{2^{\beta  - 1}} + 1}}} \right)^{(\beta  - {\rm{1)}}/\beta }}\frac{{{{\E}} H_{Y,X}^{\beta}  }}{{{{[R_{\lambda \circ l}(\Theta)]}^{(\beta  - 1)/\beta }}}},
\end{align*}
where the last inequality is by \eqref{eq:covering}. Then we have
\begin{align*}
{R_l}(\hat \theta_{n} ) -\mathop {\inf}\limits_{\theta  \in {\Theta}} {R_l}({\theta})&\le \frac{2}{n}{{\E}}H_{Y,X} +\frac{1}{{{n^{{{\frac{\beta  - 1}{\beta} }}}}}}{\left( {\frac{{\log ({\delta ^{ - 2}}) + p\log \left( {1 + 2r_n n} \right)}}{{{2^{\beta  - 1}} + 1}}} \right)^{{\textstyle{{\beta  - {\rm{1}}} \over \beta }}}}\\
& \cdot \left[ {\frac{{{2^{\beta  - 1} }{{\E}}H_{Y,X}^\beta }}{{{\beta n^{\beta } }{{[R_{\lambda \circ l}(\Theta)]}^{{\textstyle{{\beta  - {\rm{1}}} \over \beta }}}}}} + 2({2^{\beta  - 1}} + 1){{[R_{\lambda \circ l}(\Theta)]}^{{\beta ^{ - 1}}}}} \right]\\
& \le \frac{2{{\E}}H_{Y,X}}{n}+C_{\beta,R_{\lambda \circ l}} \left[\frac{C_{\delta,n,r}(p) }{n}\right]^{\frac{\beta-1}{\beta}}+\rho\|\Theta^*\|_2^2
\end{align*}
with probability at least $1 - 2\delta$.
\end{proof}

\label{eq:elastic}

\subsection{The proof of Theorem \ref{thm:Elastic}}
\begin{proof}
Similar to the inequality \eqref{eq:lem1lem2} in the proof of Theorem \ref{thm1loss}, with ridge penalty replaced by elastic net penalty, we have for all $\theta^*\in \Theta^*$
\begin{align}\label{eq:lem1lem2elastic}
{R_l}(\hat \theta_n ) - {R_l}({\theta^*})& \le 2\kappa {{\E}}H_{Y,X} + \frac{{{c_2}f(\alpha \kappa )}}{\alpha }{{\E}}[\lambda(H_{Y,X})]  \nonumber\\
&  +\frac{{(c_2+1)f(\alpha )}}{\alpha }R_{\lambda \circ l}(\Theta)+\frac{1}{{n\alpha }}\log \frac{N({\Theta}, \kappa)}{\delta^2}+\rho\|\theta^*\|_2^2+\gamma\|\theta^*\|_1,
\end{align}
with probability at least $1 - 2\delta$.

For the term $\frac{1}{{n\alpha }}\log \frac{N({\Theta}, \kappa)}{\delta^2}$ in \eqref{eq:lem1lem2elastic} with $\kappa=1/n$, Lemma \ref{net cardinality1} implies
\begin{align*}
\log {N}\left(\Theta,1/n\right)\le \log {N}\left(B_{2}^{p}(r_n)\cap B_{0}^{p}(s_n),1/n\right) \leq \log\left(\frac{1}{ \sqrt{2 e s_n}}\right)+s_n\log\left[\frac{e(1+2r_nn)p}{s_n}\right].
\end{align*}
Let $f(t)=t^{\beta}$ and we have $c_2=2^{{\beta}-1}$, which shows that
\begin{align}\label{eq:h2elastic}
&~~~~{R_l}(\hat \theta_n ) - {R_l}({\theta^*})\le \rho\|\theta^*\|_2^2+\gamma\|\theta^*\|_1+\frac{2{{\E}}H_{Y,X}}{n} + \frac{{{c_2}f(\alpha \kappa )}}{\alpha }{{\E}}[\lambda(H_{Y,X})]\nonumber\\
& +\frac{{(c_2+1)f(\alpha )}}{\alpha }R_{\lambda \circ l}(\Theta)+\frac{1}{{n\alpha }}\left\{s_n\log\left[{e(1+2r_nn)}\frac{p}{s_n}\right]+\log\left(\frac{\delta^{-2}}{{ \sqrt{2 e s_n}}}\right)\right\}
\end{align}
with probability at least $1 - \delta$.

For the last two terms in \eqref{eq:h2elastic}, we put
\begin{center}
${\alpha^{\beta-1}}{{(c_2+1)}}R_{\lambda \circ l}(\Theta) = \frac{1}{{n\alpha }}\left\{s_n\log\left[{e(1+2r_nn)}\frac{p}{s_n}\right]+\log\left(\frac{\delta^{-2}}{{ \sqrt{2 e s_n}}}\right)\right\}$.
\end{center}
Then we obtain $\alpha =\frac{1}{n^{1/{\beta}}}\left(\frac{\log ({\delta ^{ - 2}}/\sqrt{2 e s_n}) + s_n\log\left[(1+2r_nn){ep}/{s_n} \right]}{(2^{\beta-1} + 1)R_{\lambda \circ l}(\Theta )}\right)^{{1}/{\beta}}$. Moreover, in \eqref{eq:lem1lem2elastic}
\begin{align*}
\frac{{{c_2}f(\alpha /n)}}{\alpha }{{\E}}[\lambda(H_{Y,X})] & = \frac{{{2^{\beta  - 1}/\beta}}}{{ {n^{1 - {\beta ^{ - 1}} + \beta }}}}{\left( {\frac{{\log [N(\Theta ,1/n )/{\delta ^2}]}}{{({2^{\beta  - 1}} + 1)R_{\lambda \circ l}(\Theta)}}} \right)^{(\beta  - {\rm{1)}}/\beta }}{{\E}}H_{Y,X}^{\beta} \\
 &\le \frac{{{2^{\beta  - 1}/\beta }}}{{ {n^{1 - {\beta ^{ - 1}} + \beta }}}}{\left({\log (\frac{\delta ^{ - 2}}{2 es_n}) + s_n\log\left[(1+2r_n n)\frac{ep}{s_n} \right]}\right)^{(\beta  - {\rm{1)}}/\beta }}\frac{{{{\E}} H_{Y,X}^{\beta}  }}{{{{[R_{\lambda \circ l}(\Theta)]}^{(\beta  - {\rm{1)}}/\beta }}}}.
\end{align*}
By the above terms, with probability at least $1 - 2\delta$, one has
\begin{small}
\begin{align*}
{R_l}(\hat \theta_{n} ) -\mathop {\inf}\limits_{\theta  \in {\Theta}} {R_l}({\theta})\le \frac{2{{\E}}H_{Y,X}}{n}+\frac{C_{\beta,R_{\lambda \circ l}}}{n^{\frac{\beta-1}{\beta}}}\left({\log (\frac{\delta ^{ - 2}}{2 e s_n}) + s_n\log\left[(1+2r_n n)\frac{ep}{s} \right]}\right)^{\frac{\beta-1}{\beta}}+\|{\Theta^*}\|_{\rho,\gamma},
\end{align*}
\end{small}
where $C_{\beta,R_{\lambda \circ l}}$ is a constant given in Theorem \ref{thm1loss2}, and $\|{\Theta^*}\|_{\rho,\gamma}:=\inf_{\theta^*\in \Theta^*}(\rho\|\theta^*\|_2^2+\gamma\|\theta^*\|_1)$.
\end{proof}

\subsection{The proof of Corollary \ref{Thm2}}\label{eq:knight}

We use Knight's identity \citep{knight1998limiting} to obtain the expression of $H_{y,x}$ for QR.
$$\rho_{\tau}(u-v)-\rho_{\tau}(u)=-v[\tau-{\rm{1}}(u<0)]+\int_{0}^{v}[{\rm{1}}(u \leq s)-{\rm{1}}(u \leq 0)] d s,$$
which provides a Taylor-like expansion for non-smooth function. On the one hand, we have
\begin{align}
\rho_{\tau}(y - x^ \top {\eta _1})-\rho_{\tau}(y - x^ \top {\eta _2})&=x^ \top ({\eta _2}-{\eta _1})[\tau-{\rm{1}}(y - x^ \top {\eta _2})]\nonumber\\
&+\int_{0}^{x^ \top ({\eta _1}-{\eta _2})}[{\rm{1}}(y - x^ \top {\eta _2} \leq s)-{\rm{1}}(y - x^ \top {\eta _2} \leq 0)] d s\nonumber\\
& \le |x^ \top ({\eta _2}-{\eta _1})|\cdot |\tau-{\rm{1}}(y - x^ \top {\eta _2})|\nonumber\\
&+\left| \int_{0}^{x^ \top ({\eta _1}-{\eta _2})}[{\rm{1}}(y - x^ \top {\eta _2} \leq s)-{\rm{1}}(y - x^ \top {\eta _2} \leq 0)] d s \right|\nonumber\\
& \le \max \{ \tau ,1 - \tau \}|x^ \top ({\eta _2}-{\eta _1})|+|x^ \top ({\eta _2}-{\eta _1})|\nonumber\\
& \le \max \{1+ \tau ,2 - \tau \}\|{\eta _2} - {\eta _1}\|_{2} \| x\|_{{2}}.\nonumber
\end{align}
Let $l_\tau:=\max \{1+ \tau ,2 - \tau \}$. Hence, $H_{y,x}=l_\tau\left\| x \right\|_2$.

\subsection{Remarks and proofs for GLMs}\label{se:GLMap}
\begin{remark}
The \eqref{eq-log-like} is originally derived from negative log-likelihood functions of exponential family. The exponential family contains many sub-exponential and sub-Gaussian distributions such as binomial, Poisson, negative binomial, Normal, Gamma distributions \citep{mccullagh1989generalized}. Let $\nu(\cdot)$ be some dominated measure and $b(\cdot)$ be a function. Consider a $Y$ follows the distribution of the  natural
exponential families $P_{\eta}$ indexed by parameter ${\eta}$
\begin{equation}\label{eq-glm}
P_{\eta}(dy)=c(y)\exp\{y{\eta}-b({\eta})\}\nu(dy),
\end{equation}
where the function $c(y)>0$ is free of ${\eta} \in \Xi  := \{{\eta} :\smallint c (y)\exp \{ y{\eta} \} \nu (dy) < \infty \}.$
\end{remark}

\begin{remark}\label{eq:GLMS}
For GLMs, the link function $u(x)$ is {canonical, i.e. $u(x)=x$,  whence we can choose ${g_A}(\cdot) \equiv1$ in (G1).
Note that in this case $k(t)=b(t)$, a choice of ${h_A}(\cdot)$ in condition (G2) is derived by $\ddot b(t)>0$
 \[\dot k(x^ \top \theta) \le  \dot b(r_n\left\| {x} \right\|_{2})  =:{h_{r_n}}(x),~\text{for}~\| \theta\|_2\le {r_n}~\text{due to}~| {x^ \top \theta} | \le {r_n}\left\| {x} \right\|_{2}.\]

 For GLM with non-canonical link function $u(x)$, we first choose ${g_A}(\cdot)$ in condition (G1) by
 $ \dot u(x^ \top \theta) \le  \dot u({r_n}\left\| x \right\|_{2})  =:{g_{r_n}}(x),~\text{for}~\| \theta\|_2\le {r_n}$ due to $| {x^ \top \theta} | \le {r_n}\left\| x \right\|_{2}.$ Under the (C.2) and (G.1), it implies (G.2) with $A={r_n}$ and ${h_{r_n}}(x) = {g_{r_n}}(x)\dot b(u({r_n}\left\| x \right\|_{2}))$ by the following inequality
 \[\dot k(x^ \top \theta) = \dot u(x^ \top \theta)\dot b(u(x^ \top \theta)) \le {g_{r_n}}(x)\dot b(u({r_n}\left\| x \right\|_{2})) := {h_{r_n}}(x),~\text{for}~\| \theta\|_2\le {r_n}.\]
Suppose the input $\{X_i\}_{i=1}^n$ is i.i.d. drawn from $X$, and $X$ is bounded (see \cite{yang2021law}). Under the (C.2), the $k(X^ \top \theta )$ and $u(X^ \top \theta )$ are also bounded,} then (G.4) is true under the finite second moments of output
\[R_{{l^2}}(\theta )= {{\E}}{[k(X_{}^ \top \theta ) - Yu(X^ \top \theta )]^2} \le 2{{\E}}{[k(X^ \top \theta )]^2} + 2{{\E}}{[Yu(X^ \top \theta )]^2} \le {C_1} + {C_2}{{\E}}Y^2<\infty\]
for $\| \theta\|_2\le {r_n}$, where $C_1$ and $C_2$ are some positive constants.
\end{remark}

From \eqref{eq-glm}, one can formally derive the quasi-GLMs loss in \eqref{eq-log-like}. Indeed, given $\{ X_i\}_{i=1}^{n}$, the conditional likelihood function of $\{Y_i\}_{i=1}^{n}$ is the product of $n$ terms in (\ref{eq-glm}) with $\eta_i:=u(X_{i}^ \top \theta)$, and the average negative log-likelihood function is
\begin{center}
${{\hat R}_l}(\theta ):=\frac{-1}{n}\sum_{i=1}^{n}[Y_{i}u({X}_{i}^ \top \theta)-b(u({X}_{i}^ \top \theta))]=\frac{1}{n}\sum_{i=1}^{n}l({Y_{i}} ,X_{i}^ \top \theta ),~\theta \in \mathbb{R}^{p}$.
\end{center}

The (C.2) can be obtained by a first-order Taylor expansion of $l(y,x,\cdot)$ as the following
\[l(y,x,{\eta _2}) = l(y,x,{\eta _1}) +({\eta _2} - {\eta _1})^\top \dot l[y,x,(t{\eta _2} + (1 - t){\eta _1})];{\eta _1},{\eta _2} \in \Theta,~\exists~t \in (0,1),\]
where $\dot l(y,x,\cdot)$ is a (sub-)gradient, we can choose a $H_{y,x}$ satisfying
\begin{equation}\label{eq:Hfunction}
H_{y,x}\ge {\sup }_{{\eta _1},{\eta _2} \in \Theta} \| {\dot l[y,x,(t{\eta _2} + (1 - t){\eta _1})]}\|_2.
\end{equation}
Fix a $\eta \in {\Theta}$, we compute the gradient for the loss function in \eqref{eq-log-like}
\begin{align*}
{\dot l(y,x,\eta )}:=\nabla_\eta l(y,x^\top\eta )= [-y\dot u(x^\top\eta)+\dot k(x^\top\eta) ]x^\top.
\end{align*}
From \eqref{eq:Hfunction}, $H_{y,x}$ in Theorem \ref{thm1loss2} is given by
\begin{align}\label{eq:compactGLM}
{\sup }_{{\eta _1},{\eta _2} \in \Theta} \| {\dot l[y,x,(t{\eta _2} + (1 - t){\eta _1})]}\|_2& \le \mathop {\sup }\limits_{\| \theta\|_2\le {r_n}} |-y\dot u(x^\top\eta)+\dot k(x^\top\eta) |\cdot\left\| x \right\|_2 \nonumber\\
&\le [|y|{g_{r_n}}(x)+{h_{r_n}}(x)]\left\| x \right\|_2 :=H_{y,x}
\end{align}
under condition (C.2), which implies the excess risk bound in Corollary \ref{eq:GLMThm2}.

Next, we provides two examples of $H_{y,x}$.

\textbf{Robust logistic regression}. We have  $u(t)=t,~k(t)=\log(1+e^{t})$ and $y\in \{0,1\}.$ Note that $H_{y,x}$ in Theorem \ref{thm1loss2} is given by
$${\sup }_{{\eta _1},{\eta _2} \in \Theta} \| {\dot l[y,x,(t{\eta _2} + (1 - t){\eta _1})]}\|_2 \le \mathop {\sup }\limits_s |y+ \dot k(s)|\cdot\left\| x \right\|_2 \le 2\left\| x \right\|_2:=H_{y,x}$$
from \eqref{eq:Hfunction} and \eqref{eq:compactGLM}, we have
$H_{y,x}^{\beta} = {2^{\beta}}{\left\| x \right\|_2^{\beta} }$. \\

\textbf{Robust negative binomial regression}. The connection of $u( \cdot )$ and $k( \cdot )$ of NBR is
$u(t)=t-\log(\eta+e^{t})$ and $k(t)=\eta\log(\eta+e^{t})$.
From \eqref{eq:Hfunction} and \eqref{eq:compactGLM}, we have via Theorem \ref{thm1loss2}
\begin{align}\label{eq:Hnbr}
{\sup }_{{\eta _1},{\eta _2} \in \Theta} \| {\dot l[y,x,(t{\eta _2} + (1 - t){\eta _1})]}\|_2 &\le \mathop {\sup }\limits_s |y+ \dot k(s)|\cdot\left\| x \right\|_2  \nonumber\\
& = (y+\eta)\left\| x \right\|_2:=H_{y,x},~y \ge 0.
\end{align}
Hence, we obtain
$H_{y,x}^{\beta}  = |(y+\eta)\left\| x \right\|_2|^{\beta}  \le {2^{\beta-1}}[{\left\| yx \right\|_2^{\beta} } + {(\eta\left\| x \right\|_2)^{\beta} }].$

\subsection{The proof of Theorem \ref{eq:DNNsThm3}}\label{eq:DNNPROOF}
For a fixed $L$, Lipschitz property of DNN function (Proposition 6 in \cite{taheri2021statistical}) implies the following excess risk guarantee for elastic net regularization DNN regression estimators. Since the neural network $\mathcal{NN}(N, L)$ in \eqref{eq:NN} has ReLU activation functions, and it has the approximation error promise \citep{schmidt2020nonparametric} grounded on the smooth assumption of ${f^*}$.

\begin{proof}
Let ${\hat R}_{\psi_\lambda,l,\alpha}(f):=\frac{1}{{n\alpha }}\sum_{i = 1}^n {\psi_\lambda}[\alpha l( Y_i, f_\theta(X_i) )]$. From the definition of $f_{\hat\theta_n}$, one has
\begin{align*}
&~~~~{{\hat R}_{\psi_\lambda,l,\alpha}}(f_{\hat\theta_n})+\rho\|\hat\theta_n\|_2^2+\gamma\|\hat\theta_n\|_1=\frac{1}{{n\alpha }}\sum\limits_{i = 1}^n \psi_\lambda  [\alpha l( Y_i, f_{\hat\theta_n}(X_i) )]+\rho\|\hat\theta_n\|_2^2+\gamma\|\hat\theta_n\|_1\nonumber\\
&\le \frac{1}{{n\alpha }}\sum\limits_{i = 1}^n \psi_\lambda  [\alpha l( Y_i, f_{\theta_{\mathcal{N}}^*}(X_i) )]+\rho\|\theta_{\mathcal{N}}^*\|_2^2+\gamma\|\theta_{\mathcal{N}}^*\|_1=:{{\hat R}_{\psi_\lambda,l,\alpha}}(\theta_{\mathcal{N}}^*)+\rho\|\theta_{\mathcal{N}}^*\|_2^2+\gamma\|\theta_{\mathcal{N}}^*\|_1
\end{align*}
which yields
\begin{align}\label{eq:deFDNN}
{{\hat R}_{\psi_\lambda,l,\alpha}}(f_{\hat\theta_n}) - {{\hat R}_{\psi_\lambda,l,\alpha}}(\theta_{\mathcal{N}}^*)&\le \gamma(\|\theta_{\mathcal{N}}^*\|_1-\|\hat\theta_n\|_1)+\rho(\|\theta_{\mathcal{N}}^*\|_2^2-\|\hat\theta_n\|_2^2)\nonumber\\
&\le \rho\|\theta_{\mathcal{N}}^*\|_2^2+\gamma\|\theta_{\mathcal{N}}^*\|_1.
\end{align}
Let $\F=\mathcal{NN}(N, L)$. The excess risk ${R}_l(f_{\hat\theta_n})- {R}_l (f^*)$ can be decomposed and bounded by,
\begin{align}
&{R}_l(f_{\hat\theta_n})- {R}_l (f^*)= \underbrace{{R_l}(f_{\hat\theta_n}) - {{\hat R}_{\psi_\lambda,l,\alpha}}(f_{\hat\theta_n})}_{Genaralization}   \nonumber\\
&+  \underbrace{{{\hat R}_{\psi_\lambda,l,\alpha}}(f_{\hat\theta_n}) - {{\hat R}_{\psi_\lambda,l,\alpha}}(f_{\theta_{\mathcal{N}}^*})}_{Optimaztion}+  \underbrace{{{\hat R}_{\psi_\lambda,l,\alpha}}(f_{\theta_{\mathcal{N}}^*}) - {R_l}(f_{\theta_{\mathcal{N}}^*})}_{Concentration}+  \underbrace{{{R}_l}(f_{\theta_{\mathcal{N}}^*}) - {R_l}(f^*)}_{Approximation}\nonumber\\
&\le[{{R_l}(f_{\hat\theta_n}) - {{\hat R}_{\psi_\lambda,l,\alpha}}(f_{\hat\theta_n})}] + [ {{{\hat R}_{\psi_\lambda,l,\alpha}}(f_{\theta_{\mathcal{N}}^*}) - {R_l}(f_{\theta_{\mathcal{N}}^*})} ]\nonumber\\
&+\inf\limits_{f \in \F} | {{{R}_l}(f)- {R_l}(f^*)}|+\rho\|\theta_{\mathcal{N}}^*\|_2^2+\gamma\|\theta_{\mathcal{N}}^*\|_1,\label{eq:decompDNN}
\end{align}
where the last inequality is from \eqref{eq:deFDNN} and ${{{R}_l}(f_{\theta_{\mathcal{N}}^*}) - {R_l}(f^*)}\le \inf\nolimits_{f \in \F} | {{{R}_l}(f)- {R_l}(f^*)}|$.

The second term in \eqref{eq:decompDNN} is the concentration error bound, which can be bounded from the same proof in Lemma \ref{eq:LOSSshp1} to get for all $f_{\theta_{\mathcal{N}}^*}$ with ${\theta^*} \in \Theta_{\mathcal{N}}^*$

\begin{align}\label{eq:h1}
\mathbb{P}\left\{ {{{{\hat R}_{\psi_\lambda,l,\alpha}}(f_{\theta_{\mathcal{N}}^*}) - {R_l}(f_{\theta_{\mathcal{N}}^*})}\le \frac{f(\alpha)}{\alpha}{R_{\lambda \circ l}}(f_{\theta_{\mathcal{N}}^*})+ \frac{1}{{n\alpha }}\log \frac{1}{\delta }}\right\}\ge 1 - \delta.
\end{align}

The first term in \eqref{eq:decompDNN} is the generalization error bound with a upper bound from the same proof in Lemma \ref{eq:LOSSwhp2} from  Lipschitz property of loss function. For every ${x} \in \mathbb{R}^{p}$ and parameters $\theta_1=\left(W_{0}, \ldots, W_{L}\right), {\theta _2}=\left(V_{0}, \ldots, V_{L}\right)$ in \eqref{eq:NN}. One can check Lipschitz property of DNNs
\begin{equation}\label{NN.xie}
\left| {{f_{{\theta _1}}}(x) - {f_{{\theta _2}}}(x)} \right| \le c_{\mathrm{Lip}}({x})\left\| {{\theta _2} - {\theta _1}} \right\|_{\mathrm{F}}
\end{equation}
 for a given function $c_{\mathrm{Lip}}({x}):=2\sqrt{L}\|{x}\|_{2} \max _{l \in\{0, \ldots, L\}} \prod_{j \in\{0, \ldots, L\}, j \neq l}\sigma_{\max }(W^{j}) \vee \sigma_{\max }(V^{j})$ by Proposition 6 in \cite{taheri2021statistical}.

Next, we consider the $s_n$-sparse parameter space $\Theta$ given in \eqref{eq:l1}. Suppose that $l(\cdot,\cdot)$ satisfies Lipschitz condition with a function $D_{x,y}$  
\begin{center}
$|l(y,f_{\theta _2}(x)) - l(y,f_{\theta _1}(x))|\le D_{x,y}{{|f_{\theta _2}(x)-f_{\theta _1}(x)|}},~\theta_1,~\theta_2 \in \Theta$.
\end{center}
From \eqref{NN.xie}, the loss function has Lipschitz property
\begin{center}
  $|l(y,f_{\theta _2}(x)) - l(y,f_{\theta _1}(x))|\le D_{x,y}c_{\mathrm{Lip}}({x})\|{\theta _2} - {\theta _1}\|_{\mathrm{F}}\le 2W^{L} \sqrt{L}D_{x,y}\|{x}\|_{2}\|{\theta _2} - {\theta _1}\|_{\mathrm{F}}$,
\end{center}
which shows that
\begin{equation}\label{eq:HDNN}
{H_{y,x}:=2W^{L} \sqrt{L}\|{x}\|_{2}D_{x,y}.}
\end{equation}

\textbf{Examples}: For robust DNN LAD regression, we have $D_{x,y}=1$ and thus $H_{y,x}=2W^{L} \sqrt{L}\|{x}\|_{2}$. For robust DNN logistic regression, it gives $D_{x,y}=y+1 \le 2$ and $H_{y,x}=4W^{L} \sqrt{L}\|{x}\|_{2}$. For robust DNN NBR, we get $D_{x,y}=y+\eta $ and $H_{y,x}=2W^{L} \sqrt{L}\|{x}\|_{2}(y+\eta)$.

By using Lipschitz constant \eqref{eq:HDNN}, under (C.1)-(C.4) and $f(t)=|t|^{\beta}$ and $c_2=2^{{\beta}-1}$, one has by Lemma \ref{eq:LOSSwhp2}
\begin{small}
\begin{align*}
&{{R_l}(f_{\hat\theta_n}) - {{\hat R}_{\psi_\lambda,l,\alpha}}(f_{\hat\theta_n})} \le 2\kappa {{\E}}H_{Y,X} + \frac{{2^{{\beta}-1}f(\alpha )}}{\alpha }R_{\lambda \circ l}( \Theta)+ \frac{{2^{{\beta}-1}f(\alpha \kappa )}}{\alpha }{{\E}}[\lambda(H_{Y,X})] + \frac{1}{{n\alpha }}\log \frac{N({\Theta}, \kappa)}{\delta }\\
& =4\kappa W^{L} \sqrt{L}{\E}[\|{X}\|_{2}D_{X,Y}]+(2\alpha)^{\beta-1}R_{\lambda \circ l}( \Theta)+(2\alpha)^{\beta-1}(2\kappa W^{L} \sqrt{L})^{\beta}\frac{{{\E}} |\|X\|_{2}D_{X,Y}|^{\beta}}{\beta}+\frac{1}{{n\alpha }}\log \frac{N({\Theta}, \kappa)}{\delta }
\end{align*}
\end{small}
with probability at least $1 - \delta$ for any $\kappa>0$. For the last term with covering number, we apply Lemma \ref{net cardinality1} to conclude that
\begin{align*}
\log {N}\left(\Theta,\kappa\right)\le \log {N}\left(B_{2}^p(r_n)\cap B_{0}^p(s_n),\kappa\right) \leq \log\left(\frac{1}{ \sqrt{2 e s_n}}\right)+s_n\log\left[\frac{e(\kappa+2r_n)p}{\kappa s_n}\right],
\end{align*}
which shows that
\begin{align}\label{eq:h2}
&~~~~{{R_l}(f_{\hat\theta_n}) - {{\hat R}_{\psi_\lambda,l,\alpha}}(f_{\hat\theta_n})} \le 4\kappa W^{L} \sqrt{L}{\E}[\|{X}\|_{2}D_{X,Y}]+(2\alpha)^{\beta-1}R_{\lambda \circ l}( \Theta)\nonumber\\
& +(2\alpha)^{\beta-1}(2\kappa W^{L} \sqrt{L})^{\beta}\frac{{{\E}} |\|X\|_{2}D_{X,Y}|^{\beta}}{\beta}+\frac{1}{{n\alpha }}\left\{s_n\log\left[{e(\kappa+2r_n)}\frac{p}{\kappa s_n}\right]+\log\left(\frac{\delta^{-1}}{{ \sqrt{2 e s_n}}}\right)\right\}
\end{align}
with probability at least $1 - \delta$ for any $\kappa>0$.

Under the two high provability events in \eqref{eq:h1} and \eqref{eq:h2}, inequality \eqref{eq:decompDNN} shows that
\begin{small}
\begin{align}\label{eq:lem1lem2DNN}
&~~~~{R}_l(f_{\hat\theta_n})- {R}_l (f^*)\nonumber\\
&\le \inf\limits_{f \in \F} | {{{R}_l}(f)- {R_l}(f^*)}|+\rho\|\theta_{\mathcal{N}}^*\|_2^2+\gamma\|\theta_{\mathcal{N}}^*\|_1+\left(\alpha^{\beta-1}{R_{\lambda \circ l}}(f_{\theta_{\mathcal{N}}^*}) + \frac{1}{{n\alpha }}\log \frac{1}{\delta }\right)+\frac{1}{{n\alpha }}\log \frac{N({\Theta}, \kappa)}{\delta }\nonumber\\
 &+4\kappa W^{L} \sqrt{L}{\E}[\|{X}\|_{2}D_{X,Y}]+(2\alpha)^{\beta-1}R_{\lambda \circ l}( \Theta)+(2\alpha)^{\beta-1}\kappa^{\beta}(2W^{L} \sqrt{L})^{\beta}\frac{{{\E}} |\|X\|_{2}D_{X,Y}|^{\beta}}{\beta}\nonumber\\
&\le  \rho\|\theta_{\mathcal{N}}^*\|_2^2+\gamma\|\theta_{\mathcal{N}}^*\|_1+\frac{1}{{n\alpha }}\left\{s_n\log\left[{e(\kappa+2r_n)}\frac{p}{\kappa s_n}\right]+\log\left(\frac{\delta^{-2}}{{ \sqrt{2 e s_n}}}\right)\right\}+(1+2^{\beta-1})\alpha^{\beta-1}R_{\lambda \circ l}( \Theta)\nonumber\\
&+4\kappa W^{L} \sqrt{L}{\E}[\|{X}\|_{2}D_{X,Y}]+(2\alpha)^{\beta-1}(2\kappa W^{L} \sqrt{L})^{\beta}\frac{{{\E}} |\|X\|_{2}D_{X,Y}|^{\beta}}{\beta}+\inf\limits_{f \in \F} | {{{R}_l}(f)- {R_l}(f^*)}|
\end{align}
\end{small}
with probability at least $1 - 2\delta$.

Let $\kappa=1/n$ and put $\frac{1}{{n\alpha }}\left\{\log({\delta ^{ - 2}}/\sqrt{2 e s_n})+s_n\log\left[{e(\kappa+2r_n)p}/(\kappa s_n)\right]\right\}= (1+2^{\beta-1})\alpha^{\beta-1}R_{\lambda \circ l}( \Theta)$ in \eqref{eq:lem1lem2DNN}, and it gives
$\alpha = \frac{1}{n^{1/\beta}}\left(\frac{\log ({\delta ^{ - 2}}/\sqrt{2 e s_n}) + s_n\log\left[{e(1+2nr_n)p}/s_n \right]}{(2^{\beta-1} + 1)R_{\lambda \circ l}(\Theta )}\right)^{1/\beta}$. Hence, \eqref{eq:lem1lem2DNN} implies by taking $\inf_{\theta_{\mathcal{N}}^*\in \Theta_{\mathcal{N}}^*}$ on the upper bound \eqref{eq:lem1lem2DNN}
\begin{align*}
&{R}_l(f_{\hat\theta_n})- {R}_l (f^*) \le \inf_{\theta_{\mathcal{N}}^*\in \Theta_{\mathcal{N}}^*}(\rho\|\theta_{\mathcal{N}}^*\|_2^2+\gamma\|\theta_{\mathcal{N}}^*\|_1)+\frac{2}{{n\alpha }}\left\{\log({\delta ^{ - 2}}/\sqrt{2 e s_n})+s_n\log\left[{(1+2nr_n)ep}/s_n\right]\right\}\nonumber\\
&+\frac{4W^{L} \sqrt{L}}{n}{\E}[\|{X}\|_{2}D_{X,Y}]+\frac{2^{\beta-1}\alpha^{\beta-1}}{n^{\beta}}\frac{{{\E}} |\|X\|_{2}D_{X,Y}|^{\beta}}{\beta}+\inf\limits_{f \in \F} | {{{R}_l}(f)- {R_l}(f^*)}|\\
 \le & \inf_{\theta_{\mathcal{N}}^*\in \Theta_{\mathcal{N}}^*}(\rho\|\theta_{\mathcal{N}}^*\|_2^2+\gamma\|\theta_{\mathcal{N}}^*\|_1)+\frac{2\left\{\log({\delta ^{ - 2}}/\sqrt{2 e s_n})+s_n\log\left[{e(1+2nr_n)p}/s_n\right]\right\}^{1-\beta^{-1}}}{{n^{1-\beta^{-1}} }(2^{\beta-1} + 1)^{1-\beta^{-1}}R_{\lambda \circ l}^{1-\beta^{-1}}(\Theta )} {(2^{\beta-1} + 1)R_{\lambda \circ l}(\Theta )}\nonumber\\
&+\frac{2^{\beta-1}(2W^{L} \sqrt{L})^{\beta}}{n^{1 - {\beta ^{ - 1}} + \beta }}\left(\frac{\log ({\delta ^{ - 2}}/\sqrt{2 e s_n}) + s_n\log\left[{e(1+2nr_n)p}/s_n \right]}{(2^{\beta-1} + 1)R_{\lambda \circ l}(\Theta )}\right)^{1-\beta^{-1}}\frac{{{\E}} |\|X\|_{2}D_{X,Y}|^{\beta}}{\beta}\\
&+\frac{4W^{L} \sqrt{L}}{n}{\E}[\|{X}\|_{2}D_{X,Y}]+\inf\limits_{f \in \F} | {{{R}_l}(f)- {R_l}(f^*)}|\\
 \le & \frac{4W^{L} \sqrt{L}}{n}{\E}[\|X\|_{2}D_{X,Y}]+\frac{F_{\beta,L,W}(R_{\lambda \circ l})}{n^{\frac{1-\beta}{\beta}}}\left[{\log ({\delta ^{ - 2}}/\sqrt{2 e s_n}) + s_n\log\left[{e(1+2nr_n)p}/s_n \right]}\right]^{\frac{1-\beta}{\beta}}\nonumber\\
&+\inf_{\theta_{\mathcal{N}}^*\in \Theta_{\mathcal{N}}^*}(\rho\|\theta_{\mathcal{N}}^*\|_2^2+\gamma\|\theta_{\mathcal{N}}^*\|_1)+\inf\limits_{f \in \F} | {{{R}_l}(f)- {R_l}(f^*)}|
\end{align*}
with probability at least $1 - 2\delta$, where
\begin{center}
$F_{\beta,L,W}(R_{\lambda \circ l}):=\left[2{(2^{\beta-1} + 1)R_{\lambda \circ l}(\Theta )}+\frac{(4W^{L} \sqrt{L})^\beta}{2\beta }{{{\E}} |\|X\|_{2}D_{X,Y}|^\beta}\right]/{[(2^{\beta-1} + 1)R_{\lambda \circ l}(\Theta )]^{\frac{\beta-1}{\beta}}}$.
\end{center}

Let us treat the approximation error bound of $\inf\nolimits_{f \in \F} | {{{R}_l}(f)- {R_l}(f^*)}|$ under (D.2). By the Lipschitz condition (D.1), we have
\begin{align*}
|{{{R}_l}(f)- {R_l}(f^*)}|:&= |\E [ l ( Y,f(X) )-l ( Y,f^*(X) )]|\le  \E [ D_{X,Y} \cdot |f(X)- f^*(X)|]\\
& \le (\E D_{X,Y}^2)^{1/2}\cdot(\E|f(X)- f^*(X)|^2)^{1/2}=(\E D_{X,Y}^2)^{1/2}\left\|f(X)-f^{*}(X)\right\|_{L^{2}(\nu)}
\end{align*}
from which,
\begin{align}\label{eq:dnnn}
\inf\limits_{f \in \F} | {{{R}_l}(f)- {R_l}(f^*)}|\le (\E D_{X,Y}^2)^{1/2}\inf\limits_{f \in \F}\left\|f(X)-f^{*}(X)\right\|_{L^{2}(\nu)}.
\end{align}

The sparse Relu DNN has approximation power in terms of $\inf\nolimits_{f \in \F}\left\|f(X)-f^{*}(X)\right\|_{L^{2}(\nu)}$ by tuning the width and depth. Theorem 3 in \cite{schmidt2020nonparametric} shown the approximation ability of the sparse ReLU DNN under H{\"o}lder functional class with smoothness index $\gamma$: \emph{Given a continuous function $f^* \in \mathcal{C}^\gamma([0,a_n]^{d},B)$ with a sequence $a_n\ge 1$, there exists a function $f$ implemented by a ReLU network with width $N=6(d+\lceil\gamma\rceil) M$, ($M\ge (\gamma+1)^{d} \vee(B+1) e^{d}$), and depth
\begin{center}
$L=8+(m+5)\left(1+\left\lceil\log _{2}(d \vee \gamma)\right\rceil\right)$ with an integer $m \ge 1$
\end{center}
and number of parameters $s \leq 141(d+\gamma+1)^{3+d} M(m+6)$ such that
\begin{equation}\label{eq:APP}
\|f-f^*\|_{L^{\infty}([0,a_n]^{d})} \leq \frac{(2B+1)\left(1+d^{2}+\gamma^{2}\right) M 6^{d}a_n^{\gamma}}{ 2^{m}}+\frac{K 3^{\gamma}a_n^{\gamma}} {N^{\frac{\gamma}{d}}}.
\end{equation}}

Under (D.3), we define the event $\mathcal{A}_n:=\{\|X\|_{\infty} \le a_n \}$ for a given non-decreasing and positive sequence $\{a_n\}$. Then, Markov's inequality shows
\begin{equation}\label{eq:An}
\mathbb{P}\left(\mathcal{A}_n\right) \ge 1-\frac{b}{a_n}.
\end{equation}
Under the restriction $\|f-f^*\|_{\infty}\le F<\infty$ with ${f \in \mathcal{N}\mathcal{N}(N, L)}$, it leads to
\begin{align*}
\E|f(X)- f^*(X)|^2&=\int_{\{\|x\|_{\infty} \le a_n \}} |f(x)- f^*(x)|^2\nu(dx)+\int_{\{\|x\|_{\infty} > a_n \}} |f(x)- f^*(x)|^2\nu(dx)\\
& \le \|f-f^*\|_{L^{\infty}([0,a_n]^d)}^2+F^2\mathbb{P}(\mathcal{A}_n^c)\\
&\le \left[\frac{(2B+1)\left(1+d^{2}+\gamma^{2}\right)  6^d M}{ 2^{m}}+\frac{B 3^{\gamma}} {N^{\frac{\gamma}{d}}}\right]^2 a_n^{2\gamma}+F^2\frac{b}{a_n},
\end{align*}
where the last inequality is from \eqref{eq:APP} and \eqref{eq:An}. The \eqref{eq:dnnn} gives
\begin{align*}
\inf\limits_{f \in \F} | {{{R}_l}(f)- {R_l}(f^*)}|&\le (\E D_{X,Y}^2)^{\frac{1}{2}}\left(\left[\frac{(2B+1)\left(1+d^{2}+\beta^{2}\right)  6^{d}M}{ 2^{m}}+\frac{B 3^{\gamma}} {N^{\frac{\gamma}{d}}}\right]^2a_n^{2\gamma}+F^2\frac{b}{a_n}\right)^{\frac{1}{2}} \\
&\le (\E D_{X,Y}^2)^{\frac{1}{2}}\left(\left[\frac{(2B+1)\left(1+d^{2}+\gamma^{2}\right)  6^{d}M}{ 2^{m}}+\frac{B 3^{\gamma}} {N^{\frac{\gamma}{d}}}\right]a_n^{\gamma}+(\frac{b}{a_n})^{\frac{1}{2}}F\right)
\end{align*}
Finally, we put
$\left[\frac{(2B+1)\left(1+d^{2}+\beta^{2}\right)  6^dM}{ 2^{m}}+\frac{3^{\gamma}B} {N^{\frac{\gamma}{p}}}\right]a_n^{\gamma}=(\frac{b}{a_n})^{\frac{1}{2}}F$, which implies
\begin{center}
$a_n=\frac{b^{\frac{1}{2\gamma+1}}F^{\frac{2}{2\gamma+1}}}{\left[\frac{(2B+1)\left(1+d^{2}+\gamma^{2}\right)  6^dM}{ 2^{m}}+\frac{ 3^{\gamma}B} {N^{{\gamma}/d}}\right]^{\frac{2}{2\gamma+1}}}$.
\end{center}
Then
\begin{align*}
\inf\limits_{f \in \F} | {{{R}_l}(f)- {R_l}(f^*)}|&\le 2(\E D_{X,Y}^2)^{\frac{1}{2}}({b}/{a_n})^{\frac{1}{2}}F\\
&= 2(\E D_{X,Y}^2)^{\frac{1}{2}}F^{\frac{2}{2\gamma+1}}{b^{\frac{r}{2\gamma+1}}}{\left[\frac{(2B+1)\left(1+d^{2}+\gamma^{2}\right)  6^d M}{ 2^{m}}+\frac{ 3^{\gamma}B} {N^{{\gamma}/d}}\right]^{\frac{1}{2\gamma+1}}}.
\end{align*}
\end{proof}

\subsection{The proof of Corollary \ref{coroln}}
\begin{proof}
Based on \eqref{eq:DNNER}, if the depth and width of DNNs is increasing with $n$, i.e. $L = {L_n}$ and $N = {N_n}$, it allows $\frac{{{W^{L_n}}\sqrt {L_n} }}{n} \le {C}( {{s_n\log(npr_n)}}/{n} )^{(\beta-1)/\beta}$ for a constant $C>0$. Then the inequality
\begin{center}
${\rm{log}}W\cdot{L_n} + \frac{1}{2}\log {L_n} \le \log C+ \frac{\beta-1}{\beta}[\log( {{s_n\log(npr_n)}})-\log n]+\log n$
\end{center}
guarantees the rate $O( {{s_n\log(npr_n)}}/{n} )^{(\beta-1)/\beta}$ excess risk if ${L_n}\lesssim  \frac{ \log( {{s_n\log(npr_n)}})}{\log W} + \frac{\log n}{\log W}$.  The sparsity level $s_n=o(n)$ performs as an effective dimension in DNN parameter that plays a key role to determine the main term in the excess risk bound. According to restriction $s_n=o(n)$, we put a \emph{depth-sample condition}:
\begin{equation}\label{eq:Ln}
{L_n} \lesssim  \frac{\log n+{\log( {{s_n\log(npr_n)}})}}{\log W}
\end{equation}
for a sufficient large $n$ and constant.

Under the H{\"o}lder functional class, we assume that the approximation error has a same or smaller order as the statistical error $O( {( {{s_n\log(npr_n)}}/{n} )^{(\beta-1)/\beta}})$. Suppose that
\begin{center}
$\frac{(2B+1)\left(1+d^{2}+\gamma^{2}\right)  6^d M}{ 2^{m}}\lesssim\frac{ 3^{\gamma}B} {N_n^{{\gamma}/d}}$,
\end{center}
which implies $(\gamma/d)\log_2 N_n \lesssim m$.

The condition of $L_n$ in Theorem 3 in \cite{schmidt2020nonparametric} requires that $L_n \lesssim 8+(\log n+{\log(s_n\log(npr_n))}+5)\left(1+\left\lceil\log _{2}(d \vee \gamma)\right\rceil\right)$, which coincides \eqref{eq:Ln}. To obtain the rate $O( {{s_n\log(npr_n)}}/{n} )^{(\beta-1)/\beta}$ approximation error, one must have
\begin{align*}
\inf\limits_{f \in \F} | {{{R}_l}(f)- {R_l}(f^*)}|&\lesssim {b^{\frac{r}{2\gamma+1}}}{\left[\frac{ 3^{\gamma}B} {N_n^{{\gamma}/d}}\right]^{\frac{1}{2\gamma+1}}}\lesssim ( {{s_n\log(npr_n)}}/{n} )^{(\beta-1)/\beta}.
\end{align*}
It leads to the \emph{width-sample condition}
\begin{equation*}\label{eq:Nn}
N_n \gtrsim b^d\left[\frac{n}{{{s_n\log(npr_n)}}} \right]^{{\frac{d(2\gamma+1)}{\gamma}}\cdot\frac{\beta-1}{\beta}}.
\end{equation*}
 So, $(\gamma/d )\log_2 N_n \lesssim m \lesssim \log n+{\log(s_n\log(npr_n))}$ and $L_n \lesssim 8+(\log n+{\log(s_n\log(npr_n))}+5)\left(1+\left\lceil\log _{2}(d \vee \gamma)\right\rceil\right)$ implies the rate $( {{s_n\log(npr_n)}}/{n} )^{(\beta-1)/\beta}$ approximation error, under the order of tuning parameters $\rho \vee \gamma \lesssim  ( {{s_n\log(npr_n)}}/{n} )^{(\beta-1)/\beta}$.
\end{proof}

\subsection{Robust two-component mixed linear regression}\label{mixed}
For convenience, we define a combined parameter $\theta : = \left( {\pi ,{\eta _0},{\eta _1}} \right)$, which is restricted in following space $\Theta \subset {\mathbb{R}^{1 + 2d}}$ with $r>1$.
\[\theta \in \Theta : = \left\{ \left( {a,{b_0},{b_1}} \right) \in {\mathbb{R}^{1 + 2p}}:0<\rho  \le a \le 1 - \rho, \max(\left\| {{b_1}} \right\|_2^{},\left\| {{b_2}} \right\|_2) \le u,~u: = \sqrt {({r^2} - 1)/2} \right\}.\]
Define $\hat \theta_n$ by \eqref{eq:cantonil2} with $\lambda(x)=\frac{1}{\beta}|x|^{\beta},~\beta \in (1,2)$, and ${\theta^*}$ is given by \eqref{eq:ture} with loss $l(y,x,\cdot)$ given by \eqref{eq:mix}.

We now identify the $H_{y,x}$ in Theorem \ref{thm1loss2}. Denote $p_k:=p\left(y,{x^ \top }\eta_{k}\right)$ for $k=0,1$. One has
\begin{align*}
\left| {\frac{{\partial l(y,x,\theta )}}{{\partial \pi }}} \right|& = \left| { - \frac{{p(y,{x^ \top }{\eta _0}) - (y,{x^ \top }{\eta _1})}}{{\pi p(y,{x^ \top }{\eta _0}) + (1 - \pi )p(y,{x^ \top }{\eta _1})}}} \right| = \left| {\frac{{{p_0}}}{{\pi {p_0} + (1 - \pi ){\rho _1}}} + \frac{{{p_1}}}{{\pi {p_0} + (1 - \pi ){p_1}}}} \right|\\
& \le \frac{{{p_0}}}{{\pi {p_0}}} + \frac{{{p_1}}}{{(1 - \pi ){p_1}}} = \frac{1}{{\pi (1 - \pi )}} \le \frac{1}{{\rho (1 - \rho )}},\quad (\rho  \le \pi  \le 1 - \rho ).
\end{align*}
Let $\nabla p(y,{x^ \top }{\eta _k}): = \frac{\partial }{{\partial t}}{\left. {p(y,t)} \right|_{t = {x^ \top }{\eta _k}}}$ and $\nabla \log p(y,{x^ \top }{\eta _k}): = \frac{\partial }{{\partial t}}\log {\left. {p(y,t)} \right|_{t = {x^ \top }{\eta _k}}}$. For $j=1,\cdots, d$, we have
\[\left| {\frac{{\partial l(y,x,\theta )}}{{\partial {\eta _{j0}}}}} \right| = \left| {\frac{{ - x_j\dot p(y,{x^ \top }{\eta _0})}}{{\pi p(y,{x^ \top }{\eta _0}) + (1 - \pi )p(y,{x^ \top }{\eta _1})}}} \right| \le \left| {\frac{{ - x_j\dot p(y,{x^ \top }{\eta _0})}}{{\pi p(y,{x^ \top }{\eta _0})}}} \right| \le \frac{{\left| x_j \right| \cdot \left| {\nabla \log p(y,{x^ \top }{\eta _0})} \right|}}{\rho }\]
and
\[\left| {\frac{{\partial l(y,x,\theta )}}{{\partial {\eta _{j1}}}}} \right| = \left| {\frac{{ - x_j\dot p(y,{x^ \top }{\eta _1})}}{{\pi p(y,{x^ \top }{\eta _0}) + (1 - \pi )p(y,{x^ \top }{\eta _1})}}} \right| \le \left| {\frac{{ - x_j\dot p(y,{x^ \top }{\eta _0})}}{{(1 - \pi )p(y,{x^ \top }{\eta _0})}}} \right| \le \frac{{\left| x_j \right| \cdot \left| {\nabla \log p(y,{x^ \top }{\eta _1})} \right|}}{{1 - \rho }}.\]
According to \eqref{eq:Hfunction},
\begin{align*}
&~~~~\mathop {\sup }\limits_{{\eta _1},{\eta _2} \in \Theta } {\| {\dot l[y,x,(t{\eta _2} + (1 - t){\eta _1})]} \|_2}  \le \mathop {\sup }\limits_{\theta  \in \Theta } {\| {\dot l(y,{x^ \top }\theta )}\|_2}\\
& \le \mathop {\sup }\limits_{\theta  \in \Theta } {\left[ {\left\| x \right\|_2^2\{{\rho ^{ - 2}}|\nabla \log p(y,{x^ \top }{\eta _1})|^2\} + {{(1 - \rho )}^{ - 2}}|\nabla \log p(y,{x^ \top }{\eta _1}){|^2} + {\rho ^{ - 2}}{{(1 - \rho )}^{ - 2}}} \right]^{1/2}}.\\
& \le \frac{1}{{\rho (1 - \rho )}}+ \frac{{{{\left\| x \right\|}_2}}}{\rho }\mathop {\sup }\limits_{{\eta _0} \in B_{2}^{d}(u)} |\nabla \log p(y,{x^ \top }{\eta _0})| + \frac{{{{\left\| x \right\|}_2}}}{{1 - \rho }}\mathop {\sup }\limits_{{\eta _1} \in B_{2}^{d}(u)} |\nabla \log p(y,{x^ \top }{\eta _1})|: = H(y,x).
\end{align*}

From the proof Theorem \ref{thm1loss2}, one has
\begin{small}
\begin{align*}
&{R_l}(\hat \theta_{n} ) -\mathop {\inf}\limits_{\theta  \in {\Theta}} {R_l}({\theta}) \le \frac{2}{n}\left[ {\frac{{{\rm{E[}}{{\left\| X \right\|}_2}\mathop {\sup }\limits_{{\eta _0} \in B_{2}^{d}(u)} |\nabla \log p(Y,{X^ \top }{\eta _0})| + 1]}}{\rho } + \frac{{{\rm{E[}}{{\left\| X \right\|}_2}\mathop {\sup }\limits_{{\eta _1} \in B_{2}^{d}(u)} |\nabla \log p(Y,{X^ \top }{\eta _1})| + 1]}}{{1 - \rho }}} \right]\\
 &+ {\left( {\frac{{\log ({\delta ^{ - 2}}) + (2d+1)\log \left( {1 + 2rn} \right)}}{n({2^{\beta-1}} + 1)}} \right)^{\frac{\beta-1}{\beta}}} \times \left[ \frac{{2({2^{\beta-1} } + 1)}}{{{{[{R_{\lambda  \circ l}}(\Theta )]}^{ - {\beta^{ - 1}}}}}} +  \right.\frac{{{2^{2(\beta  - 1)}}}}{{{{[{R_{\lambda  \circ l}}(\Theta )]}^{(\beta-1)/\beta}}}}\\
 &\times \left. {\left( {\frac{{{\rm{E}}[{{\left\| X \right\|}_2}\mathop {\sup }\limits_{{\eta _0} \in B_{2}^{d}(u)} |\nabla \log p(Y,{X^ \top }{\eta _0})| + 1]^\beta}}{{{\rho ^\beta}}} + \frac{{{\rm{E}}[{{\left\| X \right\|}_2}\mathop {\sup }\limits_{{\eta _1} \in B_{2}^{d}(u)} |\nabla \log p(Y,{X^ \top }{\eta _1})| + 1]^\beta}}{{{{(1 - \rho )}^\beta}}}} \right)} \right]+\rho\|\Theta^*\|_2^2.
\end{align*}
\end{small}
with probability at least $1 - 2\delta$.

For example, in Gaussian mixture regressions of two component, we have
\[p(y,{x^ \top }{\eta _k}) = \frac{1}{{\sqrt {2\pi } \sigma }}\exp \{ \frac{{{{(y - {x^ \top }{\eta _k})}^2}}}{{2{\sigma _k^2}}}\}~\text{and}~\nabla \log p(y,{x^ \top }{\eta _k}) =  - \frac{{y - {x^ \top }{\eta _k}}}{{{\sigma_k^2}}},~k=0,1.\]
{Without loss of generality, we assume that variance parameter ${\sigma_k^2}\equiv{\sigma ^2}$ is known.} To obtain rate $O( {{{( {{{(2d+1)\log (nr_n)}}/{n}} )}^{(\beta-1) /\beta}}})$ excess risk, we require moment conditions
\begin{center}
${\rm{E}}({\left\| X \right\|_2}|Y|)^{\beta}< \infty$ and ${\rm{E}}\left\| X \right\|_2^{2\beta}< \infty$
\end{center}
by noticing
${\rm{E[}}{\left\| X \right\|_2}\mathop {\sup }\limits_{{\eta _k} \in B_{2}^{d}(u)} |Y - {X^ \top }{\eta _k}|/{\sigma ^2} + 1{]^\beta} < {C_1}{\rm{E(}}{\left\| X \right\|_2}|Y|)^\beta + {C_2}{\rm{E}}{\left\| X \right\|_2^{2\beta}} < \infty,$ where $C_1$ and $C_2$ are some positive constants.

\subsection{Simulation results of negative binomial regression models}\label{se:appS}

We use the same noise settings for the NBR models as the simulations of logistic regression. And the dispersion parameter $\eta$ is set to be $20$. For the network configuration, we use the ReLU activated 5-layers DNN model with width $(d,20d,15d, 10d, 5d, 1)$ to train the DNN NBR when the real function $f^*$ is the complex function. For the case of DNN-based $f^*$, we adopt the ReLU activated 2-layers DNN model with width $(d,0.6d,0.4d,1)$ to train the model.
In Tables \ref{tab:NBPareto} and \ref{tab:NBGAU}, we compute the average $\ell_2$-estimation errors for the predicted coefficients of each normal NBR models with 100 replications.  Table \ref{tab:DNNNB} presents the absolute average errors (MAEs) of the response predictors $\{\hat{Y}_i\}_{i=1}^n$ with 100 replications, that is defined as
$$ {\rm{MAE}}=\frac{1}{n}\sum_{i=1}^{n}|\hat{Y}_i-Y_i|.$$
\begin{table}[htbp]
	\centering
	\caption{Comparison of average $\ell_2$-estimation error for NBR on Pareto noise model.}
	\resizebox{1\textwidth}{!}{
	\begin{tabular}{ccccccccc}
		\toprule
		\multicolumn{9}{c}{$\ell_2$-estimation error for NBR} \\
		\midrule
		&       & \multicolumn{3}{c}{$\varsigma=0$} &       & \multicolumn{3}{c}{$\varsigma=0.5$} \\
		\cmidrule{3-5}\cmidrule{7-9}    \multicolumn{9}{c}{$d = 100, n = 200$} \\
		\midrule
		Pareto & $\beta$   & High-order  & Cauchy & Non-truncation &       & High-order  & Cauchy & Non-truncation \\
		\midrule
		1.60   & 1.5   & 2.7324(0.0287) & 3.1469(0.0401) & 3.6866(0.0742) &       & 2.9015(0.1450) & 3.1510(0.1566) & 3.8726(0.0081) \\
		1.80   & 1.5   & 2.6656(0.0228) & 2.9910(0.0189) & 3.7492(0.0621) &       & 3.0704(0.1125) & 3.1675(0.1286) & 3.8732(0.0163) \\
		2.01  & 2.0     & 2.2859(0.0228) & 2.6282(0.0194) & 3.6218(0.0637) &       & 3.0658(0.1080) & 3.0603(0.1523) & 3.8086(0.0470) \\
		4.01  & 2.0     & 2.4727(0.0167) & 2.7543(0.0159) & 3.6441(0.0181) &       & 3.0460(0.1554) & 3.0847(0.1161) & 3.8646(0.0354) \\
		6.01  & 2.0     & 2.4430(0.0124) & 2.5573(0.0173) & 3.6831(0.0118) &       & 3.0083(0.1802) & 3.0642(0.1178) & 3.8940(0.0148) \\
		\midrule
		\multicolumn{9}{c}{$d = 200, n = 500$} \\
		\midrule
		Pareto & $\beta$   & High-order  & Cauchy & Non-truncation &       & High-order  & Cauchy & Non-truncation \\
		\midrule
		1.60   & 1.5   & 4.1902(0.9830) & 4.5767(0.5005) & 4.6152(0.4303) &       & 4.3648(0.0845) & 4.4987(0.1371) & 4.7120(0.0656) \\
		1.80   & 1.5   & 4.6652(0.5420) & 4.8640(0.5153) & 4.9682(0.3303) &       & 4.4122(0.0846) & 4.5386(0.1631) & 4.7052(0.0807) \\
		2.01  & 2.0     & 4.4903(0.1020) & 4.8618(0.1221) & 4.8653(0.4294) &       & 4.3573(0.0693) & 4.5685(0.1403) & 4.6675(0.1002) \\
		4.01  & 2.0     & 4.3130(0.1572) & 4.5597(0.1459) & 3.6799(0.2815) &       & 4.4048(0.1123) & 4.5182(0.0429) & 4.5985(0.0328) \\
		6.01  & 2.0     & 4.3355(0.1309) & 4.5255(0.1164) & 3.6624(0.0214) &       & 4.3008(0.1350) & 4.4457(0.0967) & 4.6910(0.0552) \\
		\midrule
		\multicolumn{9}{c}{$d = 1000, n = 1000$} \\
		\midrule
		Pareto & $\beta$   & High-order  & Cauchy & Non-truncation &       & High-order  & Cauchy & Non-truncation \\
		\midrule
		1.60   & 1.5   & 10.9820(0.1000) & 11.0961(0.1007) & 11.6957(0.0492) &       & 10.7608(0.1608) & 10.9988(0.1585) & 11.0076(0.1068) \\
		1.80   & 1.5   & 10.9745(0.0654) & 11.0259(0.1234) & 11.6771(0.0725) &       & 10.7184(0.0757) & 10.9624(0.1096) & 11.0044(0.0688) \\
		2.01  & 2.0     & 10.8991(0.0686) & 10.9826(0.0453) & 11.6676(0.0548) &       & 10.7373(0.1280) & 10.8823(0.1073) & 10.9500(0.0702) \\
		4.01  & 2.0     & 10.8049(0.0121) & 10.9907(0.1753) & 11.6273(0.0534) &       & 10.7024(0.0916) & 10.7892(0.0985) & 10.9563(0.0904) \\
		6.01  & 2.0     & 10.9433(0.0876) & 10.9645(0.1369) & 11.6186(0.0469) &       & 10.6301(0.0708) & 10.7471(0.1166) & 10.9544(0.0734) \\
		\bottomrule
	\end{tabular}}%
	\label{tab:NBPareto}%
\end{table}%

\begin{table}[htbp]
	\centering
\caption{Comparison of average $\ell_2$-estimation error for NBR on Uniform noise model.}
\resizebox{1\textwidth}{!}{
	\begin{tabular}{ccccccccc}
		\toprule
		\multicolumn{9}{c}{$\ell_2$-estimation error for NBR} \\
		\midrule
		&       & \multicolumn{3}{c}{$\varsigma=0$} &       & \multicolumn{3}{c}{$\varsigma=0.5$} \\
		\cmidrule{3-5}\cmidrule{7-9}    \multicolumn{9}{c}{$d = 100, n = 200$} \\
		\midrule
		Uniform & $\beta$   & High-order  & Cauchy & Non-truncation &       & High-order  & Cauchy & Non-truncation \\
		\midrule
		0.3   & 2.0     & 1.0241(0.0432) & 1.6514(0.0471) & 3.1721(0.1542) &       & 3.0282(0.1629) & 3.0570(0.3486) & 3.2361(0.1837) \\
		0.5   & 2.0      & 1.2061(0.0514) & 1.7425(0.0651) & 3.2321(0.1345) &       & 3.0229(0.1415) & 3.2562(0.1371) & 3.2746(0.0520) \\
		0.8   & 2.0      & 1.3454(0.0453) & 1.9241(0.0645) & 3.5432(0.1124) &       & 3.0570(0.3486) & 3.4104(0.1218) & 3.5124(0.1273) \\
		\midrule
		\multicolumn{9}{c}{$d = 200, n = 500$} \\
		\midrule
		Uniform & $\beta$   & High-order  & Cauchy & Non-truncation &       & High-order  & Cauchy & Non-truncation \\
		\midrule
		0.3   & 2.0      & 3.6431(0.0515) & 3.9321(0.0541) & 5.2354(0.1345) &       & 4.6455(0.1201) & 4.9514(0.1011) & 5.0917(0.1902) \\
		0.5   & 2.0      & 2.9241(0.0762) & 3.0235(0.0785) & 5.4252(0.1645) &       & 4.8460(0.1309) & 5.0286(0.1208) & 5.3802(0.1298) \\
		0.8   & 2.0      & 4.0212(0.0815) & 4.1023(0.0845) & 5.7254(0.1432) &       & 5.0538(0.0816) & 5.1568(0.1119) & 5.4333(0.0692) \\
		\midrule
		\multicolumn{9}{c}{$d = 1000, n = 1000$} \\
		\midrule
		Uniform & $\beta$   & High-order  & Cauchy & Non-truncation &       & High-order  & Cauchy & Non-truncation \\
		\midrule
		0.3   & 2.0      & 7.2532(0.1547) & 8.8754(0.1471) & 9.9584(0.1241) &       & 11.0984(0.0595) & 11.2696(0.0846) & 11.7264(0.0967) \\
		0.5   & 2.0      & 7.7652(0.1457) & 9.2413(0.1453) & 10.125(0.1892) &       & 11.1030(0.0870) & 11.3124(0.0460) & 12.1964(0.2704) \\
		0.8   & 2.0      & 7.9254(0.1745) & 9.5432(0.1793) & 11.235(0.1346) &       & 11.2125(0.0982) & 11.4113(0.0442) & 12.4671(0.0754) \\
		\bottomrule
	\end{tabular}}%
	\label{tab:NBGAU}%
\end{table}%

\begin{table}[htbp]
	\centering
	\caption{Comparison of average MAEs for DNN NBR under two noise settings.}
	\resizebox{1\textwidth}{!}{
		\begin{tabular}{|cc|cc|cc|}
			\hline
			\multicolumn{6}{|c|}{MAEs for DNN NBR} \\
			\hline
			&       & \multicolumn{2}{c|}{$d = 6$, $n = 200$ (Complex function)} & \multicolumn{2}{c|}{$d = 100$, $n = 1000$ (DNN)} \\
			\hline
			$\beta$  & Pareto ($\tau$)  & High-order  & Non-truncation  & High-order  & Non-truncation  \\
			\hline
			1.5   & 1.60   & 0.5809(0.0063) & 1.5829(0.4448) & 0.4702(0.0274)  & 0.6115(0.0392)  \\
			1.5   & 1.80   & 0.5586(0.0074) & 1.4882(0.3538) & 0.4483(0.0379)  & 0.5232(0.0402)  \\
			2.0     & 2.01  & 0.4845(0.0056) & 1.2000(0.2203) & 0.5620(0.0399)  & 0.6436(0.0294)  \\
			2.0    & 4.01  & 0.5466(0.0049) & 1.2958(0.9950) & 0.5251(0.0191)  & 0.7019(0.0152)  \\
			2.0     & 6.01  & 0.6001(0.0056) & 1.1369(0.2027) & 0.5253(0.0131)  & 0.6718(0.0178)  \\
			\hline
			$\beta$  & Uniform ($\pi$) & High-order  & Non-truncation  & High-order  & Non-truncation  \\
			\hline
			2.0     & 0.3   & 0.9842(0.2706) & 0.9672(0.3283) & 0.5260(0.0309)  & 0.5978(0.0271)  \\
			2.0     & 0.5   & 1.6386(0.3332) & 1.7000(0.1610) & 0.5784(0.0187)  & 0.6029(0.0088)  \\
			2.0     & 0.8   & 1.4314(0.4509) & 2.7475(0.3379) & 0.5824(0.0195)  & 0.6736(0.0057)  \\
			\hline
	\end{tabular}}%
	\label{tab:DNNNB}%
\end{table}%

\end{document}